\newif\ifFOCS

\FOCSfalse

\ifFOCS
  \documentclass[10pt, conference, compsocconf]{IEEEtran}
  \IEEEoverridecommandlockouts
  \renewcommand\baselinestretch{0.945} 
\else
  \documentclass[11pt]{article}
\fi

\ifFOCS
\def\AlgPDP{Algorithm Sublinear Perceptron}
\def\AlgMEB{Algorithm Sublinear MEB}
\def\AlgKernel{Algorithm Sublinear Kernel}
\else
\def\AlgPDP{Algorithm~\ref{alg:1}}
\def\AlgMEB{Algorithm~\ref{alg:2}}
\def\AlgKernel{Algorithm~\ref{alg:kernel}}
\fi

\ifFOCS

\fi

\usepackage{algorithm,algorithmic}
\usepackage[cmex10]{amsmath}
\usepackage{amsthm}
\usepackage{latexsym}
\usepackage{amssymb}
\usepackage{xspace}
\ifFOCS\else
\usepackage[hypertex]{hyperref}
\fi
\usepackage[pdftex]{graphicx}
    \graphicspath{{./}{../jpeg/}}
    \DeclareGraphicsExtensions{.pdf}

\newif\ifnormalthm 
\normalthmfalse
\normalthmtrue 

\newtheorem{theorem}{Theorem}[section]
\newtheorem{Thm}[theorem]{Theorem}
\newtheorem{Lem}[theorem]{Lemma}

\newtheorem{lemma}[theorem]{Lemma}

\newtheorem{claim}[theorem]{Claim}

\newtheorem{corollary}[theorem]{Corollary}

\newtheorem{fact}[theorem]{Fact}

\theoremstyle{definition}
\newtheorem{definition}[theorem]{Definition}
\newenvironment{proofof}[1]{\begin{trivlist} \item {\bf Proof
#1:~~}}
  {\qed\end{trivlist}}

\newcommand{\K}{\ensuremath{\mathcal K}}

\def\mH{\mathcal{H}}

\def\tO{\tilde{O}}

\def\reals{\mathbb{R}}
\def\ball{\mathbb{B}}

\def\cG{\mathcal{G}}
\def\mE{\mathcal{E}}

\newcommand\eps{\varepsilon}

\def\ftil{\tilde{f}}

\def\Otil{\tilde{O}}
\def\xtil{\tilde{x}}

\def\ftil{\tilde{f}}
\def\vtil{\tilde{v}}
\def\ptil{\tilde{p}}
\def\qtil{\tilde{q}}

\def\ktil{\tilde{k}}
\def\Ytil{\tilde{Y}}

\def\ksig{s} 
\newcommand{\ignore}[1]{}
\newcommand\E{\operatorname{\mathbf E}}
\DeclareMathOperator{\ProbP}{Prob}
\newcommand{\Prob}[1]{\ProbP\{{#1}\}}
\DeclareMathOperator{\Var}{Var}
\DeclareMathOperator{\clip}{clip}
\DeclareMathOperator{\sample}{\mbox{{\bf Sample}}}
\DeclareMathOperator{\MRadius}{Radius}
\DeclareMathOperator{\MCenter}{Center}
\DeclareMathOperator{\MEB}{MEB}
\DeclareMathOperator{\Alg}{\mathit{Alg}}
\DeclareMathOperator{\adv}{adv}
\def\DD{D} 

\newcommand\mycases[4] {{
\left\{
\begin{array}{ll}
    {#1} & {#2} \\\\
    {#3} & {#4}
\end{array}
\right. }}

\newcommand{\ceil}[1]{\lceil{#1}\rceil}

\DeclareMathOperator{\argmin}{argmin}

\DeclareMathOperator{\conv}{conv}
\def\veps{\varepsilon}
\def\norm#1{\mathopen\| #1 \mathclose\|}

\def\trans{^{\top}}

\newcommand\var{\mbox{\bf Var}}

\def\bA{\mathbf{A}}

\def\grad{\nabla}

\def\bone{\mathbf{1}}

\def\bzero{\mathbf{0}}

\newcommand\vecc[1]{\mathbf{#1}}

\def\reals{\mathbb{R}}

\DeclareMathOperator{\KernelEllTwo}{\mathbf{Kernel-L2-Sampling }}

\begin{document}
\title{Sublinear Optimization for Machine Learning}

\ifFOCS
\else
 \author{Kenneth L. Clarkson\thanks{IBM Almaden Research Center, San Jose, CA}
  \and Elad Hazan\thanks{Department of Industrial Engineering, Technion - Israel Institute of technology, Haifa   32000 Israel. Work done while at IBM Almaden Research Center}
 \and David P. Woodruff\thanks{IBM Almaden Research Center, San Jose, CA}
}
\fi

\maketitle

\begin{abstract}
We give sublinear-time approximation algorithms for some optimization problems
arising in machine learning, such as training linear classifiers and finding minimum enclosing
balls. Our algorithms can be extended to some kernelized versions of these
problems, such as SVDD, hard margin SVM, and $L_2$-SVM, for which sublinear-time
algorithms were not known before. These new algorithms use a combination of a novel sampling techniques and a new multiplicative update algorithm.  We give lower bounds which show the running times of many of our algorithms to be
nearly best possible in the unit-cost RAM model. We also give implementations of
our algorithms in the semi-streaming setting,
obtaining the first low pass polylogarithmic space {\it and} sublinear time algorithms achieving arbitrary approximation factor.
\end{abstract}


\section{Introduction}

Linear classification is a fundamental problem of machine learning, in which
positive and negative examples of a concept are represented in Euclidean space
by their feature vectors, and we seek to find a hyperplane separating the two
classes of vectors.

The Perceptron Algorithm for linear classification is one of
the oldest algorithms studied in machine learning \cite{Novikoff,Minsky88}.
It can be used to efficiently give a good
approximate solution, if one exists, and has nice noise-stability properties
which allow it to be used as a subroutine in many applications such as
learning with noise \cite{Bylander94,BlumFKV98}, boosting \cite{Ser99} and more
general optimization \cite{DunVem04}. In addition, it is extremely simple to
implement: the algorithm starts with an arbitrary hyperplane, and iteratively
finds a vector on which it errs, and moves in the direction of this vector
by adding a multiple of it to the normal vector to the current hyperplane.

The standard implementation of the Perceptron Algorithm must iteratively
find a ``bad vector'' which is classified incorrectly, that is, for which
the inner product with the current normal vector has an incorrect sign.
Our new algorithm is similar to
the Perceptron Algorithm, in that it maintains a hyperplane and modifies it
iteratively, according to the examples seen. However, instead of explicitly
finding a bad vector, we run another \emph{dual} learning algorithm to learn the ``most
adversarial" distribution over the vectors, and use that distribution to generate
an ``expected bad'' vector. Moreover, we do not compute the inner products
with the current normal vector exactly, but instead estimate them using
a fast sampling-based scheme.

Thus our update to the hyperplane uses
a vector whose ``badness'' is determined quickly, but very crudely.
We show that despite this, an approximate solution is still obtained in about the same
number of iterations as the standard perceptron. So our algorithm is faster; notably, it can
be executed in time \emph{sublinear} in the size of the input data, and still
have good output, with high probability. (Here we must make some reasonable
assumptions about the way in which the data is stored, as discussed below.)


This technique applies more generally than to the perceptron:
we also obtain
sublinear time approximation algorithms for the related problems of finding an
approximate Minimum Enclosing Ball (MEB) of a set of points, and training a
Support Vector Machine (SVM), in the hard margin or $L_2$-SVM formulations.

We give lower bounds that imply that our algorithms for classification
are best possible, up to
polylogarithmic factors, in the unit-cost RAM model, while our bounds
for MEB are best possible up to an $\tilde{O}(\eps^{-1})$ factor.
For most of these bounds,
we give a family of inputs such that a single coordinate, randomly ``planted''
over a large collection of input vector coordinates, determines the output
to such a degree that all coordinates in the collection must be examined
for even a $2/3$ probability of success.

We show that our algorithms can be implemented in the parallel setting,
and in the semi-streaming setting; for the latter, we need
a careful analysis of arithmetic precision requirements and an implementation
of our primal-dual algorithms using lazy updates, as well as some recent
sampling technology \cite{mw10}.

Our approach can be extended to give algorithms for the kernelized versions of
these problems, for some popular kernels including the Gaussian and polynomial,
and also easily gives Las Vegas results, where the output guarantees always
hold, and only the running time is probabilistic. \footnote{For MEB and the kernelized
versions, we assume that the Euclidean norms of the relevant
input vectors are known. Even with the addition of this linear-time step,
all our algorithms improve on prior bounds, with the exception of MEB
when $M = o(\eps^{-3/2}(n+d))$.} Our approach also applies to the case of soft margin
SVM (joint work in progress with Nati Srebro).

\begin{figure*}[t]
\begin{tabular}{|c|c|c|c|}
  \hline
  Problem & Previous time  & Time Here & Lower Bound \\
  \hline
  & & &\\
  classification/perceptron         & $\Otil(\eps^{-2}M)$       \cite{Novikoff} & $\Otil( \eps^{-2} (n +d))$        \S\ref{sec:perceptron} & $\Omega(\eps^{-2}(n+d))$ \S\ref{subsec:lbClass}\\
  min. enc. ball (MEB)              & $\Otil(\veps^{-1/2}M)$    \cite{SV}       & $\Otil( \eps^{-2} n + \eps^{-1} d )$        \S\ref{subsec:MEB}     & $\Omega(\eps^{-2} n+  \eps^{-1} d)$ \S\ref{subsec:lbMeb}\\
  QP in the simplex                 & $O(\veps^{-1}M)$          \cite{FW56}     & $\Otil( \eps^{-2} n + \eps^{-1} d )$           \S\ref{subsec:quad}    &  \\
  Las Vegas versions                &                                           & additive $O(M)$                       Cor \ref{cor:ubLV} & $\Omega(M)$ \S\ref{subsec:LV} \\
  kernelized MEB and QP             &                                           & factors $O(\ksig^4)$ or $O(q)$         \S\ref{sec:kernel-long}                                        &  \\
  \hline
\end{tabular}
\caption{Our results, except for semi-streaming and parallel}\label{fig:results}
\end{figure*}

Our main results, except for semi-streaming and parallel algorithms,
are given in Figure~\ref{fig:results}. 
The notation is as follows.
All the problems we consider have an $n\times d$ matrix $A$ as input,
with $M$ nonzero entries, and with each row of $A$ with Euclidean
length no more than one. The parameter $\epsilon>0$ is the
additive error; for MEB, this can be a relative error,
after a simple $O(M)$ preprocessing step.
We use the asymptotic notation
 $\Otil(f) = O(f  \cdot \textrm{polylog}\frac{nd}{\eps})$.
The parameter $\sigma$ is the \emph{margin} of the problem instance, explained below.
The parameters $\ksig$ and $q$ determine the standard deviation of a Gaussian kernel,
and degree of a polynomial kernel, respectively.

The time bounds given for our algorithms, except the Las Vegas ones, are under the assumption of
constant error probability; for output guarantees that hold with probability
$1-\delta$, our bounds should be multiplied by $\log(n/\delta)$.

The time bounds also require the assumption that the input data is stored
in such a way that a given entry $A_{i,j}$ can be recovered in constant time.
This can be done by, for example, keeping each row $A_i$ of $A$ as a hash table.
(Simply keeping the entries of the row in sorted order by column number is
also sufficient, incurring an $O(\log d)$ overhead in running time for
binary search.)

By appropriately modifying our algorithms, we obtain algorithms
with very low pass, space, and time complexity.
Many problems cannot be well-approximated in one pass, so a model permitting a
small number of passes over the data, called the semi-streaming
model, has gained recent attention \cite{fkmsz08,m05}. In this model
the data is explicitly stored,
and the few passes over it result in low I/O overhead.
It is quite suitable for
problems such as MEB, for which any algorithm using a single pass and sublinear (in $n$) space cannot
approximate the optimum value to within better than a fixed constant
\cite{as10}. Unlike traditional semi-streaming algorithms,
we also want our algorithms to be sublinear time,
so that in each pass only a small portion of the input is read.

We assume we see the points (input rows) one
at a time in an arbitrary order. The space is measured in bits.
For MEB, we obtain an algorithm with
$\tilde{O}(\eps^{-1})$ passes, $\tilde{O}(\eps^{-2})$ space, and
$\tilde{O}(\eps^{-3}(n+d))$ total time. For linear classification, we obtain
an algorithm with $\tilde{O}(\eps^{-2})$ passes, $\tilde{O}(\eps^{-2})$ space,
and $\tilde{O}(\eps^{-4}(n+d))$ total time.  
For comparison, prior streaming algorithms
for these problems \cite{as10, zc06} require a prohibitive $\Omega(d)$ space, and none
achieved a sublinear $o(nd)$ amount of time. Further, their guarantee
is an approximation up to a fixed constant, rather than for a general $\eps$
(though they can achieve a single pass).

%
%

\paragraph{Formal Description: Classification}
In the linear classification problem, the learner is given a set of $n$ labeled
examples in the form of $d$-dimensional vectors,
comprising the input matrix $A$.
The labels comprise a vector $y \in \{+1,-1\}^n$.

The goal is to find a separating hyperplane, that is, a normal vector $x$
in the unit Euclidean ball $\ball$
such that for all $i$, $y(i) \cdot A_i x \geq 0$; here
$y(i)$ denotes the $i$'th coordinate of $y$.
As mentioned, we will assume throughout that $A_i\in\ball$ for all $i\in [n]$,
where generally $[m]$ denotes the set of integers $\{1,2,\ldots,m\}$.

As is standard, we may assume that the labels $y(i)$ are all $1$, by taking $A_i \gets -A_i$
for any $i$ with $y(i)=-1$. The approximation version of linear
classification (which is necessary in case there is noise), is to find a vector
$x_\veps\in\ball$ that is an \emph{$\eps$-approximate solution}, that is,
\begin{equation} \label{eqn:formulation1}
\forall i' \, \ A_{i'} x_\veps \geq \max_{x\in\ball} \min_i  A_i x - \veps.
\end{equation}
The optimum for this formulation is obtained when $\norm{x}=1$,
except when no separating hyperplane exists, and then the optimum
$x$ is the zero vector.

Note that $\min_i  A_i x = \min_{p\in\Delta} p\trans A x$,
where $\Delta\subset\reals^n$ is the unit simplex
$\{p\in\reals^n\mid p_i\ge 0, \sum_i p_i=1\}$.
Thus we can regard
the optimum as the outcome of a game to determine $p\trans A x$,
between a minimizer choosing $p\in\Delta$, and a maximizer choosing $x\in\ball$,
yielding
\[
\sigma\equiv \max_{x\in\ball}\min_{p\in\Delta} p\trans Ax,
\]
where this optimum $\sigma$ is called the \emph{margin}. From standard
duality results, $\sigma$ is also the optimum of the dual problem
\[
\min_{p\in\Delta} \max_{x\in\ball} p\trans Ax,
\]
and the optimum vectors $p^*$ and $x^*$
are the same for both problems.

The classical
Perceptron Algorithm returns an $\varepsilon$-approximate solution
to this problem in
$\frac{1}{\varepsilon^2}$ iterations, and total time
$O(\eps^{-2}M)$.

For given $\delta\in (0,1)$, our new algorithm takes
 $O(\eps^{-2}(n+d)(\log n)\log(n/\delta))$ time to
return an $\eps$-approximate solution with probability at least $1-\delta$.
Further, we show this is optimal in the unit-cost RAM model, up to
poly-logarithmic factors.

\paragraph{Formal Description: Minimum Enclosing Ball (MEB)}
The MEB problem is to find the smallest Euclidean ball in $\reals^d$ containing
the rows of~$A$. It is a special case of quadratic programming (QP) in the
unit simplex, namely, to find
 $\min_{p\in\Delta} p\trans b + p\trans AA\trans p$, where
$b$ is an $n$-vector. This relationship, and the generalization
of our MEB algorithm to QP in the simplex, is discussed in \S\ref{subsec:quad};
for more general background on QP in the simplex, and related problems, see
for example \cite{KenFW}.

\subsection{Related work}\label{subsec:related}

Perhaps the most closely related work is that of Grigoriadis and Khachiyan
\cite{GriKha95}, who showed how to approximately solve a zero-sum game up to
additive precision $\eps$ in time $\Otil(\eps^{-2}(n+d))$,
where the game matrix is $n\times d$. This problem
is analogous to ours, and our algorithm is similar in structure to theirs,
but where we minimize over $p\in\Delta$ and
maximize over $x\in\ball$, their optimization has not only $p$ but also
$x$ in a unit simplex.

Their algorithm (and ours) relies on sampling based on $x$ and $p$,
to estimate inner products $x\trans v$ or $p\trans w$ for vectors $v$ and $w$
that are rows or columns of $A$.
For a vector $p\in \Delta$, this estimation is easily done
by returning $w_i$ with probability $p_i$.

For vectors $x\in\ball$, however, the natural estimation technique is
to pick $i$ with probability $x_i^2$, and return $v_i/x_i$.
The estimator from this \emph{$\ell_2$ sample}
is less well-behaved, since it is unbounded, and can have a high variance.
While $\ell_2$
sampling has been used in streaming applications \cite{mw10}, it has not previously
found
applications in optimization due to this high variance problem.

Indeed, it might seem surprising that sublinearity is at all
possible, given that the correct classifier might be determined by very few
examples, as shown in figure \ref{fig:antipodes-example}. It thus seems
necessary to go over all examples at least once, instead of looking at noisy
estimates based on sampling.

\begin{figure}[h!]
\begin{center}
\includegraphics[width=2.0in]{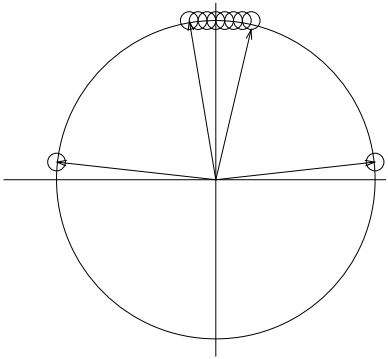}
\end{center}
\caption{The optimum $x_*$ is determined by the vectors near the horizontal axis. \label{fig:antipodes-example}}
\end{figure}

However, as we show, in our setting there is a version of the fundamental
Multiplicative Weights (MW)
technique that can cope with unbounded updates, and for which
the variance of $\ell_2$-sampling
is manageable. In our version of MW, the
multiplier associated with a value $z$ is quadratic in $z$,
in contrast to the more standard multiplier that is exponential in $z$;
while the latter is a fundamental building block
in approximate optimization algorithms, as discussed by Plotkin \emph{et al.} \cite{PST91},
in our setting such exponential updates can lead to a very expensive
$d^{\Omega(1)}$ iterations.

We analyze MW from the perspective of
on-line optimization, and show that our version of MW has
low expected expected regret given only that the random updates have the
variance bounds provable for $\ell_2$ sampling.
We also use another technique from on-line optimization, a gradient descent
variant which is better suited for the ball.

For the special case of zero-sum games in which the entries are all non-negative
(this is equivalent to packing and covering linear programs), Koufogiannakis and
Young \cite{KoufogiannakisY07} give a sublinear-time algorithm which returns a
{\it relative} approximation in time
$\Otil(\eps^{-2}(n+d))$. Our lower bounds show that a similar relative approximation
bound for sublinear algorithms is
impossible for general classification, and hence general linear programming.

\section{Linear Classification and the Perceptron}\label{sec:perceptron}


Before our algorithm, some reminders and further notation:
$\Delta\subset\reals^n$ is the unit simplex
$\{p\in\reals^n\mid p_i\ge 0, \sum_i p_i=1\}$,
$\ball\subset\reals^d$ is the Euclidean unit ball,
and the unsubscripted $\norm{x}$ denotes the Euclidean norm $\norm{x}_2$.
The $n$-vector, all of whose entries are one, is denoted by $\vecc{1}_n$.

The $i$'th row of the input matrix $A$ is denoted $A_i$, although a
vector is a column vector unless otherwise indicated. The $i$'th coordinate of
vector $v$ is denoted $v(i)$.
For a vector $v$, we let $v^2$ denote the vector
whose coordinates have $v^2(i)\equiv v(i)^2$ for all $i$.

\subsection{The Sublinear Perceptron}

Our sublinear perceptron algorithm is given in Figure~\ref{alg:1}.
The algorithm maintains a vector $w_t\in\reals^n$, with nonnegative coordinates,
and also $p_t\in\Delta$, which is $w_t$ scaled to
have unit $\ell_1$ norm. A vector $y_t\in\reals^d$ is maintained also,
and $x_t$ which is $y_t$ scaled to have
Euclidean norm no larger than one. These normalizations are done
on line~\ref{alg:pdp norm}.

In lines \ref{alg:pdp OGD 1} and \ref{alg:pdp OGD 2}, the algorithm is updating $y_t$ by adding a row of $A$
randomly chosen using $p_t$. This is a randomized version of \emph{Online
Gradient Descent} (OGD); due to the random choice of $i_t$, $A_{i_t}$ is an
unbiased estimator of $p_t\trans A$, which is the gradient of $p_t \trans A y$
with respect to $y$.

In lines \ref{alg:pdp MW first} through \ref{alg:pdp MW last},
the algorithm is updating $w_t$ using a column $j_t$ of $A$
randomly chosen based on $x_t$, and also using the value $x_t(j_t)$. This is a
version of the Multiplicative Weights (MW) technique for online optimization in
the unit simplex, where $v_t$ is an unbiased estimator of $Ax_t$, the gradient
of $p\trans A x_t$ with respect to $p$.

Actually, $v_t$ is not unbiased,
after the $\clip$ operation: for $z, V\in\reals$, $\clip(z, V) \equiv \min\{V, \max\{-V, z\}\}$,
and our analysis is helped by clipping the entries of $v_t$; we show that
the resulting slight bias is not harmful.

As discussed in \S\ref{subsec:related}, the sampling used to choose $j_t$ (and update $p_t$)
is \emph{$\ell_2$-sampling}, and that for $i_t$, $\ell_1$-sampling. These
techniques, which can be regarded as special cases of an $\ell_p$-sampling
technique, for $p \in [1,\infty)$, yield unbiased estimators of vector dot
products. It is important for us also that $\ell_2$-sampling
has a variance bound here; in particular, for each relevant $i$ and $t$,
\begin{equation}\label{eq:ell2 moment}
\E[v_t(i)^2]\le \norm{A_i}^2\norm{x_t}^2 \le 1.
\end{equation}

\ifFOCS
  \begin{figure}[!t]
\else
  \begin{algorithm}[h!]
  \caption{Sublinear Perceptron}
\fi
    \begin{algorithmic}[1]
    \STATE Input: $\eps>0$, $A \in \reals^{n \times d}$ with $A_i\in\ball$ for $i\in [n]$.
    \STATE Let $T \gets 200^2 \eps^{-2}\log n$, $y_1 \gets 0$, $w_1 \gets \vecc{1}_n$,
    	\\\ $\eta\gets \frac{1}{100}  \sqrt{\frac{\log n}{T}}$.
    \FOR{$t=1$ to $T$}
        \STATE $\ p_t \gets \frac{w_t}{\norm{w_t}_1}$,
                $\ x_t \gets \frac{y_t}{\max\{1,\norm{y_t}\}}.$  \label{alg:pdp norm}
        \STATE Choose $i_t\in [n]$ by $i_t\gets i$ with prob. $p_t(i)$. \label{alg:pdp OGD 1}
        \STATE $ y_{t+1} \gets y_t + \frac{1}{\sqrt{2T}} A_{i_t}$ \label{alg:pdp OGD 2}
        \STATE Choose $j_t\in [d]$ by \\\ $j_t \gets j$ with probability $x_t(j)^2/\norm{x_t}^2$. \label{alg:pdp MW first}
        \FOR{$i\in [n]$}
            \STATE $\vtil_t(i) \gets A_i(j_t)\norm{x_t}^2/x_t(j_t)$ \label{alg:vt initial}
            \STATE $v_t(i) \gets \clip(\vtil_t(i), 1/\eta)$
            \STATE $w_{t+1}(i) \gets w_t(i) (1-\eta v_t(i) + \eta^2 v_t(i)^2)$ \label{alg:vt update}
        \ENDFOR   \label{alg:pdp MW last}
    \ENDFOR
    \RETURN $\bar{x} = \frac{1}{T} \sum_t x_t $
    \end{algorithmic}
 \ifFOCS
  \caption{Algorithm Sublinear Perceptron, a perceptron training algorithm}\label{alg:1}
  \end{figure}
\else
   \label{alg:1}
   \end{algorithm}
\fi

First we note the running time.

\begin{Thm} \label{thm:perc runtime}
The sublinear perceptron takes $O( \eps^{-2}\log n )$ iterations, with
a total running time of
 $ {O}( \veps^{-2}(n +d)\log n )$.
\end{Thm}

\ifFOCS\else 
\begin{proof}

The algorithm iterates
 $T = O(\frac{\log n}{\veps^2})$ times.
Each iteration requires:
\begin{enumerate}
\item
One $\ell_2$ sample per iterate, which takes $O(d)$ time using known data
structures.
\item
Sampling $i_t \in _R p_t$  which takes $O(n)$ time.
\item
The update of $x_t$ and $p_t$, which takes $O(n+d)$ time.
\end{enumerate}
The total running time is $ O( \veps^{-2}(n +d)\log n)$.
\end{proof}
\fi

Next we analyze the output quality. The proof
uses new tools from regret minimization and sampling that are the building
blocks of most of our upper bound results.

Let us first state the MW algorithm used in all our algorithms.

\begin{definition}[MW algorithm]\label{def:MW}
Consider a sequence of vectors $q_1,\ldots ,q_T\in \reals^n$.
The \emph{Multiplicative Weights} (MW) algorithm is as follows.
Let $w_1\gets \vecc{1}_n$, and for $t \ge 1$,
\begin{equation}\label{eq: p update}
p_t \gets w_t/\norm{w_t}_1,
\end{equation}
and for $0 < \eta  \in\reals$
\begin{equation}\label{eq: w update}
 w_{t+1}(i) \gets w_t(i)(1-\eta q_t(i)+\eta^2 q_t(i)^2),
\end{equation}
\end{definition}

The following is a key lemma, which proves a novel bound on the regret of the MW algorithm above, suitable for the case where the losses are random variables with bounded variance.
\ifFOCS

\else
This is proven below, after a concentration
lemma, and the main theorem and its proof.
\fi

\begin{lemma}[Variance MW Lemma] \label{lem:regretMW}
The MW algorithm satisfies
\begin{eqnarray*}
\sum_{t\in [T]} p_t\trans  q_t
   & \le \min_{i\in [n]} \sum_{t\in[T]} \max\{ q_t(i), -\frac{1}{\eta} \} \\
    &        + \frac{\log n}{\eta} + \eta \sum_{t\in[T]} p_t\trans q_t^2.
\end{eqnarray*}
\end{lemma}

The following three lemmas give concentration bounds on our random variables
from their expectations. The first two are based on standard martingale
analysis, and the last is a simple Markov application.
\ifFOCS\else
The proofs are deferred
to Appendix \ref{sec:aux lemmas}.
\fi

\begin{lemma}\label{lem:v_t conc}
For $\eta \leq \sqrt{\frac{\log n}{10 T}}$, with probability at least $1-O(1/n)$,
\[
\max_i \sum_{t\in [T]} [v_t(i) - A_i x_t] \le 90 \eta T  .
\]
\end{lemma}
\begin{lemma}\label{lem:highprob2}
For $\eta \leq \sqrt{\frac{\log n}{10 T}}$, with probability at least
$1-O(1/n)$, it holds that $\left| \sum_{t\in [T]} A_{i_t} x_t - \sum_t p_t
\trans v_t \right| \le 100 \eta T .$
\end{lemma}
\begin{lemma}\label{lem:medprob1}
With probability at least $1 - \frac{1}{4}$,  it holds that $\sum_t p_t \trans v_t^2 \leq 8T .$
\end{lemma}

\begin{Thm}[Main Theorem] \label{thm:main}
With probability $1/2$, the sublinear perceptron
returns a solution $\bar{x}$ that is an $\veps$-approximation.
\end{Thm}

\ifFOCS\else 
\begin{proof}

First we use the regret bounds for lazy gradient descent to lower bound
$\sum_{t\in [T]}  A_{i_t} x_t$, next we get an upper bound for that
quantity using the Weak Regret lemma above, and then we combine the two.

By definition, $A_i x^*\ge \sigma$ for all $i\in [n]$, and so, using
the bound of Lemma~\ref{lem:lazyogd},
\begin{equation}\label{eqn:Ait lower}
T\sigma
    \le \max_{x\in\ball} \sum_{t\in [T]}  A_{i_t} x
    \le \sum_{t\in [T]}  A_{i_t} x_t + 2\sqrt{2T},
\end{equation}
or rearranging,
\begin{equation}\label{eqn:gdPrimalDual}
\sum_{t\in [T]}  A_{i_t} x_t \ge T\sigma - 2\sqrt{2T}.
\end{equation}

Now we turn to the MW part of our algorithm.
By the Weak Regret Lemma~\ref{lem:regretMW},
and using the clipping of $v_t(i)$,
\[
\sum_{t\in [T]} p_t\trans  v_t
    \le \min_{i\in [n]} \sum_{t\in[T]} v_t(i)
            + (\log n)/\eta + \eta \sum_{t\in[T]} p_t\trans v_t^2.
\]
By Lemma~\ref{lem:v_t conc} above, with high probability,
for any $i\in [n]$,
\[
\sum_{t\in [T]} A_i x_t
    \ge \sum_{t\in [T]} v_t(i) - 90\eta T,
\]
so that with high probability
\begin{eqnarray}\label{eq:weak reg + hp}
\sum_{t\in [T]} p_t\trans  v_t
    & \le \min_{i\in [n]} \sum_{t\in[T]} A_i x_t
            + (\log n)/\eta \notag \\
      &      + \eta \sum_{t\in[T]} p_t\trans v_t^2 + 90T\eta.
\end{eqnarray}

Combining \eqref{eqn:gdPrimalDual} and \eqref{eq:weak reg + hp} we get
\begin{align*}
& \min_{i\in [n]} \sum_{t\in[T]} A_i x_t
    \geq - (\log n)/\eta - \eta \sum_{t\in[T]} p_t\trans v_t^2 - 90T\eta \\
&    +  T\sigma - 2\sqrt{2T} - |\sum_{t\in [T]} p_t\trans  v_t - \sum_{t\in [T]}  A_{i_t} x_t|
\end{align*}
By Lemmas \ref{lem:highprob2}, \ref{lem:medprob1} we have w.p at least
$\frac{3}{4} - O(\frac{1}{n})  \geq \frac{1}{2}$
\begin{align*}
\min_{i\in [n]} \sum_{t\in[T]} A_i x_t
    & \geq - (\log n)/\eta - 8 \eta T - 90T\eta +  T\sigma - 2\sqrt{2T} - 100 \eta T \\
    & \geq T \sigma - \frac{\log n}{\eta} -200 \eta T.
\end{align*}

Dividing through by $T$, and using our choice of $\eta$,
we have $\min_i A_i \bar{x} \ge \sigma - \eps/2$ w.p. at least least $1/2$ as claimed.
\end{proof}
\fi 

\ifFOCS\else 
\begin{proof}[Proof of Lemma~\ref{lem:regretMW}, Weak Regret]
We first show an upper bound on $\log\norm{w_{T+1}}_1$,
then a lower bound, and then relate the two.

From \eqref{eq: w update} and \eqref{eq: p update} we have
\begin{align*}
\norm{w_{t+1}}_1
       & = \sum_{i\in[n]} w_{t+1}(i)
    \\ & = \sum_{i\in[n]} p_t(i) \norm{w_t}_1 (1-\eta q_t(i)+\eta^2 q_t(i)^2)
    \\ & = \norm{w_t}_1 (1 - \eta p_t\trans  q_t + \eta^2 p_t\trans  q_t^2).
\end{align*}
This implies by induction on $t$,
and using $1+z\le\exp(z)$ for $z\in\reals$, that
\begin{equation}\label{eq: Phi upper}
\log\norm{w_{T+1}}_1
    = \log n + \sum_{t\in [T]} \log(1 - \eta p_t\trans  q_t + \eta^2 p_t\trans q_t^2)
    \le \log n - \sum_{t\in [T]} \eta p_t\trans  q_t + \eta^2 p_t\trans  q_t^2.
\end{equation}

Now for the lower bound.
From (\ref{eq: w update}) we have by induction on $t$ that
\[
w_{T+1}(i)
    = \prod_{t\in[T]} (1-\eta q_t(i)+\eta^2 q_t(i)^2),
\]
and so
\begin{align*}
\log\norm{w_{T+1}}_1
       & = \log\left[\sum_{i\in [n]} \prod_{t\in[T]} (1-\eta q_t(i)+\eta^2 q_t(i)^2) \right]
    \\ & \ge \log\left[\max_{i\in [n]} \prod_{t\in[T]} (1-\eta q_t(i)+\eta^2 q_t(i)^2) \right]
    \\ & = \max_{i\in [n]} \sum_{t\in[T]} \log(1-\eta q_t(i)+\eta^2 q_t(i)^2)
    \\ & \ge \max_{i\in [n]} \sum_{t\in[T]} [\min\{- \eta q_t(i),1\} ],
\end{align*}
where the last inequality uses the fact that $1+z+z^2 \ge \exp(\min\{z,1\})$ for all
$z\in\reals$.

Putting this together with the upper bound \eqref{eq: Phi upper}, we have
\[
\max_{i\in [n]} \sum_{t\in[T]} [\min\{-\eta q_t(i) , 1\}]
    \le \log n - \sum_{t\in [T]} \eta p_t\trans  q_t + \eta^2 p_t\trans  q_t^2,
\]
Changing sides
\begin{align*}
\sum_{t\in [T]} \eta p_t\trans  q_t
    & \le - \max_{i\in [n]} \sum_{t\in[T]} [\min\{-\eta q_t(i) , 1\}] + \log n + \eta^2 p_t\trans  q_t^2, \\
    & =  \min_{i\in [n]} \sum_{t\in[T]} [\max\{\eta q_t(i) , -1\}] + \log n + \eta^2 p_t\trans  q_t^2, \\
\end{align*}
and the lemma follows, dividing through by $\eta$.
\end{proof}
\fi 

\begin{corollary}[Dual solution]\label{cor:perceptron dual}
The vector $\bar{p}\equiv \sum_t e_{i_t}/T$ is, with probability
$1/2$, an $O(\varepsilon)$-approximate dual solution.
\end{corollary}

\ifFOCS\else 
\begin{proof}

Observing in \eqref{eqn:Ait lower} that the middle expression $\max_{x\in\ball} \sum_{t\in [T]}  A_{i_t} x$ is equal to
 $T\max_{x\in\ball} \bar{p}\trans A x$,
we have $ T \max_{x\in\ball} \bar{p}\trans A x \le \sum_{t\in[T]} A_{i_t} x_t + 2\sqrt{2T}$,
or changing sides,
$$ \sum_{t\in[T]} A_{i_t} x_t \geq T \max_{x\in\ball} \bar{p}\trans A x - 2\sqrt{2T}$$
Recall from \eqref{eq:weak reg + hp} that with high probability,
\begin{equation}
\sum_{t\in [T]} p_t\trans  v_t
    \le \min_{i\in [n]} \sum_{t\in[T]} A_i x_t
            + (\log n)/\eta + \eta \sum_{t\in[T]} p_t\trans v_t^2 + 90T\eta.
\end{equation}
Following the proof of the main Theorem, we combine both inequalities and use Lemmas \ref{lem:highprob2},\ref{lem:medprob1}, such that with probability at least $\frac{1}{2}$:
\begin{align*}
T  \max_{x\in\ball} \bar{p}\trans A x
    & \le  \min_{i\in [n]} \sum_{t\in[T]} A_i x_t
            + (\log n)/\eta + \eta \sum_{t\in[T]} p_t\trans v_t^2 + 90 T\eta + 2 \sqrt{2T} + | \sum_{t\in [T]} p_t\trans  v_t - \sum_{t\in [T]}  A_{i_t} x_t| \\
    &  \le T \sigma
            + O( \sqrt{T \log n})
\end{align*}
Dividing through by $T$ we have with probability at least $\frac{1}{2}$ that
$  \max_{x\in\ball} \bar{p}\trans A x \le \sigma + O(\epsilon)$ for our choice
of $T$ and $\eta$.
\end{proof}
\fi 

\subsection{High Success Probability and Las Vegas}\label{subsec:highprob}

Given two vectors $u,v\in\ball$, we have seen that a single $\ell_2$-sample is
an unbiased estimator of their inner product with variance at most one.
Averaging $\frac{1}{\varepsilon^2}$ such samples reduces the variance to
$\varepsilon^2$, which reduces the standard deviation to $\varepsilon$.
Repeating $O(\log \frac{1}{\delta})$ such estimates, and taking the median,
gives an estimator denoted $X_{\varepsilon,\delta}$, which satisfies, via a
Chernoff bound:
$$ \Pr[ |X_{\varepsilon,\delta} - v^\top u | > \varepsilon ] \leq \delta $$
As an immediate corollary of this fact we obtain:
\begin{corollary}\label{cor:verification}
There exists a randomized algorithm that with probability $1-\delta$,
successfully determines whether a given hyperplane with normal vector $x
\in\ball$, together with an instance of linear classification and parameter
$\sigma > 0$, is an $\varepsilon$-approximate solution. The algorithm runs in
time $O(d + \frac{n}{\varepsilon^2} \log \frac{n}{\delta})$.
\end{corollary}
\ifFOCS\else 
\begin{proof}

Let $\delta' = \delta / n$. Generate the random variable
$X_{\varepsilon,\delta'}$ for each inner product pair $\langle x , A_i \rangle$,
and return true if and only if $X_{\varepsilon,\delta'} \geq \sigma -
\varepsilon$ for each pair. By the observation above and taking union bound over
all $n$ inner products, with probability $1-\delta$ the estimate
$X_{\varepsilon,\delta'}$ was $\varepsilon$-accurate for all inner-product
pairs, and hence the algorithm returned a correct answer. \\ The running time
includes preprocessing of $x$ in $O(d)$ time, and $n$ inner-product estimates,
for a total of $O(d + \frac{n}{\varepsilon^2} \log \frac{n}{\delta})$.
\end{proof}
\fi 

Hence, we can amplify the success probability of \AlgPDP\  to
$1-\delta$ for any $\delta > 0$ albeit incurring additional poly-log factors in
running time:
\begin{corollary}[High probability]\label{cor:hp}
There exists a randomized algorithm that with probability $1-\delta$ returns an
$\varepsilon$-approximate solution to the linear classification problem, and
runs in expected time $O(\frac{n+d}{\varepsilon^2} \log \frac{n}{\delta} )$.
\end{corollary}
\ifFOCS\else 
\begin{proof}

Run \AlgPDP\  for $\log_2 \frac{1}{\delta}$ times to generate that
many candidate solutions. By Theorem \ref{thm:main}, at least one candidate
solution is an $\varepsilon$-approximate solution with probability at least $1 -
2^{-\log_2 \frac{1}{\delta}} = 1 - \delta $.

For each candidate solution apply the verification procedure above with success
probability $1 - \delta^2 \geq 1-\frac{\delta}{\log \frac{1}{\delta}} $, and all
verifications will be correct again with probability at least $1-\delta$. Hence,
both events hold with probability at least $1-2\delta$. The result follows after
adjusting constants.

The worst-case running time comes to $O(\frac{n+d}{\varepsilon^2} \log \frac{n}{\delta} \log \frac{1}{\delta})$. However, we can generate the candidate solutions and verify them one at a time, rather than all at once. The expected number of candidates we need to generate is constant.
\end{proof}
\fi 

It is also possible to obtain an algorithm that never errs:
\begin{corollary}[Las Vegas Version]\label{cor:ubLV}
After $O( \eps^{-2}\log n )$ iterations, the sublinear perceptron
returns a solution that with probability $1/2$
can be verified in $O(M)$ time to be $\eps$-approximate.
Thus with expected $O(1)$ repetitions,
and a total of expected $O(M + \veps^{-2}(n +d)\log n )$ work,
a verified $\eps$-approximate solution can be found.
\end{corollary}

\ifFOCS\else 
\begin{proof}

We have
\[
\min_i A_i\bar{x}
    \le \sigma \le \norm{\bar{p}\trans A},
\]
and so if
\begin{equation}\label{eq:test}
\min_i A_i\bar{x}
    \ge \norm{\bar{p}\trans A} - \epsilon,
\end{equation}
then $\bar{x}$ is an $\eps$-approximate solution, and $\bar{x}$
will pass this test if it and $\bar{p}$ are $(\eps/2)$-approximate solutions,
and the same for $\bar{p}$.

Thus, running the algorithm for a constant factor more iterations,
so that with probability $1/2$,
$\bar{x}$ and $\bar{p}$
are both $(\epsilon/2)$-approximate solutions, it can be verified
that both are $\eps$-approximate solutions.
\end{proof}
\fi 

\ifFOCS\else 
\subsection{Further Optimizations}

The regret of OGD as given in Lemma \ref{lem:lazyogd} is smaller than the dual
strategy of random MW. We can take advantage of this and improve the running time
slightly, by replacing line [\ref{alg:pdp OGD 2}] of the sublinear algorithm with the line shown below.

  \begin{algorithm}[h!]
	\begin{algorithmic}[\ref{alg:pdp OGD 2}]
        \STATE[\ref{alg:pdp OGD 2}'] With probability $\frac{1}{\log T}$, let
              $ y_{t+1} \gets y_t + \frac{1}{2\sqrt{T}} A_{i_t}$ (else do nothing).
	\end{algorithmic}
   \end{algorithm}

This has the effect of increasing the regret of the primal online algorithm by a
$\log n$ factor, which does not hurt the number of iterations required to
converge, since the overall regret is dominated by that of the MW algorithm.

Since the primal solution $x_t$ is not updated in every iteration, we improve the
running time slightly to
$$ O( \veps^{-2} \log n (n +d /(\log 1/\eps + \log \log n)) ).$$
We use this technique to greater effect for the MEB problem, where it is discussed
in more detail.

\fi 

\subsection{Implications in the PAC model}

Consider the ``separable" case of hyperplane learning, in which there exists a hyperplane classifying all data points correctly.
It is well known that the concept class of hyperplanes in $d$ dimensions with margin $\sigma$ has effective dimension at most $\min\{d,\frac{1}{\sigma^2}\} + 1$. Consider the case in which the margin is significant, i.e. $\frac{1}{\sigma^2} < d$. PAC learning theory implies that the number of examples needed to attain generalization error of $\delta$ is $O(\frac{1}{\sigma^2 \delta})$.

Using the method of online to batch conversion (see \cite{Cesa-BianchiCG04}), and applying the online gradient decent algorithm, it is possible to obtain $\delta$ generalization error in time $O( \frac{d}{\sigma^2 \delta})$ time, by going over the data once and performing a gradient step on each example.

Our algorithm improves upon this running time bound as follows: we use the sublinear perceptron to compute a $\sigma/2$-approximation to the best hyperplane over the test data, where the number of examples is taken to be $n = O(\frac{1}{\sigma^2 \delta})$ (in order to obtain $\delta$ generalization error). As shown previously, the total running time amounts to $\tilde{O}( \frac{\frac{1}{\sigma^2 \delta} + d}{\sigma^2}) = O(\frac{1}{\sigma^4 \delta} + \frac{d}{\sigma^2}) $.

This improves upon standard methods by a factor of $\tilde{O}( {\sigma^2 d}) $, which is always an improvement by our initial assumption on $\sigma$ and $d$.

\section{Strongly convex problems: MEB and SVM}\label{sec:quad}


\subsection{Minimum Enclosing Ball}\label{subsec:MEB}

In the Minimum Enclosing Ball problem the input consists of a matrix
 $A \in \reals^{n \times d}$.
The rows are interpreted as vectors and the problem is to find a vector $x \in \reals^d$ such that
$$ x_* \equiv \argmin_{x\in\reals^d} \max_{i\in [n]} \norm{x - A_i }^2 $$
We further assume for this problem that all vectors $A_i$ have Euclidean norm at
most one. Denote by $\sigma = \max_{i\in [n]} \norm{x - A_i }^2$ the radius of the optimal ball, and we say that a solution is $\veps$-approximate if the ball it generates has radius at most $\sigma + \veps$.

As in the case of linear classification, to obtain tight running time bounds we
use a primal-dual approach;
\ifFOCS
\else
the algorithm is given below.

(This is a ``conceptual'' version of the algorithm: in the analysis of the running time,
we use the fact that we can batch together the updates for $w_t$ over
the iterations for which $x_t$ does not change.)

  \begin{algorithm}[h!]
  \caption{Sublinear Primal-Dual MEB}

    \begin{algorithmic}[1]
    \STATE Input: $\eps > 0 $, $A \in \reals^{n \times d}$ with $A_i\in\ball$ for $i\in [n]$ and $\norm{A_i}$ known.
    \STATE Let $T \gets \Theta( \eps^{-2}\log n)$ , $y_1 \gets \bzero$, $w_1 \gets \vecc{1} $, $\eta \gets \sqrt{(\log n)/T}$,
                $\alpha \gets \frac{\log T}{\sqrt{ T \log n}}$.
    \FOR{$t=1$ to $T$}
    \STATE $\ p_{t} \gets \frac{w_{t}}{\norm{w_{t}}_1}$
    \STATE Choose $i_t\in [n]$ by $i_t\gets i$ with probability $p_t(i)$.
    \STATE With probability $\alpha$, update  $ y_{t+1} \gets y_t + A_{i_t}  \ ,
            \ x_{t+1} \gets \frac{y_{t+1}}{t} .$ (else do nothing)
    \STATE Choose $j_t\in [d]$ by $j_t \gets j$ with probability $x_t(j)^2/\norm{x_t}^2$.
    \FOR{$i\in [n]$}
        \STATE $\vtil_t(i)\gets -2 A_i(j_t)\norm{x_t}^2/x_t(j_t) + \norm{A_i}^2 + \norm{x_t}^2.$
        \STATE $v_t(i) \gets \clip(\vtil_t(i) , \frac{1}{\eta}) $.
        \STATE $w_{t+1}(i) \gets w_t(i) (1+\eta v_t(i) + \eta^2 v_t(i)^2)$.
    \ENDFOR
    \ENDFOR
    \RETURN $\bar{x} = \frac{1}{T} \sum_t x_t $
    \end{algorithmic}
   \label{alg:2}
   \end{algorithm}

\fi 

\begin{Thm}
\AlgMEB\  runs in $O(\frac{\log n}{\veps^2} )$ iterations,
with a total expected running time of
$$ \Otil \left ( \frac{n}{\veps^2} + \frac{d }{\veps} \right ),$$
and with probability $1/2$, returns an $\veps$-approximate solution.
\end{Thm}
\ifFOCS\else 
\begin{proof}

Except for the running time analysis, the proof of this theorem is very similar
to that of Theorem~\ref{thm:main}, where we take advantage of a tighter
regret bound for strictly convex loss functions in the
case of MEB, for which the OGD algorithm with a learning rate of $\frac{1}{t}$
is known to obtain a tighter regret bound of $O(\log T)$ instead of
$O(\sqrt{T})$. For presentation, we use asymptotic notation rather than computing the exact constants (as done for the linear classification problem).

Let $f_t(x) = \norm{x - A_{i_t}}^2$. Notice that $\arg \min_{x \in \ball}
\sum_{\tau=1}^t f_\tau(x) = \frac{\sum_{\tau=1}^t A_{i_\tau}}{t} $. By
Lemma~\ref{lem:RandomStrictlyConvex} such that $f_t(x) = \norm{x - A_{i_t}}^2$,
with $G\le 2$ and $H=2$, and $x^*$ being the solution to the instance, we have
\begin{equation}\label{eqn:gdPrimalDualMEB}
  \  \E_{\{c_t\}}[  \sum_t  \norm{x_t - A_{i_t}}^2]
    \le \E_{\{c_t\}}[\sum_t \norm{ x^* - A_{i_t} }^2 ] + \frac{4}{\alpha} \log T
    \le T\sigma + \frac{4}{\alpha} \log T,
\end{equation}
where $\sigma$ is the squared MEB radius.
Here the expectation is taken only over the random coin tosses for updating $x_t$,
denoted $c_t$, and holds for any outcome of the indices $i_t$ sampled from $p_t$ and the coordinates $j_t$ used for the $\ell_2$ sampling.

Now we turn to the MW part of our algorithm.
By the Weak Regret Lemma~\ref{lem:regretMW},
using the clipping of $v_t(i)$, and reversing inequalities to account for the change of sign, we have
\[
\sum_{t\in [T]} p_t\trans  v_t
    \ge \max_{i\in [n]} \sum_{t\in[T]} v_t(i)
            - O(  \frac{\log n}{\eta} +  \eta \sum_{t\in[T]} p_t\trans v_t^2 ).
\]

Using Lemmas~\ref{lem:genericHP_v_and_mu},\ref{lem:highprob2_generic} with high probability
\[
\forall i \in [n] \ . \ \sum_{t\in [T]} v_t(i)
    \ge \sum_{t\in [T]} \| A_i  - x_t\|^2  - O( \eta T),
\]

\[ \left| \sum_{t\in [T]} \|x_t - A_{i_t} \|^2 - \sum_t p_t
\trans v_t \right| = O(\eta T)  .\]
Plugging these two facts in the previous inequality we have w.h.p
\begin{equation*}
 \sum_{t\in [T]} \|x_t- A_{i_t}\|^2
    \ge \max_{i\in [n]} \sum_{t\in[T]} \| A_i - x_t \|^2
            - O( \frac{\log n}{\eta} + \eta \sum_{t\in[T]} p_t\trans v_t^2 + T\eta ).
\end{equation*}
This holds w.h.p over the random choices of $\{i_t,j_t\}$, and irrespective of the coin tosses $\{c_t\}$. Hence, we can take expectations w.r.t $\{c_t\}$, and obtain
\begin{equation}\label{eq:weak reg + hp2}
\E_{\{c_t\}} [ \sum_{t\in [T]} \|x_t- A_{i_t}\|^2]
    \ge \E_{\{c_t\}} [ \max_{i\in [n]} \sum_{t\in[T]} \| A_i - x_t \|^2]
            - O( \frac{\log n}{\eta} + \eta \sum_{t\in[T]} p_t\trans v_t^2 + T\eta ).
\end{equation}
Combining with equation \eqref{eqn:gdPrimalDualMEB}, we obtain that w.h.p. over the random variables $\{i_t,j_t\}$
\[
T\sigma + \frac{4}{\alpha} \log T
    \ge \E_{\{c_t\}}[\max_{i\in [n]}\sum_{t\in[T]} \norm{x_t - A_i}^2] - O( \frac{\log n}{\eta} + \eta \sum_{t\in[T]} p_t\trans v_t^2 + T\eta )
\]
Rearranging and using Lemma \ref{lem:medprob_generic}, we have w.p. at least $\frac{1}{2}$
\[
\E_{\{c_t\}}[\max_{i\in [n]}\sum_{t\in[T]} \norm{x_t - A_i}^2]
    \le O( T \sigma + \frac{\log T}{\alpha}  + \frac{\log n}{\eta} +  T\eta )
\]
Dividing through by $T$ and applying Jensen's inequality, we have
\begin{equation*}\label{eq:MEB xbar}
\E [ \max_j    \norm{  \bar{x} - A_{j} }^2 ]
    \le \frac{1}{T}\E[\max_{i\in [n]}\sum_{t\in[T]} \norm{x_t - A_i}^2]
    \le O(  \sigma + \frac{\log T}{T \alpha}  + \frac{\log n}{T \eta} +  \eta ).
\end{equation*}
Optimizing over the values of $\alpha$, $\eta$, and $T$, this implies
that the expected error is $O(\eps)$, and so
using Markov's inequality, $\bar{x}$ is a $O(\veps)$-approximate solution with
probability at least 1/2.

\paragraph{Running time}
The algorithm above consists of
$T = O(\frac{\log n}{\veps^2})$ iterations. Naively, this would result in the same
running time as for linear classification. Yet notice that $x_t$ changes only
an expected $\alpha T$ times, and only then do we perform an $O(d)$ operation.
The expected number of iterations in which $x_t$ changes is $\alpha T \le 16 \eps^{-1}\log T$,
and so the running time is
\[
O(\eps^{-1}(\log T) \cdot d + \frac{\log n}{\veps^2} \cdot n))
    = \Otil( \veps^{-2} n  + \veps^{-1} d).
\]

\end{proof}
\fi 

The following Corollary is a direct analogue of Corollary~\ref{cor:perceptron dual}.
\begin{corollary}[Dual solution]\label{cor:MEB dual}
The vector $\bar{p}\equiv \sum_t e_{i_t}/T$ is, with probability
$1/2$, an $O(\varepsilon)$-approximate dual solution.
\end{corollary}

\subsection{High Success Probability and Las Vegas}

As for linear classification, we can amplify the success probability of \AlgMEB\  to $1-\delta$ for any $\delta > 0$ albeit incurring additional poly-log factors in
running time.

\begin{corollary}[MEB high probability]\label{cor:MEBhp}
There exists a randomized algorithm that with probability $1-\delta$ returns an
$\varepsilon$-approximate solution to the MEB problem, and
runs in expected time $\tilde{O}(\frac{n}{\eps^2}\log\frac{n}{\eps\delta} + \frac{d}{\eps}\log\frac{1}{\eps})$.
There is also a randomized algorithm that returns an $\eps$-approximate solution
in $\tO(M+\frac{n}{\eps^2} +\frac{d}{\eps})$ time.
\end{corollary}

\ifFOCS\else 
\begin{proof}

We can estimate the distance between two points in $\ball$
in $O(\eps^{-2}\log(1/\delta))$ time, with error at most $\eps$ and failure probability at most $\delta$,
using the dot product estimator described in \S\ref{subsec:highprob}.
Therefore we can estimate the maximum distance of a given point to every input
point in $O(n\eps^{-2}\log(n/\delta))$ time, with error at most $\eps$ and
failure probability at most $\delta$. This distance is $\sigma - \veps$, where $\sigma$ is the optimal radius attainable, w.p. $1-\delta$.

Because \AlgMEB\  yields an $\eps$-dual solution with probability 1/2, we can use this solution to verify that the radius of any possible solution to the farthest point is at least $\sigma - \veps$.

So, to obtain a solution as described in the lemma statement,
run \AlgMEB, and verify
that it yields an $\eps$-approximation, using this approximate dual solution;
with probability 1/2, this gives a verified $\eps$-approximation. Keep trying until
this succeeds, in an expected 2 trials.

For a Las Vegas algorithm, we simply apply the same scheme, but verify the distances
exactly.
\end{proof}
\fi 

\subsection{Convex Quadratic Programming in the Simplex}\label{subsec:quad}

We can extend our approach to problems of the form
\begin{equation}\label{eq:FW primal}
\min_{p\in\Delta} p\trans b + p\trans AA\trans p,
\end{equation}
where $b\in\reals^n$, $A\in\reals^{n\times d}$, and $\Delta$ is, as usual, the
unit simplex in $\reals^n$.
\ifFOCS
\else
As is well known, and as we partially review below,
\fi
this problem includes the MEB problem, margin estimation as for hard margin
support vector machines, the $L_2$-SVM variant of support vector machines, the
problem of finding the shortest vector in a polytope, and others.
\ifFOCS
\else

Applying
 $\norm{v-x}^2 = v\trans v + x\trans x-2v\trans x \ge 0$ with
 $v \gets A\trans p$, we have
\begin{equation}\label{eq:max quad}
\max_{x\in\reals^d} 2p\trans Ax - \norm{x}^2
    = p\trans AA\trans p,
\end{equation}
with equality at $x=A\trans p$. Thus
(\ref{eq:FW primal}) can be written as
\begin{equation}\label{eq:FW primalp}
\min_{p\in\Delta} \max_{x\in\reals^d} p\trans (b + 2 Ax - \vecc{1}_n\norm{x}^2).
\end{equation}
The \emph{Wolfe dual} of this problem exchanges the max and min:
\begin{equation}\label{eq:FW dualp}
    \max_{x\in\reals^d} \min_{p\in\Delta} p\trans (b + 2Ax - \vecc{1}_n\norm{x}^2).
\end{equation}
Since
\begin{equation}\label{eq:min p}
\min_{p\in\Delta} p\trans (b + 2Ax - \vecc{1}_n\norm{x}^2)
    = \min_i b(i) + 2A_i x + \norm{x}^2,
\end{equation}
with equality when $p_{\hat{i}} = 0$ if $\hat{i}$ is not a minimizer,
the dual can also be expressed as
\begin{equation}\label{eq:FW dual}
    \max_{x\in\reals^d} \min_i b(i) + 2A_i x - \norm{x}^2
\end{equation}
By the two relations \eqref{eq:max quad} and
\eqref{eq:min p} used to derive the dual problem from the primal,
we have immediately the \emph{weak duality} condition that the objective
function of the dual \eqref{eq:FW dual} is always no more than
the objective function value of the primal \eqref{eq:FW primal}.
The strong duality condition, that the two problems
take the same optimal value, also holds here; indeed,
the optimum $x_*$ also solves \eqref{eq:max quad}, and
the optimal $p_*$ also solves
\eqref{eq:min p}.

To generalize \AlgMEB, we
make $v_t$ an unbiased estimator of $b + 2Ax_t - \vecc{1}_n\norm{x_t}^2$,
and set $x_{t+1}$ to be the minimizer of
\[
\sum_{t'\in[t]} b(i_{t'}) + 2A_{i_{t'}}x_{t'} - \norm{x_{t'}}^2,
\]
namely, as with MEB, $y_{t+1} \gets \sum_{t'\in[t]} A_{i_{t'}}$,
and $x_{t+1}\gets y_{t+1}/t$.
(We also make some sign changes to account for
the max-min formulation here, versus the min-max formulation used for MEB
above.) This allows the use of Lemma~\ref{lem:OGDStrictlyConvexLazy} for essentially
the same analysis as for MEB; the gradient bound
$G$ and Hessian bound $H$ are both at most $2$, again assuming that all
$A_i\in\ball$.

\paragraph{MEB}
When the $b(i)\gets -\norm{A_i}^2$, we have
\[
-\max_{x\in\reals^d} \min_i b(i) + 2A_i x - \norm{x}^2
    = \min_{x\in\reals^d} \max_i \norm{A_i}^2 - 2A_i x  + \norm{x}^2
    = \min_{x\in\reals^d} \max_i \norm{x - A_i}^2,
\]
the objective function for the MEB problem.

\paragraph{Margin Estimation}
When $b\gets 0$ in the primal problem (\ref{eq:FW primal}),
that problem is one of finding the shortest vector in the
polytope $\{A\trans p\mid p\in\Delta\}$. Considering this case
of the dual problem (\ref{eq:FW dual}), for any given $x\in\reals^d$
with $\min_i A_i x \le 0$, the value of $\beta\in\reals$ such
that $\beta x$ maximizes
 $\min_i 2A_i \beta x - \norm{\beta x}^2$
is $\beta = 0$. On the other hand if $x$ is such that
$\min_i A_i x > 0$,
the maximizing value $\beta$ is $\beta = A_ix/\norm{x}^2$, so that the solution
of (\ref{eq:FW dual}) also maximizes $\min_i (A_ix)^2/\norm{x}^2$. The latter is
the square of the margin $\sigma$, which as before is the minimum distance of
the points $A_i$ to the hyperplane that is normal to $x$ and passes through the
origin.

Adapting \AlgMEB\  for margin estimation,
and with the slight changes needed for its analysis, we have that there is an
algorithm taking $\Otil(n/\epsilon^2+d/epsilon)$ time that finds $\bar{x}\in\reals^d$
such that, for all $i\in[n]$,
\[
2A_i\bar{x} - \norm{\bar{x}}^2 \ge \sigma^2 - \epsilon.
\]
When $\sigma^2\le\epsilon$, we don't appear to gain any useful information.
However, when $\sigma^2>\epsilon$, we have $\min_{i \in [n]} A_i\bar{x} > 0$,
and so, by appropriate scaling of $\bar{x}$, we have $\hat{x}$ such that
\[
\hat{\sigma}^2
    = \min_{i \in [n]} (A_i\hat{x})^2/\norm{\hat{x}}^2
    = \min_{i \in [n]} 2A_i\hat{x} - \norm{\hat{x}}^2
    \ge \sigma^2 - \epsilon,
\]
and so $\hat{\sigma} \ge \sigma - \epsilon/\sigma$. That is, letting
 $\epsilon \equiv \epsilon'\sigma$,
if $\epsilon'\le\sigma$, there is an algorithm taking
 $\Otil(n/(\epsilon\sigma)^2+d/\epsilon'\sigma)$ time that finds a solution $\hat{x}$ with
$\hat{\sigma} \ge \sigma - \epsilon'$.
\fi 

\section{A Generic Sublinear Primal-Dual Algorithm}

We note that our technique above can be applied more broadly to any constrained optimization problem for which low-regret algorithms exist and low-variance sampling can be applied efficiently; that is, consider the general
problem with optimum $\sigma$:
\begin{equation} \label{eqn:generic_opt_problem}
\max_{x \in \K} \min_i  c_i(x) = \sigma.
\end{equation}
Suppose that for the set $\K$ and cost functions $c_i(x)$, there exists an iterative low regret algorithm, denoted $LRA$, with regret $R(T) = o(T)$. Let $T_\varepsilon(LRA)$ be the smallest $T$ such that $\frac{R(T)}{T} \leq \varepsilon$. We denote by $x_{t+1} \leftarrow LRA(x_t,c)$ an invocation of this algorithm, when at state $x_t \in \K$ and the cost function $c$ is observed.

Let $\sample(x,c)$ be a procedure that returns an unbiased estimate of $c(x)$ with variance at most one,
that runs in constant time. Further assume $|c_i(x)| \leq 1$ for all $x \in K \ , \ i \in [n]$.

\ifFOCS 
  \begin{figure}[!t]
\else
  \begin{algorithm}[h!]
 \caption{Generic  Sublinear Primal-Dual Algorithm}
\fi
    \begin{algorithmic}[1]
    \STATE Let $T \leftarrow \max \{T_\varepsilon(LRA), \frac{\log n}{\varepsilon^2} \}$ ,
       \\ $x_1 \gets LRA(\mbox{initial})$, $w_1 \gets \vecc{1}_n$, $\eta\gets \frac{1}{100}  \sqrt{\frac{\log n}{T}}$.
    \FOR{$t=1$ to $T$}
        \FOR{$i\in [n]$}
            \STATE Let $v_t(i) \gets \sample(x_t,c_i)$
            \STATE $v_t(i) \gets \clip(\vtil_t(i), 1/\eta)$
            \STATE $w_{t+1}(i) \gets w_t(i) (1-\eta v_t(i) + \eta^2 v_t(i)^2)$
        \ENDFOR
        \STATE $\ p_t \gets \frac{w_t}{\norm{w_t}_1}$,
        \STATE Choose $i_t\in [n]$ by $i_t\gets i$ with probability $p_t(i)$.
        \STATE $\ x_t \gets LRA(x_{t-1},c_{i_{t}})$
    \ENDFOR
    \RETURN $\bar{x} = \frac{1}{T} \sum_t x_t $
    \end{algorithmic}
 \ifFOCS 
  \caption{Sublinear Primal-Dual Generic Algorithm}\label{alg:generic}
  \end{figure}
\else
   \label{alg:generic}
   \end{algorithm}
\fi

Applying the techniques of section \ref{sec:perceptron} we can obtain the following generic lemma.
\begin{Lem} \label{lem:generic}
The generic sublinear primal-dual  algorithm returns a solution $x$ that with probability at least $\frac{1}{2}$ is an $\varepsilon$-approximate solution in $\max \{T_\varepsilon(LRA), \frac{\log n}{\varepsilon^2} \}$ iterations.
\end{Lem}
\ifFOCS\else 
\begin{proof}

First we use the regret bounds for LRA to lower bound
$\sum_{t\in [T]}  c_{i_t}(x_t)$, next we get an upper bound for that
quantity using the Weak Regret Lemma, and then we combine the two
in expectation.

By definition, $c_i(x^*) \ge \sigma$ for all $i\in [n]$, and so, using
the LRA regret guarantee,
\begin{equation}\label{eqn:Ait lower_generic}
T\sigma
    \le \max_{x\in\ball} \sum_{t\in [T]}  c_{i_t}(x)
    \le \sum_{t\in [T]}  c_{i_t}(x_t) + R(T) ,
\end{equation}
or rearranging,
\begin{equation}\label{eqn:gdPrimalDualgeneric}
\sum_{t\in [T]}  c_{i_t}(x_t) \ge T\sigma - R(T) .
\end{equation}

Now we turn to the MW part of our algorithm.
By the Weak Regret Lemma~\ref{lem:regretMW},
and using the clipping of $v_t(i)$,
\[
\sum_{t\in [T]} p_t\trans  v_t
    \le \min_{i\in [n]} \sum_{t\in[T]} v_t(i)
            + (\log n)/\eta + \eta \sum_{t\in[T]} p_t\trans v_t^2.
\]
Using Lemma~\ref{lem:genericHP_v_and_mu} and Lemma~\ref{lem:highprob2_generic}, since the procedure $\sample$ is unbiased and has variance at most one, with high probability:
\[
\forall i\in [n] \ \ , \ \  \sum_{t\in [T]} v_t(i)
    \le \sum_{t\in [T]} c_i(x_t)   + O( \eta T),
\]
\[ \left| \sum_{t\in [T]} c_{i_t}(x_t) - \sum_t p_t
\trans v_t \right| = O(\eta T)  .\]
Plugging these two facts in the previous inequality we have w.h.p,
\begin{eqnarray}\label{eq:weak reg + hp3}
\sum_{t\in [T]} c_{i_t}(x_t)
    \le \min_{i\in [n]} \sum_{t\in[T]} c_{i}(x_t)
            + O(\frac{\log n}{\eta} + \eta \sum_{t\in[T]} p_t\trans v_t^2 + \eta T)
\end{eqnarray}
Combining \eqref{eqn:gdPrimalDualgeneric} and \eqref{eq:weak reg + hp3} we get w.h.p
\begin{align*}
\min_{i\in [n]} \sum_{t\in[T]} c_{i}( x_t)
    \geq - O( \frac{\log n}{\eta} + \eta T + \eta \sum_{t\in[T]} p_t\trans v_t^2 ) - R(T)
\end{align*}
And via Lemma \ref{lem:medprob_generic} we have w.p. at least $\frac{1}{2}$ that
\begin{align*}
\min_{i\in [n]} \sum_{t\in[T]} c_{i}( x_t)
    \geq - O( \frac{\log n}{\eta} + \eta T  ) - R(T)
\end{align*}
Dividing through by $T$, and using our choice of $\eta$,
we have $\min_i c_i \bar{x} \ge \sigma - \eps/2$ w.p. at least least $1/2$ as claimed.
\end{proof}
\fi 
High-probability results can be obtained using the same technique as for linear classification.

\subsection{More applications}

The generic algorithm above can be used to derive the result of Grigoriadis and Khachiyan
\cite{GriKha95} on sublinear approximation of zero sum games with payoffs/losses bounded by one (up to poly-logarithmic factors in running time). A zero sum game can be cast as the following min-max optimization problem:
$$ \min_{x \in \Delta_d} \max_{i \in \Delta_n} A_i x $$
That is, the constraints are inner products with the rows of the game matrix. This is exactly the same as the linear classification problem, but the vectors $x$ are taken from the convex set $\K$ which is the simplex - or the set of all mixed strategies of the column player.

A low regret algorithm for the simplex is the multiplicative weights algorithm, which attains regret $R(T) \leq 2 \sqrt{T \log n}$. The procedure $\sample(x,A_i)$ to estimate the inner product $A_i x$ is much simpler than the one used for linear classification: we sample from the distribution $x$ and return $A_i(j)$ w.p. $x(j)$. This has correct expectation and variance bounded by one (in fact, the random variable is always bounded by one). Lemma \ref{lem:generic} then implies:

\begin{corollary}
The sublinear primal-dual algorithm applied to zero sum games returns a solution $x$ that with probability at least $\frac{1}{2}$ is an $\varepsilon$-approximate solution in $O(\frac{\log n}{\varepsilon^2}) $ iterations and total time $\tilde{O}(\frac{n+d}{\varepsilon^2})$.
\end{corollary}

Essentially any constrained optimization problem which has convex or linear constraints, and is over a simple convex body such as the ball or simplex, can be approximated in sublinear time using our method. The particular application to soft margin SVM, together with its practical significance, is explored in ongoing work with Nati Srebro.

\section{A Semi-Streaming Implementation} \label{sec:streaming}
In order to achieve space that is sublinear in $d$, we cannot afford to
output a solution vector. We instead output both the cost of the solution, and
a set of indices $i_1, \ldots, i_t$ for which the solution is
a linear combination (that we know) of $A_{i_1}, \ldots, A_{i_t}$. We note that all
previous algorithms for these problems, even to achieve this notion of output,
required $\Omega(d)$ space and/or $\Omega(nd)$ time, see, e.g., the references
in \cite{as10}.

We discuss the modifications to the sublinear primal-dual algorithm that need to be done
for classification and minimum enclosing ball problems.

\def\CB{P}

Our algorithm assumes it sees entire points at a time, i.e., it sees the entries
of $A$ row at a time, though the rows may be ordered arbitrarily. It relies on two streaming results about a $d$-dimensional vector $x$
undergoing updates to its coordinates. We assume that each update is of the form
$(i, z)$, where $i \in [d]$ is a coordinate of $x$ and $z \in \{-\CB, -\CB+1,
\ldots, \CB\}$ indicates that $x_i \leftarrow x_i + z$. The first is an efficient
$\ell_2$-sketching algorithm of Thorup and Zhang. This algorithm allows for
$(1+\eps)$-approximation of $\|x\|_2$ with high probability using $1$-pass,
$\tilde{O}(\eps^{-2})$ space, and time proportinal to the length of the stream.

\ifFOCS\else
\begin{theorem}(\cite{TZ})\label{thm:l2streaming}
There is a $1$-pass algorithm which outputs a $(1 \pm \eps)$-approximation to
$\|x\|_2$ with probability $\geq 1-\delta$ using
 $O(\eps^{-2}\log (\CB dQ) \log 1/\delta)$ bits of space and
$O(Q \log 1/\delta)$ time, where $Q$ is the total
number of updates in the stream.
\end{theorem}
\fi
The second component is due to Monemizadeh and Woodruff \cite{mw10}. We are
given a stream of updates to a $d$-dimensional 
vector $x$, and want to output a
random coordinate $I \in [d]$ for which for any $j \in [d]$,
 $\Pr[I = j] = \frac{|x_j|^2}{\|x\|_2^2}$.
We also want the algorithm to return the value
$x_I$. Such an algorithm is called an exact augmented $\ell_2$-Sampler.
As shown in \cite{mw10}, an augmented $\ell_2$-Sampler with $O(\log d)$
space, $\tilde{O}(1)$ passes, and running time $\tilde{O}(Q)$ exists, where
$Q$ is the number of updates in the stream.
This is what we use to $\ell_2$-sample from an iterate vector
that we can only afford to represent implicitly.

\ifFOCS\else
\begin{theorem}(Theorem 1.3 of \cite{mw10})\label{thm:streamSample}
There is an $O(\log d)$-pass exact augmented $\ell_2$-Sampler
that uses $O(\log^5 (\CB d))$ bits of space and has running time
$Q\log^{O(1)}(PdQ)$, where $Q$ is the total number of updates in the stream. The
algorithm fails with probability $\leq d^{-c}$ for an arbitrarily large constant
$c > 0$.
\end{theorem}
\fi
We maintain the indices $i_t$ and $j_t$ used in all $\tilde{O}(\eps^{-2})$
iterations of the primal dual algorithm. Notice that in a single iteration $t$ the same
$\ell_2$-sample index $j_t$ can be used for all $n$ rows. While we cannot afford
to remember the probabilities in the dual vector, we can store the values
$\frac{\alpha_t}{x_t(j)}$,
where $\alpha_t$ is a $(1\pm \eps)$-approximation of $\|x_t\|^2$ which can be obtained
using the Thorup-Zhang sketch. We also need such an approximation to $\|x_t\|$ to
appropriately weight the rows used to do $\ell_2$-sampling (see below). Since
we see rows (i.e., points) of $A$ at a time, we can reconstruct the probability of each
row in the dual vector on the fly in low space, and can use reservoir sampling to
make the next choice of $i_t$. Then we use an augmented $\ell_2$-sampler to make the
next choice of $j_t$, where we must $\ell_2$ sample from a weighted sum of rows indexed by
$i_1, \ldots, i_t$ in low space.
\ifFOCS\else
We use the fact argued in \S\ref{subsec:precision}
\fi
We can show that the algorithm remains correct
given the per-iteration rounding of the updates $v_t(i)$ to relative error $\mu$,
where $\mu$ is on the order of $\eta\epsilon/T$.
\ifFOCS\else
Throughout we round matrix entries to the nearest integer multiple of
$\textrm{poly}(1/d)$ for a sufficiently large polynomial.

We implicitly represent the primal and dual vectors. At iteration $t$ of the
sublinear primal-dual algorithm, we have indices $i_1, \ldots, i_{t-1}$ of the sampled
rows
and indices
$j_1, \ldots, j_t$ of the sampled columns for $\ell_2$-sampling (in a given
iteration $t$, we use the same column $j_t$ for $\ell_2$-sampling from all rows).
We maintain $\mu/2$-approximations
$\frac{1}{\tilde{x}_1(j_1)}, \ldots, \frac{1}{\tilde{x}_{t-1}(j_{t-1})}$
to
$\frac{1}{x_1(j_1)}, \ldots, \frac{1}{x_{t-1}(j_{t-1})}$.
We compute $i_t, j_{t+1}$, and
a $\mu/2$-approximation $\frac{1}{\tilde{x}_{t}(j_t)}$ to $\frac{1}{x_t(j_t)}$.

We first determine $i_t$ in one pass. This can be done since
$A$ is presented in row order, together with reservoir sampling. Namely, given
row $A_k$, we compute for each $1 \leq t' \leq t-1$, a
$\mu$-approximation
$\tilde{v}_{t'}(k) = A_k(j_{t'}) \cdot \frac{1}{\tilde{x}_{t'}(j_{t'})}$ to
 $v_{t'}(k) = A_k(j_{t'}) \cdot \frac{1}{x_{t'}(j_{t'})}$,
and then
\[
\ptil_{t}(k)
    = \frac{1}{n} \cdot \prod_{t'=1}^{t-1} (1 + \eta \tilde{v}_{t'}(k) + \eta^2 \tilde{v}^2_{t'}(k)).
\]
Thus, we can reconstruct $\ptil_{t}(k)$ for use with reservoir sampling
to obtain a sample $i_t$.

In the next $O(\log n)$ passes we obtain $j_{t+1}$ as follows. To
$\ell_2$-sample from $x_t$, we use Theorem~\ref{thm:streamSample}
to sample a coordinate from the length-$(t-1)d$ stream consisting of the entries
of the concatenated list:
 $L = A_{i_1}, A_{i_2}, \ldots, A_{i_{t-1}}$.
Notice that $y_t = \frac{1}{\sqrt{2T}} \cdot \sum_{j=1}^{t-1} A_{i_j}$, and so
Theorem~\ref{thm:streamSample} applied to $L$ implements $\ell_2$-sampling
from $x_t$. However, the algorithm returns $y_t(j_t)$ rather than $x_t(j_t)$.
To obtain an approximation to $x_t(j_t)$,
we $(\epsilon/3)$-approximate $\norm{y_t}$
using Theorem \ref{thm:l2streaming}, from
which $x_t(j_t) = \frac{y_t(j_j)}{\max\{1, \norm{y_t}\}}$.
%
%
We thus obtain an $(\epsilon/2)$-approximation
$\frac{1}{\tilde{x}_t(j_t)}$ to $\frac{1}{x_{t}(j_t)}$.

Using Lemma \ref{lem:lazyogd}, letting
$y_{T+1} = \frac{1}{\sqrt{2T}} \sum_{j=1}^T A_{i_j}$, then
$x_{T+1} = \frac{y_{T+1}}{\max\{1, \|y_{T+1}\|\}}$ results in an additive $\eps$
approximation. To compute this, we must
$(1 \pm \eps)$-approximate $\|y_{T+1}\|$, which
we do in an additional pass using Theorem \ref{thm:l2streaming}
Note that we cannot afford $d$ space, which would be required to
compute the norm exactly.
\fi
\begin{theorem}
There is an $\Otil(\eps^{-2})$-pass, $\Otil(\eps^{-2})$-space algorithm running in total
time $\Otil(\eps^{-4}(n+d))$ which returns a list of $T = \Otil(\eps^{-2})$ row indices $i_1,
\ldots, i_T$ which implicitly represent the normal vector to a hyperplane for
$\eps$-approximate classification, together with an additive-$\eps$ approximation to
the margin.
\end{theorem}
For the MEB problem with high probability there are only $\tilde{O}(\eps^{-1})$ different
values of $i_t$ (i.e., updates to the primal vector).
An important point is that we can get all $\tilde{O}(\eps^{-1})$ $\ell_2$-samples
independently from the same primal vector between changes to it
by running the algorithm of \cite{mw10} independently
$\tilde{O}(\eps^{-1})$ times in parallel.

We spend
$\tilde{O}((n+d)\eps^{-2})$ time per iteration, to reconstruct the dual vector and
run the algorithm of \cite{mw10} independently $\tilde{O}(\eps^{-1})$ times on
a stream of length $\tilde{O}(d\eps^{-1})$ to do $\ell_2$-sampling).
%
%
\ifFOCS\else
\paragraph{Minimum Enclosing Ball}
For the MEB problem we need the following standard tool.
\begin{fact}(see, e.g., \cite{ams99})\label{fact:ams}
Let $\sigma \in \{-1,1\}^d$ be uniform from a $4$-wise independent family of sign vectors.
For any $n$-dimensional vector $v$, ${\bf E}_{\sigma}[\langle \sigma, v \rangle^2] = \|v\|_2^2$
and ${\bf Var}_{\sigma}[\langle \sigma, v \rangle^2] \leq 2\|v\|_2^4$.
\end{fact}
Define an epoch to be a contiguous
block of iterations for which $x_t$ does not change. Notice that $x_t$ does not
change with probability $1-\alpha$.

We describe the necessary modifications to \AlgMEB.
Throughout we round matrix entries to the nearest integer multiple of poly$(1/d)$ for a
sufficiently large polynomial. We use the fact argued
in \S\ref{subsec:precision} that the algorithm remains correct given the per-iteration rounding
of the updates $v_t(i)$ to relative error $\mu$,
where $\mu$ is on the order of $\eta\epsilon/T$.

We will not compute $\norm{x_t}$
in each epoch. This would require $\Omega(d)$ space. However, unlike in the case of
classification, for the MEB problem we cannot even afford to use Theorem \ref{thm:l2streaming}
to approximate $\norm{x_t}^2$, as that would cost $\Omega(\eps^{-2})$ space.
Instead, we will use Fact \ref{fact:ams}
to obtain an unbiased estimator of $\norm{x_t}^2$, which suffices for our analysis to go through.
Namely, by the triangle inequality, $\norm{x_t} \leq 1$ (since we divide by $t$), and so
the estimator of Fact \ref{fact:ams} has variance $O(1)$.

Again, we implicitly represent the primal and dual vectors.
We only store one index $i_s$ and $j_s$ per epoch $s$.
In epoch $s$, we have
indices $i_1, \ldots, i_{s}$, which correspond to indices of the row
$A_i$ chosen for use to update the primal vector in the current and previous epochs.
As in the non-streaming version of this algorithm, we use the same coordinate $j_s$
for $\ell_2$-sampling in all iterations in an epoch and for all rows.
Hence, throughout the course of the algorithm,
the expected number of indices $i_s$ and $j_s$ that the algorithm
stores is the number $\tilde{O}(\alpha T) = \tilde{O}(\eps^{-1})$ of epochs.
The algorithm also stores the number $m_s$ of iterations in each epoch in the same
amount of space.

At the beginning of the $s$-th epoch, we have maintained $\mu/2$-approximations
$\frac{1}{\tilde{x}_1(j_1)}, \ldots, \frac{1}{\tilde{x}_{s-1}(j_{s-1})}$
to
$\frac{1}{x_1(j_1)}, \ldots, \frac{1}{x_{s-1}(j_{s-1})}$.
We compute $i_{s}, j_{s}$, and
a $\mu/2$-approximation $\frac{1}{\tilde{x}_{s}(j_s)}$ to $\frac{1}{x_s(j_s)}$.

We first determine $i_s$ in one pass. This can be done as in classification since
$A$ is presented in row order, together with reservoir sampling. Namely, given
row $A_k$, we compute for each $1 \leq s' \leq s-1$, a
$\mu$-approximation
$\tilde{v}_{s'}(k) = A_k(j_{s'}) \cdot \frac{1}{\tilde{x}_{s'}(j_{s'})}$ to
 $v_{s'}(k) = A_k(j_{s'}) \cdot \frac{1}{x_{s'}(j_{s'})}$,
and then
\[
\ptil_{s}(k)
    = \frac{1}{n} \cdot \prod_{s'=1}^{s-1} (1 + \eta \tilde{v}_{s'}(k) + \eta^2 \tilde{v}^2_{s'}(k))^{m_{s'}}.
\]
Thus, we can reconstruct $\ptil_{s}(k)$ for use with reservoir sampling
to obtain a sample $i_s$.

In the next $O(\log n)$ passes we obtain $j_{s}$ as follows. To
$\ell_2$-sample from $x_s$, we use Theorem~\ref{thm:streamSample}
to sample a coordinate from the length-$(s-1)d$ stream consisting of the entries
of the concatenated list:
 $L = A_{i_1}, A_{i_2}, \ldots, A_{i_{s-1}}$.
Notice that $y_s = \sum_{j=1}^{s-1} A_{i_j}$, and so
Theorem~\ref{thm:streamSample} applied to $L$ implements $\ell_2$-sampling
from $y_s$, and hence $x_s$ as well. We thus obtain an $(\epsilon/2)$-approximation
$\frac{1}{\tilde{x}_s(j_s)}$ to $\frac{1}{x_{s}(j_s)}$.

Applying Fact \ref{fact:ams},
we obtain
\fi

\begin{theorem}
Given the norms of each row $A_i$,
there is an $\Otil(\eps^{-1})$-pass, $\Otil(\eps^{-2})$-space algorithm running in total
time $\Otil(\eps^{-3}(n+d))$ which returns a list of $T = \Otil(\eps^{-1})$ row indices $i_1,
\ldots, i_T$ which implicitly represent the MEB center, together with an
additive $\eps$-approximation to the MEB radius.
\end{theorem}

\section{Kernelizing the Sublinear algorithms}\label{sec:kernel-long}

An important generalization of linear classifiers is that of kernel-based linear
predictors (see e.g. \cite{SchoSmola03}). Let $\Psi : \reals^d \mapsto \mH $ be
a mapping of feature vectors into a reproducing kernel Hilbert space. In this
setting, we seek a non-linear classifier given by $h \in \mH$ so as to maximize
the margin:
$$ \sigma \equiv \max_{h \in \mH }\min_{i \in [n] } \langle h , \Psi(A_i ) \rangle.$$
The kernels of interest are those for which we can compute inner products of the
form $k(x,y) = \langle \Psi(x), \Psi(y) \rangle $ efficiently.

One popular kernel is the polynomial kernel, for which the corresponding Hilbert
space is the set of polynomials over $\reals^d$ of degree $q$. The mapping
$\Psi$ for this kernel is given by
$$\forall S \subseteq [d] \ , \ |S| \leq q \ . \ \Psi(x)_S = \prod_{i \in S} x_i.$$
That is, all monomials of degree at most $q$. The kernel function in this case
is given by $k(x,y) = (x\trans y)^q$. Another useful kernel is the Gaussian
kernel $k(x,y) = \exp(-\frac{\norm{x-y}^2}{2 \ksig^2})$, where $\ksig$ is a
parameter. The mapping here is defined by the kernel function (see
\cite{SchoSmola03} for more details).

\ifFOCS\else 

  \begin{algorithm}[h!]
  \caption{Sublinear Kernel Perceptron}
    \begin{algorithmic}[1]
    \STATE Input: $\eps>0$, $A \in \reals^{n \times d}$ with $A_i\in\ball$ for $i\in [n]$.
    \STATE Let $T \gets 200^2\eps^{-2}\log n$, $y_1 \gets 0$, $w_1 \gets \vec{1}_n$, $\eta\gets \frac{1}{100}  \sqrt{\frac{\log n}{T}}$.
    \FOR{$t=1$ to $T$}
    \STATE $\ p_{t} \gets \frac{w_{t}}{\norm{w_{t}}_1}$, $\ x_t \gets \frac{y_t}{\max\{1,\norm{y_t}\}}.$
    \STATE Choose $i_t\in [n]$ by $i_t\gets i$ with probability $p_t(i)$.
    \STATE $y_{t+1} \gets \sum_{\tau\in[t]} \Psi(A_{i_\tau})/\sqrt{2T}$.
      \FOR{$i\in [n]$}
        \STATE $\vtil_t(i) \leftarrow \KernelEllTwo(  x_t ,  \Psi(A_i) )$.  (estimating $\langle x_t, \Psi(A_i) \rangle$)
         \STATE $v_t(i) \gets \clip(\vtil_t(i), 1/\eta)$.
        \STATE $w_{t+1}(i) \gets w_t(i) (1-\eta v_t(i) + \eta^2 v_t(i)^2)$.
      \ENDFOR
    \ENDFOR
    \RETURN $\bar{x} = \frac{1}{T} \sum_t x_t $
    \end{algorithmic}
   \label{alg:kernel}
   \end{algorithm}

The kernel version of \AlgPDP\  is shown in Figure~\ref{alg:kernel}.
Note that $x_t$ and $y_t$ are members of $\mH$, and not maintained explicitly,
but rather are implicitly represented by the values $i_t$. (And thus
$\norm{y_t}$ is the norm of $\mH$, not $\reals^d$.) Also, $\Psi(A_i)$ is not
computed. The needed kernel product $ \langle x_t , \Psi(A_i) \rangle$ is
estimated by the procedure $\KernelEllTwo$, using the implicit representations
and specific properties of the kernel being used. In the regular sublinear
algorithm, this inner product could be sufficiently well approximated in $O(1)$
time via $\ell_2$-sampling. As we show below, for many interesting kernels the
time for $\KernelEllTwo$ is not much longer.

For the analog of Theorem \ref{thm:main} to apply, we need the expectation of
the estimates $v_t(i)$ to be correct, with variance $O(1)$. By
Lemma~\ref{lem:approxOGD}, it is enough if the estimates $v_t(i)$ have an
additive bias of $O(\epsilon)$. Hence, we define the procedure $\KernelEllTwo$
to obtain such an not-too-biased estimator with variance at most one; first
we show how to implement $\KernelEllTwo$, assuming that there
is an estimator $\ktil()$ of the kernel $k()$ such that
$\E[\tilde{k}(x,y)] = k(x,y)$ and $\var(\tilde{k}(x,y)) \leq 1$,
and then we show how to implement such kernel estimators.

\subsection{Implementing $\KernelEllTwo$}

\paragraph{Estimating $\norm{y}_t$}
A key step in $\KernelEllTwo$ is the estimation of $\norm{y_t}$,
which readily reduces to estimating
\[
Y_t \equiv 2T\norm{y_t}^2/t^2 = \frac{1}{t^2} \sum_{\tau, \tau'\in[t]} k(A_{i_\tau}, A_{i_{\tau'}}),
\]
that is, the mean of the summands. Since we use $\max\{1,\norm{y_t})$, we need
not be concerned with small $\norm{y_t}$, and it is enough that the additive
bias in our estimate of $Y$ be at most $\epsilon/T \le \epsilon(2T/t^2)$ for $t\in [T]$,
 implying
a bias for $\norm{y_t}$ no more than $\epsilon$. Since we need
$1/\norm{y_t}$ in the algorithm, it is not enough for estimates of $Y$ just to
be good in mean and variance; we will find an estimator whose error bounds hold
with high probability.

Our estimate $\Ytil_t$ of $Y_t$ can first be considered assuming we only
need to make an estimate for a single value of $t$.

Let $N_Y \gets t^2\ceil{(8/3)\log(1/\delta)T^2/\epsilon^2t^2}$. To estimate $Y_t$, we
compute, for each $\tau, \tau'\in[t]$, $n_t \gets N_Y/t^2$ independent estimates
\[
X_{\tau,\tau',m} \gets \clip(\ktil(A_{i_\tau}, A_{i_{\tau'}}), T/\epsilon), \mathrm{ for }\ m\in [n_t],
\]
and our estimate is
\[
\Ytil_t \gets \sum_{\substack{\tau, \tau'\in [t]\\ m\in[n_t]}} X_{\tau,\tau',m}/N_Y.
\]

\begin{lemma}\label{lem:Y est}
With probability at least $1-\delta$, $|Y-\Ytil_t| \le \epsilon/T$.
\end{lemma}

\begin{proof}
We apply Bernstein's inequality (as in \ref{eq:Bernstein}) to the
$N_Y$ random variables $X_{\tau,\tau',m} - \E[X_{\tau, \tau', m}]$.
which have mean zero, variance at most one, and are at most
$T/\epsilon$ in magnitude. Bernstein's inequality implies, using
$\Var[X_{\tau,\tau',m}] \le 1$,
\[
\log\Prob{\sum_{\substack{\tau,\tau'\in [t]\\ m\in[n_t]}} (X_{\tau,\tau',m} - \E[X_{\tau, \tau', m}]) > \alpha}
	\le -\alpha^2/(N_Y + (T/\epsilon)\alpha/3),
\]
and putting $\alpha \gets N_Y\epsilon/T$ gives
\begin{align*}
\log\Prob{\Ytil - \E[\Ytil] > \epsilon/T}
	   & \le -N_Y^2(\epsilon/T)^2 / (N_Y + (T/\epsilon)N_Y(\epsilon/T)/3)
	\\ & \le -(8/3)\log(1/\delta)(3/4) \le -2\log(1/\delta).
\end{align*}
Similar reasoning for $-X_{\tau,\tau',m}$, and the union bound, implies the lemma.
\end{proof}

To compute $Y$ for $t=1\ldots T$, we can save some work by reusing
estimates from one $t$ to the next.  Now let
$N_Y \gets \ceil{(8/3)\log(1/\delta)T^2/\epsilon^2}$.
Compute $\Ytil_1$ as above for $t=1$, and
let $\hat{Y}_1 \gets \Ytil_1$.
For $t>1$, let $n_t \gets \ceil{N_Y/t^2}$,
and let
\[
\hat{Y}_t \gets \sum_{m\in[n_t]} X_{t,t,m}/n_t
	+ \sum_{\substack{\tau \in [t]\\ m\in[n_t]}} (X_{t, \tau, m}+ X_{\tau,t,m})/n_t,
\]
and return $\Ytil_t \gets \sum_{\tau\in[t]} \hat{Y}_\tau/t^2$.

Since for each $\tau$ and $\tau'$, the expected total contribution
of all $X_{\tau,\tau', m} $ terms to $\Ytil_t$
is $k(A_{i_\tau}, A_{i_{\tau'}})$, we have $\E[\Ytil_t] = Y_t$.
Moreover, the number of instances of $X_{\tau,\tau', m} $ averaged to compute $\Ytil_t$
is always at least as large as the number used for the above ``batch''
version; it follows that the total variance of $\Ytil_t$ is non-increasing
in $t$, and therefore Lemma~\ref{lem:Y est} holds also for the
$\tilde{Y}_t$ computed stepwise.

Since the number of calls to $\ktil(,)$ is $\sum_{t\in [T]} (1+2n_t) = O(N_Y) $,
we have the following lemma.

\begin{lemma}
The values $\tilde{Y}_t (t^2/2T) \approx \norm{y_t}$, $t\in [T]$, can be estimated with
$O((\log(1/\epsilon\delta)T^2/\epsilon^2)$
calls to $\ktil(,)$, so that with probability at least $1-\delta$,
$|\tilde{Y}_t(t^2/2T) - \norm{y_t}|\le \epsilon$.
The values $\norm{y_t}$, $t\in[T]$, can be computed
exactly with $T^2$ calls to the exact kernel $k(,)$.
\end{lemma}

\begin{proof}
This follows from the discussion above, applying the union bound over $t\in[T]$,
and adjusting constants. The claim for exact computation is straightforward.
\end{proof}

Given this procedure for estimating $\norm{y_t}$, we can describe
$\KernelEllTwo$. Since $x_{t+1} = y_{t+1}/\max\{1,\norm{y_{t+1}}\}$,
we have
\begin{align}
\langle x_{t+1}, A_i \rangle
	\nonumber              & = \frac{1}{\max\{1,\norm{y_{t+1}}\} \sqrt{2T}} \sum_{\tau\in[t]} \langle \Psi(A_{i_\tau}), \Psi(A_i) \rangle
	\\ \label{eq:kernel}   & = \frac{1}{\max\{1,\norm{y_{t+1}}\} \sqrt{2T}} \sum_{\tau\in[t]} k(A_{i_\tau}, A_i),
\end{align}
so that the main remaining step is to estimate $\sum_{\tau\in[t]} k(A_{i_\tau}, A_i)$, for $i\in [n]$.
Here we simply call $\ktil(A_{i_\tau}, A_i)$ for each $\tau$. We save time,
at the cost of $O(n)$ space, by saving the value of the sum for each $i\in[n]$, and updating
 it for the next $t$ with $n$ calls $\ktil(A_{i_t}, A_i)$.

\begin{lemma}
Let $L_k$ denote the expected time needed for one call to $\ktil(,)$,
and $T_k$ denote the time needed for one call to $k(,)$.
Except for estimating $\norm{y_t}$,
$\KernelEllTwo$ can be computed in $nL_k$ expected time per iteration $t$.
The resulting estimate has expectation within additive $\epsilon$ of $\langle x_t, A_i \rangle$,
and variance at most one.
Thus \AlgKernel\  runs in time
 $\tilde{O}( \frac{(L_k n+d)}{\eps^2} + \min\{\frac{L_k}{\varepsilon^6}, \frac{T_k}{\varepsilon^4}\})$,
 and produces a solution with properties as in \AlgPDP.
\end{lemma}

\begin{proof}
For $\KernelEllTwo$ it remains only to show that its variance is at most one,
given that each $\ktil(,)$ has variance at most one. We observe from
(\ref{eq:kernel} that $t$ independent estimates $\ktil(,)$ are added together, and
scaled by a value that is at most $1/\sqrt{2T}$. Since the variance of the sum
is at most $t$, and the variance is scaled by a value no more than $1/2T$,
the variance of  $\KernelEllTwo$ is at most one. The only bias in the estimate
is due to estimation of $\norm{y_t}$, which gives relative error of $\epsilon$.
For our kernels, $\norm{\Psi(v)} \le 1$ if $v\in\ball$, so the additive
error of $\KernelEllTwo$ is $O(\epsilon)$.

The analysis of \AlgKernel\  then follows as for the un-kernelized
perceptron; we neglect the time needed for preprocessing for the calls
to $\ktil(,)$, as it is dominated by other terms for the kernels we consider,
and this is likely in general.
\end{proof}

\subsection{Implementing the Kernel Estimators}

Using the lemma above we can derive corollaries for the Gaussian and polynomial
kernels.

\paragraph{Polynomial kernels}
For the polynomial kernel of degree $q$, estimating a single kernel product,
i.e. $k(x,y) = k(A_{i},A_j)$, where the norm of $x,y$ is at most one, takes
$O(q)$ as follows:
Recall that for the polynomial kernel, $k(x,y) = (x\trans y)^q$. To estimate this
kernel we take the product of $q$ independent $\ell_2$-samples, yielding
$\tilde{k}(x,y)$. Notice that the expectation of this estimator is exactly equal
to the product of expectations, $\E[\tilde{k}(x,y)] = (x\trans y)^q$. The
variance of this estimator is equal to the product of variances, which is
$\var(\tilde{k}(x,y)) \leq (\|x\| \|y\|)^q \leq 1$. Of course, calculating the
inner product exactly takes $O(d \log q )$ time. We obtain:

\fi 

\begin{corollary}
For the polynomial degree-q kernel, \AlgKernel\  runs in time
\[\tilde{O}( \frac{q(n+d)}{\eps^2} + \min\{ \frac{ d \log q }{ \varepsilon^4} , \frac{q}{\varepsilon^6}\}).\]
\end{corollary}

\ifFOCS\else 

\paragraph{Gaussian kernels}
To estimate the Gaussian kernel function, we assume that $\norm{x}$ and
$\norm{y}$ are known and no more than $\ksig/2$; thus to estimate
\[
k(x,y)
    = \exp(\norm{x-y}^2)
    = \exp((\norm{x}^2 + \norm{y}^2)/2\ksig^2)\exp(x\trans y/\ksig^2),
\]
we need to estimate $\exp(x\trans y/\ksig^2)$. For
 $\exp(\gamma X)=\sum_{i\ge 0} \gamma^i X^i/i!$ with
random $X$ and parameter $\gamma>0$, we pick index $i$ with probability
$\exp(-\gamma)\gamma^i/i!$ (that is, $i$ has a Poisson distribution)
and return $\exp(\gamma)$ times the product of $i$
independent estimates of $X$.

In our case we take $X$ to be the average of $c$ $\ell_2$-samples of $x\trans
y$, and hence $\E[X] = x\trans y \ , \ \E[X^2 ] \leq \frac{1}{c} \E[ (x\trans
y)^2] \leq \frac{1}{c} $. The expectation of our kernel estimator is thus:
$$
\E[\tilde{k}(x,y)]
    = \E[ \sum_{i\ge 0} e^{-\gamma} \gamma^i i! \cdot e^\gamma \cdot X^i ]
    = \sum_{i\ge 0} \gamma^i i! \prod_{j=1}^i \E[ X]
    =  \exp(\gamma x\trans y).
$$
The second moment of this estimator is bounded by:
$$
\E[\tilde{k}(x,y)^2]
    = \E[ \sum_{i\ge 0} e^{-\gamma} \gamma^i i! \cdot e^{2 \gamma} \cdot (X^i)^2 ]
    = e^\gamma \sum_{i\ge 0} \gamma^i i! \prod_{j=1}^i \E[ X^2]
    \leq  \exp(\frac{2 \gamma}{c}  ).
$$

Hence, we take $\gamma = c = \frac{1}{\ksig^2}$. This gives a correct estimator
in terms of expectation and constant variance. The variance can be further made
smaller than one by taking the average of a constant estimators of the above
type.

As for evaluation time, the expected size of the index $i$ is $\gamma =
\frac{1}{\ksig^2}$. Thus, we require on the expectation $\gamma \times c =
\frac{1}{\ksig^4}$ of $\ell_2$-samples.

We obtain:

\fi 

\begin{corollary}
For the Gaussian kernel with parameter $\ksig$, \AlgKernel\  runs
in time
\[
\tilde{O}( \frac{(n+d)}{\ksig^4 \eps^2} + \min\{ \frac{ d }{ \varepsilon^4} , \frac{1}{\ksig^4 \varepsilon^6}\}).
\]
\end{corollary}

\ifFOCS
\else 

\subsection{Kernelizing the MEB and strictly convex problems}

Analogously to \AlgKernel, we can define the kernel version of
strongly convex problems, including MEB. The
kernelized version of MEB is particularly efficient, since as in \AlgMEB, the norm $\norm{y_t}$
is never required. This means that the procedure $\KernelEllTwo$ can be computed
in time $O(nL_k)$ per iteration, for a total running time of $O(L_k(\eps^{-2}n+\eps^{-1}d))$.

\fi 


\section{Lower bounds}\label{sec:LB}
All of our lower bounds are information-theoretic, meaning that any
successful algorithm must read at least some number of entries of the input
matrix $A$. Clearly this also lower bounds the time complexity of the algorithm
in the unit-cost RAM model.

Some of our arguments 
use the following meta-theorem. Consider a $p \times q$
matrix $A$, where $p$ is an even integer. Consider the following random process.
Let $W \geq q$. Let $a = 1-1/W$, and let $e_j$ denote the $j$-th standard
$q$-dimensional unit vector.
For each $i \in [p/2]$, choose a random $j\in [q]$ uniformly, and set
 $A_{i+p/2} \gets A_i \gets a e_j + b (\vecc{1}_q - e_j)$,
where $b$ 
is chosen so that
$\norm{A_i}_2 = 1$. We say that such an $A$ is a YES instance. With probability $1/2$,
transform $A$ into a NO instance as follows: choose a random $i^* \in [p/2]$ uniformly,
and if
 $A_{i^*} = a e_j + b (\vecc{1}_q - e_j)$ for a particular $j^*\in [q]$,
set
 $A_{i^*+p/2}\gets -a e_{j^*} + b (\vecc{1}_q - e_{j^*})$.

Suppose there is a randomized algorithm reading at most $s$ positions of $A$
which distinguishes YES and NO instances with probability $\geq 2/3$, where the
probability is over the algorithm's coin tosses and this distribution $\mu$ on
YES and NO instances. By averaging this implies a deterministic algorithm $Alg$
reading at most $s$ positions of $A$ and distinguishing YES and NO instances
with probability $\geq 2/3$, where the probability is taken only over $\mu$.
We show the following meta-theorem with a standard argument.
\begin{theorem}(Meta-theorem)\label{thm:meta}
For any such algorithm $Alg$, $s = \Omega(pq)$.
\end{theorem}
\ifFOCS\else
This Meta-Theorem follows from the following folklore fact:
\begin{fact}\label{fact:list}
Consider the following random process. Initialize a length-$r$ array $A$ to an
array of $r$ zeros. With probability $1/2$, choose a random position $i \in [r]$
and set $A[i] = 1$. With the remaining probability $1/2$, leave $A$ as the all
zero array. Then any algorithm which determines if $A$ is the all zero array
with probability $\geq 2/3$ must read $\Omega(r)$ entries of $A$.
\end{fact}
\fi
\ifFOCS\else 
Let us prove Theorem \ref{thm:meta} using this fact:
\begin{proof}
Consider the matrix $B \in \reals^{(p/2) \times q}$ which is defined by
subtracting the ``bottom" half of the matrix from the top half, that is,
$B_{i,j} = A_{i,j} - A_{i+p/2,j}$. Then $B$ is the all zeros matrix, except
that with probability 1/2, there is one
entry whose value is roughly two, and
whose location is random and distributed uniformly.
An algorithm distinguishing between YES and NO instances of A in
particular distinguishes between the two cases for $B$, which cannot be done
without reading a linear number of entries.
\end{proof}
In the proofs of Theorem \ref{thm:lbRelative}, Corollary
\ref{thm:lbAdditive}, and Theorem \ref{thm:lbMeb},
it will be more convenient to use $M$ as an upper bound
on the number of non-zero entries of $A$ rather than the exact
number of non-zero entries. However, it should be understood
that these theorems (and corollary) hold even when $M$ is exactly the
number of non-zero entries of $A$.

To see this, our random matrices $A$ constructed in the proofs
have at most $M$ non-zero entries. If this number $M'$ is strictly less than
$M$, we
arbitrarily replace $M-M'$ zero entries with the value
$(nd)^{-C}$ for a large enough constant $C > 0$.
Under our assumptions on the margin
or the minimum enclosing ball radius of the points,
the solution value changes by at most a factor of
$(1 \pm (nd)^{1-C})$, which does not affect the proofs.

\fi 

\subsection{Classification}\label{subsec:lbClass}
Recall that the margin $\sigma(A)$ of an $n \times d$ matrix $A$ is given by
$\max_{x\in\ball} \min_i A_ix$. Since we assume that $\|A_i\|_2 \leq 1$
for all $i$, we have that $\sigma(A) \leq 1$.

\subsubsection{Relative Error}
We start with a theorem for relative error
algorithms.

\begin{theorem}\label{thm:lbRelative}
Let $\kappa > 0$ be a sufficiently small constant. Let $\eps$ and $\sigma(A)$
have
 $\sigma(A)^{-2}\eps^{-1} \leq \kappa \min(n,d)$, $\sigma(A) \leq 1-\eps$,
with $\eps$ also bounded above by a sufficiently small constant. Also assume that
$M \geq 2(n+d)$, that $n \geq 2$, and that $d \geq 3$. Then
any randomized algorithm which, with probability at least $2/3$, outputs a
number in the interval $\left [\sigma(A) - \eps \sigma(A), \sigma(A) \right ]$
must read
$$\Omega(\min(M, \sigma(A)^{-2}\eps^{-1}(n+d)))$$ entries of $A$. This
holds even if $\norm{A_i}_2 = 1$ for all rows $A_i$.
\end{theorem}
Notice that this yields a stronger theorem than assuming that both $n$ and
$d$ are sufficiently large, since 
one of these values may
be constant. 
\ifFOCS\else  
\begin{proof}

We divide the analysis into cases: the case in which $d$ or $n$ is constant,
and the case in which each is sufficiently large.
Let $\tau \in [0,1-\eps]$ be a real number to be determined.

\paragraph{Case: $d$ or $n$ is a constant}
By our assumption that
 $\sigma(A)^{-2}\eps^{-1} \leq \kappa \min(n,d)$,
the values $\sigma(A)$
and $\eps$ are constant, and sufficiently large.
Therefore we just need to show
an $\Omega(\min(M, n + d))$ bound on the number of entries read.
By the premise of the
theorem, $M =\Omega(n+d)$, so we can just show an $\Omega(n+d)$ bound.
\\\\
{\it An $\Omega(d)$ bound.}
We give a randomized construction
of an $n \times d$ matrix $A$.

The first row of $A$ is built as follows. Let
$A_{1,1}\gets\tau$ and $A_{1,2}\gets 0$.
Pick $j^*\in \{3, 4, \ldots, d\}$ uniformly
at random, and let $A_{1,j^*}\gets \eps^{1/2} \tau$.
For all remaining $j\in \{3, 4, \ldots, d\}$, assign $A_{1,j}\gets\zeta$,
where $\zeta\gets 1/d^3$.
(The role of $\zeta$ is to make an entry
slightly non-zero to prevent an algorithm which has access to exactly
the non-zero entries from skipping over it.)
Now using the conditions on $\tau$, we have
\[
X \gets \norm{A_1}^2
    = \tau^2 + (d-3)\zeta^2 + \eps \tau^2
    \le (1-\eps)^2 + d^{-2} + \eps
    \le 1 - \eps + \eps^2 + \kappa^2 \eps^2
    \le 1,
\]
and so by letting $A_{1,2}\gets \sqrt{1-X}$, we have $\norm{A_1}=1$.

Now we let $A_2\gets -A_1$, with two exceptions: we let $A_{2,1}\gets A_{1,1}=\tau$,
and with probability $1/2$, we negate $A_{2,j^*}$. Thus $\norm{A_2} = 1$ also.

For row $i$ with $i>2$, put $A_{i,1}\gets (1+\eps)\tau$,
$A_{i,2} \gets \sqrt{1-A_{i,1}^2}$, and all remaining entries zero.

We have the following picture.

\[ \left( \begin{array}{ccccccccc}
\tau & (1-\tau^2 - (d-3)\zeta^2 - \eps \tau^2)^{1/2} & \zeta & \cdots & \zeta & \eps^{1/2}\tau & \zeta & \cdots & \zeta \\
\tau & -(1-\tau^2 - (d-3)\zeta^2 - \eps \tau^2)^{1/2} & -\zeta & \cdots & -\zeta & \pm \eps^{1/2}\tau & -\zeta & \cdots & -\zeta\\
(1+\eps)\tau &  (1-(1+\eps)^2\tau^2)^{1/2}& 0 & \cdots &&&&& 0 \\
(1+\eps)\tau &  (1-(1+\eps)^2\tau^2)^{1/2}& 0 & \cdots &&&&& 0 \\
\cdots & \cdots & \cdots & \cdots &&&&& \cdots\\
(1+\eps)\tau &  (1-(1+\eps)^2\tau^2)^{1/2}& 0 & \cdots &&&&& 0
\end{array} \right)\]

Observe that the the number of non-zero entries of the resulting matrix is
$2n+2d-4$, which satisfies the premise of the theorem. Moreover, all rows $A_i$
satisfy $\norm{A_i} = 1$.

Notice that if $A_{1, j^*} = - A_{2,j^*}$, then the margin of $A$ is at most
$\tau$, which follows by observing that all but the first coordinate
of $A_1$ and $A_2$ have opposite signs.

On the other hand,
if $A_{1, j^*} = A_{2, j^*}$, consider the vector $y$
with $y_1 \gets 1$, $y_{j^*}\gets \sqrt{\eps}$, and all other
entries zero. Then for all $i$, $A_iy = \tau(1+\eps)$,
and so the unit vector $x\gets y/\norm{y}$ has
\[
A_ix
    = \frac{\tau(1+\eps)}{\sqrt{1+\eps}}
    = \tau (1+\eps)^{1/2} = \tau(1+\Omega(\eps)).
\]
It follows that in this case the margin of $A$ is at least
$\tau(1+\Omega(\eps))$.
Setting $\tau = \Theta(\sigma)$ and
rescaling $\eps$ by a constant factor, it follows that these two cases can be
distinguished by an algorithm satisfying the premise of the theorem.
By Fact \ref{fact:list}, any algorithm distinguishing these two cases
with probability $\geq 2/3$ must read $\Omega(d)$ entries of $A$.
\\\\
{\it An $\Omega(n)$ bound.} We construct the $n \times d$ matrix $A$ as
follows. All but the first two columns are $0$.
We set $A_{i,1} \gets \tau$ and $A_{i,2} \gets \sqrt{1-\tau^2}$ for all $i\in [n]$.
Next, with probability $1/2$, we pick a random row $i^*$, and negate
$A_{i^*,2}$.
We have the following picture.
\[ \left( \begin{array}{ccccc}
\tau & \sqrt{1-\tau^2} & 0 & \cdots & 0\\
\cdots & \cdots & 0 & \cdots & 0 \\
\tau & \sqrt{1-\tau^2} & 0 & \cdots & 0\\
\tau & \pm \sqrt{1-\tau^2} & 0 & \cdots & 0\\
\tau & \sqrt{1-\tau^2} & 0 & \cdots & 0\\
\cdots & \cdots & 0 & \cdots & 0\\
\tau & \sqrt{1-\tau^2} & 0 & \cdots & 0
\end{array} \right)\]
The number of non-zeros of the resulting matrix is $2n < M$.
Depending on the sign of $A_{i^*,2}$, the margin of $A$ is either
$1$ or $\tau$. Setting $\tau = \Theta(\sigma)$,
an algorithm satisfying the premise of the
theorem can distinguish the two cases. By Fact \ref{fact:list}, any algorithm
distinguishing these two cases with probability $\geq 2/3$ must read
$\Omega(n)$ entries of $A$.

\paragraph{Case: $d$ and $n$ are both sufficiently large}
Suppose first that $M = \Omega(\sigma(A)^{-2}\eps^{-1}(n+d))$ for a
sufficiently large constant in the $\Omega()$.
Let $s$ be an even integer in $\Theta(\tau^{-2} \eps^{-1})$
and with $s <\min(n,d)-1$. We will also choose a value
$\tau$ in $\Theta(\sigma(A))$.
We can assume without loss of generality
that $n$ and $d$ are sufficiently large, and even.

{\it An $\Omega(ns)$ bound.} We set the $d$-th entry of each row of $A$
to the value $\tau$. We set all entries in columns $s+1$ through $d-1$ to $0$.
We then choose the remaining entries of $A$ as follows. We apply Theorem
\ref{thm:meta} with parameters $p = n, q = s$, and $W = d^2$, obtaining an $n
\times s$ matrix $B$, where $\norm{B_i} = 1$ for all rows $B_i$. Put
$B' \gets B\sqrt{1-\tau^2}$. We then set $A_{i,j} \gets B'_{i,j}$ for all $i \in [n]$ and $j
\in [s]$. We have the following block structure for $A$.
$$
\bigg[\begin{array}{ccc}
B\sqrt{1-\tau^2} & \bzero_{n\times (d-s-1)} & \bone_n \tau \\
\end{array}\bigg ]
$$
Here $\bzero_{n\times (d-s-1)}$ is a matrix of all $0$'s, of the given
dimensions. Notice that $\norm{A_i} = 1$ for all
rows $A_i$, and the number of non-zero entries is at most $n(s+1)$, which is
less than the value $M$.

We claim that if $B$ is a YES instance, then the margin of $A$ is
$\tau(1+\Omega(\eps))$. Indeed, consider the unit vector $x$ for which
\begin{equation}\label{eqn:x}
x_j \gets
    \begin{cases}
        \left (\frac{\eps}{s} - \frac{\eps^2}{4s} \right )^{1/2}    & j\in [s] \\
        0                                                           & j\in [s+1, d-1] \\
        1-\eps/2                                                    & j=d
    \end{cases}
\end{equation}
For any row $A_i$,
\begin{align*}
A_ix
       & \ge \left (\frac{\eps}{s} - \frac{\eps^2}{4s} \right )^{1/2} \left (\sqrt{1-\tau^2}-O \left (\frac{\sqrt{1-\tau^2}}{d^2} \right ) \right ) + \left (1- \frac{\eps}{2} \right ) \tau
    \\ & \ge \left (\frac{\eps}{s} - \frac{\eps^2}{4s} \right )^{1/2} \left (1-\tau -O \left (\frac{\sqrt{1-\tau^2}}{d^2} \right ) \right ) + \tau-\frac{\eps \tau}{2}
                & \textrm{since } \sqrt{1-\tau^2} \ge 1-\tau
    \\ & \ge \left (\frac{\eps}{s} \right )^{1/2} (1-\tau) + \tau-\frac{\eps \tau}{2} - O(\eps^2 \tau^2)
                & \textrm{since } \sqrt{\frac{\eps}{s}} \cdot \frac{1}{d^2} = O(\eps^2 \tau^2)
\end{align*}
If we set $s = c \tau^{-2} \eps^{-1}$ for $c\in (0, 4)$, then
\begin{align}\label{eqn:xBound}
A_ix
    \ge \tau + \frac{\tau \eps}{c^{1/2}} - \tau \left (\frac{\eps}{2} + \frac{\tau \eps}{c^{1/2}} \right ) - O(\eps^2 \tau^2)
    = \tau(1+\Omega(\eps)).
\end{align}
On the other hand, if $B$ is a NO instance, we claim that the margin of $A$ is
at most $\tau(1+O(\eps^2))$. By definition of a NO instance, there are rows
$A_i$ and $A_j$ of $A$ which agree except on a single column $k$, for which
$A_{i,k} = \sqrt{1-\tau^2} - O \left (\frac{1-\tau^2}{d^2} \right )$ while
$A_{j,k} = -A_{i,k}$. It follows that the $x$ which maximizes
 $\min \{A_i x, A_jx\}$ has $x_k = 0$.
But
 $\sum_{k' \neq k} A_{i, k'}^2
  = 1 - (1-\tau^2) + O \left (\frac{1}{d^2} \right )
  = \tau^2 + O \left (\frac{1}{d^2} \right )$.
Since $\norm{x} \leq 1$, by the Cauchy-Schwarz inequality
\begin{align}\label{eqn:xBound2}
 A_i x = A_j x \leq
\left (\tau^2 + O \left (\frac{1}{d^2} \right ) \right )^{1/2} \leq
\tau + O \left (\eps^2 \right ) = \tau(1+O(\eps^2)),
\end{align}
where the
first inequality follows from our bound $\tau^{-2} \eps^{-1} = O(d)$.

Setting $\tau = \Theta(\sigma(A))$ and rescaling
$\eps$ by a constant factor, an algorithm satisfying the premise of the
theorem can distinguish the two cases, and so by Theorem \ref{thm:meta},
it must read $\Omega(ns) = \Omega(\sigma(A)^{-2}\eps^{-1}n)$ entries of $A$.
\\\\
{\it An $\Omega(ds)$ bound.} We first define rows $s+1$ through $n$
of our $n \times d$ input matrix $A$. For $i>s$,
put $A_{i,d} \gets \tau(1+\eps)$,
$A_{i, d-1} \gets (1-\tau^2(1+\eps)^2)^{1/2}$,
and all remaining entries zero.

We now define rows $1$ through $s$. Put $A_{i,d} \gets \tau$ for all $i \in [s]$.
Now we apply Theorem \ref{thm:meta} with
$p = s$, $q = d-2$, and $W = d^2$, obtaining an $s \times (d-2)$ matrix $B$, where
$\norm{B_i} = 1$ for all rows $B_i$. Put $B' \gets B\sqrt{1-\tau^2}$, and set
$A_{i,j} \gets B'_{i,j}$ for all $i \in [s]$ and $j \in [d-2]$.
We have the following block structure for $A$.
\[
\left [\begin{array}{ccc}
       B\sqrt{1-\tau^2}          & \bzero_s                               & \bone_s \tau
    \\ \bzero_{(n-s)\times(d-2)} & \bone_{n-s} (1-\tau^2(1+\eps)^2)^{1/2} & \bone_{n-s}\tau(1+\eps)
\end{array}\right ]
\]

Notice that $\norm{A_i} = 1$ for all rows $A_i$, and the number of non-zero
entries is at most $2n+sd < M$.

If $B$ is a YES instance, let $x$ be as in Equation (\ref{eqn:x}).
Since the first $s$ rows of $A$ agree with those in our proof of the
$\Omega(ns)$ bound, then as shown in Equation (\ref{eqn:xBound}),
$A_ix = \tau(1+\Omega(\eps))$ for $i \in [s]$. Moreover, for $i > s$,
since YES instances $B$ are entry-wise positive, we have
$$A_ix > \left (1-\frac{\eps}{2} \right ) \cdot \tau(1+\eps) = \tau(1+\Omega(\eps)).$$
Hence, if $B$ is a YES instance the margin is $\tau(1+\Omega(\eps))$.

Now suppose $B$ is a NO instance. Then, as shown in Equation
(\ref{eqn:xBound2}), for any $x$ for which $\norm{x} \leq 1$,
we have $A_ix  \le \tau(1+O(\eps^2))$ for $i \in [s]$. Hence,
if $B$ is a NO instance, the margin is at most $\tau(1+O(\eps^2))$.

Setting $\tau = \Theta(\sigma(A))$ and rescaling
$\eps$ by a constant factor, an algorithm satisfying the premise of the
theorem can distinguish the two cases, and so by Theorem \ref{thm:meta},
it must read $\Omega(ds) = \Omega(\sigma(A)^{-2}\eps^{-1}d)$ entries of $A$.
%
\\\\
Finally, if $M = O((n+d)\sigma(A)^{-2}\eps^{-1})$, then we must show an $\Omega(M)$
bound. We will use our previous construction for showing an $\Omega(ns)$ bound,
but replace the value of $n$ there with $n'$,
where $n'$ is the largest
integer for which $n's \leq M/2$.
We claim that $n' \geq 1$. To see this,
by the premise of the theorem $M \geq 2(n+d)$.
Moreover, $s = \Theta(\eps^{-1})$ and $\eps^{-1} \leq \kappa (n+d)$. For a small
enough constant $\kappa > 0$, $s \leq (n+d) \leq M/2$, as needed.

As the theorem statement concerns matrices with $n$ rows, each
of unit norm, we must have
an input $A$ with $n$ rows.
To achieve this, we put $A_{i,d} = \tau(1+\eps)$
and $A_{i, d-1} = (1-\tau^2(1+\eps)^2)^{1/2}$
for all $i > n'$.
In all remaining entries
in rows $A_i$ with $i > n'$, we put the value $0$. This ensures that
 $\norm{A_i} = 1$ for all $i > n'$, and it is easy to verify that
this does not change the margin of $A$. Hence, the lower bound
is $\Omega(n's) = \Omega(M)$. Notice that the number of non-zero entries is at most
$2n + n's \leq 2M/3 + M/3 = M$, as needed.

This completes the proof.
\end{proof}
\fi 

\subsubsection{Additive Error}
Here we give a lower bound for the additive error case. We give two different
bounds, one when $\eps < \sigma$, and one when $\eps \geq \sigma$. Notice
that $\sigma \geq 0$ since we may take the solution
$x = \bzero_d$.
The following is a corollary of Theorem \ref{thm:lbRelative}.
\begin{corollary}\label{thm:lbAdditive}
Let $\kappa > 0$ be a sufficiently small constant. Let $\eps, \sigma(A)$ be
such that $\sigma(A)^{-1}\eps^{-1} \leq \kappa \min(n,d)$ and
$\sigma(A) \leq 1-\eps/\sigma(A)$, where $0 < \eps \leq \kappa' \sigma$ for a sufficiently
small constant $\kappa' > 0$.
Also assume that $M \geq 2(n+d)$, $n \geq 2$, and $d \geq 3$.
Then any randomized algorithm which, with
probability at least $2/3$, outputs a number in the interval $[\sigma - \eps,
\sigma]$ must read
$$\Omega(\min(M, \sigma^{-1} \eps^{-1}(n+d)))$$
entries of $A$. This holds even if $\norm{A_i} = 1$ for all rows $A_i$.
\end{corollary}
\ifFOCS\else  
\begin{proof}
We simply set the value of $\eps$ in Theorem \ref{thm:lbRelative} to
$\eps/\sigma$. Notice that $\eps$ is at most a sufficiently small
constant and the value
$\sigma^{-2} \eps^{-1}$ in Theorem \ref{thm:lbRelative} equals $\sigma^{-1}
\eps^{-1}$, which is at most $\kappa \min(n,d)$ by the premise of the corollary,
as needed to apply Theorem \ref{thm:lbRelative}.
\end{proof}
\fi 
The following handles the case when $\eps = \Omega(\sigma)$.
\begin{corollary}\label{thm:lbAdditive2}
Let $\kappa > 0$ be a sufficiently small constant. Let $\eps, \sigma(A)$ be
such that $\eps^{-2} \leq \kappa \min(n,d)$,
$\sigma(A) + \eps < \frac{1}{\sqrt{2}}$,
and $\eps = \Omega(\sigma)$.
Also assume that $M \geq 2(n+d)$, $n \geq 2$, and $d \geq 3$.
Then any randomized algorithm which, with
probability at least $2/3$, outputs a number in the interval $[\sigma - \eps,
\sigma]$ must read
$$\Omega(\min(M, \ \eps^{-2}(n+d)))$$
entries of $A$. This holds even if $\norm{A_i} = 1$ for all rows $A_i$.
\end{corollary}
\ifFOCS\else  
\begin{proof}

The proof is very similar to that of Theorem \ref{thm:lbRelative}, so we just
outline the differences. In the case that
$d$ or $n$ is constant, we have the following families of hard instances:
\\\\
{\it An $\Omega(n)$ bound for constant $d$:}
\[ \left( \begin{array}{ccccccccc}
\tau & \left (1-\tau^2 - (d-3)\zeta^2 - 2(\eps+\tau)^2 \right )^{1/2} & \zeta & \cdots & \zeta & \sqrt{2}(\eps +\tau) & \zeta & \cdots & \zeta \\
\tau & -\left (1-\tau^2 - (d-3)\zeta^2 - 2(\eps+\tau)^2 \right )^{1/2} & -\zeta & \cdots & -\zeta & \pm \sqrt{2}(\eps +\tau)& -\zeta & \cdots & -\zeta\\
\sqrt{2}(\eps+\tau) &  \left (1-2(\eps+\tau)^2 \right )^{1/2}& 0 & \cdots &&&&& 0 \\
\sqrt{2}(\eps+\tau) &  \left (1-2(\eps+\tau)^2 \right )^{1/2}& 0 & \cdots &&&&& 0 \\
\cdots & \cdots & \cdots & \cdots &&&&& \cdots\\
\sqrt{2}(\eps+\tau) &  \left (1-2(\eps+\tau)^2 \right )^{1/2}& 0 & \cdots &&&&& 0
\end{array} \right)\]
\\\\
{\it An $\Omega(d)$ bound for constant $n$:}
\[ \left( \begin{array}{ccccc}
\tau & \sqrt{1-\tau^2} & 0 & \cdots & 0\\
\cdots & \cdots & 0 & \cdots & 0 \\
\tau & \sqrt{1-\tau^2} & 0 & \cdots & 0\\
\tau & \pm \sqrt{1-\tau^2} & 0 & \cdots & 0\\
\tau & \sqrt{1-\tau^2} & 0 & \cdots & 0\\
\cdots & \cdots & 0 & \cdots & 0\\
\tau & \sqrt{1-\tau^2} & 0 & \cdots & 0
\end{array} \right)\]
\\\\
In these two cases, depending on the
sign of the undetermined entry the margin
is either $\tau$ or at least
$\tau + \eps$ (in the $\Omega(d)$ bound, it is $\tau$ or $1$, but we assume
$\tau + \eps < \frac{1}{\sqrt{2}}$). It follows for $\tau = \sigma(A)$, the algorithm of the
corollary can distinguish these two cases, for which the lower bounds follow
from the proof of Theorem \ref{thm:lbRelative}.

For the case of $n$ and $d$ sufficiently large, we have the following families
of hard instances. In each case, the matrix $B$ is obtained by invoking Theorem
\ref{thm:meta} with the value of $s = \Theta(\eps^{-2})$.
\\\\
{\it An $\Omega(n\eps^{-2})$ bound for $n,d$ sufficiently large:}
$$
\bigg[\begin{array}{ccc}
B\sqrt{1-\tau^2} & \bzero_{n\times (d-s-1)} & \bone_n \tau \\
\end{array}\bigg ]
$$
{\it An $\Omega(d\eps^{-2})$ bound for $n,d$ sufficiently large:}
\[
\left [\begin{array}{ccc}
       B\sqrt{1-\tau^2}          & \bzero_s                               & \bone_s \tau
    \\ \bzero_{(n-s)\times(d-2)} & \bone_{n-s} (1-(\tau + \eps)^2)^{1/2} & \bone_{n-s}(\tau + \eps)
\end{array}\right ]
\]
In these two cases,  by setting $W = \textrm{poly}(nd)$ to be sufficiently large in
Theorem \ref{thm:meta}, depending on whether $B$ is YES or a NO instance
the margin is either at most $\tau + \frac{1}{\textrm{poly}(nd)}$ or at
least $\tau + \sqrt{1-\tau^2} \cdot 2\eps$ (for an appropriate choice of $s$).
For $\tau < 1/\sqrt{2}$, the algorithm of the corollary
can distinguish these two cases, and therefore needs $\Omega(ns)$ time in the first
case, and $\Omega(ds)$ time in the second.

The extension of the proofs to handle the case $M = o((n+d)\eps^{-2})$ is identical
to that given in the proof of Theorem \ref{thm:lbRelative}.
\end{proof}
\fi 

\subsection{Minimum Enclosing Ball}\label{subsec:lbMeb}

We start by proving the following lower bound for estimating the squared MEB radius
to within an additive $\eps$. In the next subsection we improve the $\Omega(\eps^{-1}n)$
term in the lower bound to $\tilde{\Omega}(\eps^{-2}n)$ for algorithms that either additionally
output a coreset, or output a MEB center that is a convex combination of the input points.
As our
primal-dual algorithm actually outputs a coreset, as well as a MEB center that is a convex
combination of the input points, those bounds apply to it. Our algorithm has
both of these properties though satisfying one or the other would be enough to apply the
lower bound. Together with the $\eps^{-1}d$
bound given by the next theorem, these bounds establish its optimality.
\begin{theorem}\label{thm:lbMeb}
Let $\kappa > 0$ be a sufficiently small constant.
Assume $\eps^{-1} \leq \kappa \min(n,d)$ and $\eps$ is less than a
sufficiently small constant.
Also assume that
$M \geq 2(n+d)$ and that $n \geq 2$.
Then any randomized algorithm which, with probability at least
$2/3$, outputs a number in the interval
 $$\left [\min_x \max_i \norm{x-A_i}^2 - \eps, \min_x \max_i \norm{x-A_i}^2\right
 ]$$
must read
$$\Omega(\min(M, \eps^{-1}(n+d)))$$
entries of $A$. This holds even if $\norm{A_i} = 1$ for all rows $A_i$.
\end{theorem}
\ifFOCS\else  
\begin{proof}

As with classification, we divide the analysis into cases: the case in
which $d$ or $n$ is constant, and the case in which each is
sufficiently large.

\paragraph{Case $d$ or $n$ is a constant}
By our assumption that
$\eps^{-1} \leq \kappa \min(n,d)$, $\eps$ is a constant, and
sufficiently large. So we just need to show an
$\Omega(\min(M, n+d))$ bound. By the premise of the theorem,
$M \geq 2(n+d)$, so we need only show an $\Omega(n+d)$ bound.

{\it An $\Omega(d)$ bound.}
We construct an $n \times d$ matrix
$A$ as follows. For $i > 2$, each row
$A_i$ is just the vector $e_1 = (1, 0, 0, \ldots, 0)$.

Let $A_{1,1}\gets 0$, and initially assign $\zeta \gets 1/d$
to all remaining entries of $A_1$. Choose a random
integer $j^*\in [2,d]$, and assign $A_{1,j^*}\gets \sqrt{1-(d-2)\zeta^2}$.
Note that $\norm{A_1}=1$.

Let $A_2 \gets -A_1$, and then with probability $1/2$,
negate $A_{2,j^*}$.

Our matrix $A$ is as follows.

\[ \left( \begin{array}{cccccccc}
0 & \zeta & \cdots & \zeta & \sqrt{1-(d-2)\zeta^2} & \zeta & \cdots & \zeta \\
0 & -\zeta & \cdots & -\zeta & \pm \sqrt{1-(d-2)\zeta^2} & -\zeta & \cdots & -\zeta\\
1 & 0 & \cdots &&&&& 0 \\
1 & 0 & \cdots &&&&& 0 \\
1 & \cdots & \cdots &&&&& \cdots\\
1 & 0 & \cdots &&&&& 0
\end{array} \right)\]
Observe that $A$ has at most $2n+2d\le M$ non-zero entries, and
all rows satisfy $\norm{A_i} = 1$.

If $A_{1,j^*} = -A_{2, j^*}$, then $A_1$ and $A_2$ form a diametral pair,
and the MEB radius is $1$.

On the other hand,
if $A_{1, j^*} = A_{2,j^*}$, then consider the ball center $x$ with
$x_1 \gets x_{j^*} \gets 1/\sqrt{2}$, and all other entries zero.
Then for all $i > 2$, $\norm{x-A_i}^2 = \left (1-\frac{1}{\sqrt{2}} \right )^2$. On the
other hand, for $i \in \{1,2\}$, we have
$$\norm{x-A_i}^2 \leq \frac{1}{2} + (d-2) \zeta^2 + \left (1-\frac{1}{\sqrt{2}} \right )^2
\leq 2-\sqrt{2} + \frac{1}{d}.$$
It follows that for $\eps$ satisfying the premise of the theorem,
an algorithm satisfying the premise of the
theorem can distinguish the two cases. By Fact \ref{fact:list}, any algorithm
distinguishing these two cases with probability $\geq 2/3$ must read
$\Omega(d)$ entries of $A$.
\\\\
{\it An $\Omega(n)$ bound.} We construct the $n \times d$ matrix
$A$ as follows. Initially set all rows $A_i\gets e_1 = (1, 0, 0, \ldots, 0)$.
Then with probability $1/2$ choose a random $i^*\in [n]$,
and negate $A_{i^*,1}$.

We have the following picture.
\[ \left( \begin{array}{cccc}
1 &  0 & \cdots & 0\\
\cdots & \cdots & \cdots & \cdots \\
1 & 0 & \cdots & 0\\
\pm 1 & 0 & \cdots & 0\\
1 & 0 & \cdots & 0\\
\cdots & 0 & \cdots & 0\\
1 & 0 & \cdots & 0
\end{array} \right)\]
The number of non-zeros of the resulting matrix is $n < M$.
In the case where there is an entry of $-1$, the MEB radius of $A$
is $1$, but otherwise the MEB radius is $0$.
Hence, an algorithm satisfying the premise of the
theorem can distinguish the two cases. By Fact \ref{fact:list}, any algorithm
distinguishing these two cases with probability $\geq 2/3$ must read
$\Omega(n)$ entries of $A$.
\\\\

\paragraph{Case: $d$ and $n$ are sufficiently large}
Suppose first that $M = \Omega(\eps^{-1}(n+d))$ for a
sufficiently large constant in the $\Omega()$.
Put $s = \Theta(\eps^{-1})$.
We can assume without loss of generality that $n$, $d$, and $s$ are sufficiently
large integers.
%
We need the following simple claim.
\begin{claim}\label{claim:meb}
Given an instance of the minimum enclosing ball problem in $T > t$ dimensions on
a matrix with rows
 $\{\alpha e_i + \beta \sum_{j \in [t] \setminus \{i\}} e_j\}_{i=1}^t$
for distinct standard unit vectors $e_i$ and $\alpha \geq \beta \geq 0$, the
solution $x = \sum_{i=1}^t (\alpha + (t-1)\beta)e_i/t$ of cost
$(\alpha-\beta)^2(1-1/t)$ is optimal.
\end{claim}
\begin{proof}
We can subtract the point $\beta \bone_T$ from each of the points, and an optimal
solution $y$ for the translated problem yields an optimal solution
$y + \beta\bone_T$ for the original problem with the same cost. We can assume
without loss of generality that
$T = t$ and that $e_1, \ldots, e_t$ are the $t$ standard unit vectors in
$\reals^t$. Indeed, the value of each of the rows on each of the remaining
coordinates is $0$. The cost of the point $y_* = \sum_{i=1}^t (\alpha-\beta)e_i/t$
in the translated problem is
 $$ \left (\alpha - \beta \right )^2\left (1-\frac{1}{t} \right )^2 + (t-1)
 \left (\alpha - \beta \right )^2/t^2 = \left (\alpha-\beta \right )^2 \left
 (1-\frac{1}{t} \right ).$$
On the other hand, for any point $y$, the cost with respect to row $i$ is
$(\alpha-\beta - y_i)^2 + \sum_{j \neq i} (\beta - y_j)^2$. By averaging and
Cauchy-Schwarz, there is a row of cost at least
\begin{align*}
\frac{1}{t} \cdot \left [\sum_{i=1}^t (\alpha-\beta-y_i)^2 + (t-1)\sum_{i=1}^t y_i^2 \right ]
 & =  \norm{y}^2 + (\alpha-\beta)^2 - \frac{2(\alpha-\beta)\norm{y}_1}{t}\\
& \geq \norm{y}^2 + (\alpha-\beta)^2 - \frac{2(\alpha-\beta)\norm{y}}{\sqrt{t}}
\end{align*}
Taking the derivative w.r.t. to $\norm{y}$, this is minimized when
 $\norm{y} =\frac{\alpha-\beta}{\sqrt{t}}$,
for which the cost is at least $(\alpha-\beta)^2(1-1/t)$.
\end{proof}
\noindent {\it An $\Omega(ns)$ bound.} We set the first $s$
rows of $A$ to $e_1, \ldots, e_s$. We set all entries outside of the first $s$
columns of $A$ to $0$. We choose the remaining $n-s = \Omega(n)$ rows of $A$ by
applying Theorem \ref{thm:meta} with parameters $p = n-s, q = s$, and $W = 1/d$.
If $A$ is a YES instance, then by Claim \ref{claim:meb}, there is a solution
with cost $(a-b)^2(1-1/s) = 1-\Theta(1/s)$. On the other hand, if $A$ is a NO
instance, then for a given $x$, either $\norm{A_{j^*}-x}^2$ or
$\norm{A_{p/2 + j^*}-x}^2$ is at least $a^2 = 1 - O(1/d)$. By setting $s = \Theta(\eps^{-1})$
appropriately, these two cases differ by an additive $\eps$, as needed.
\\\\
{\it An $\Omega(ds)$ bound.} We choose $A$ by applying Theorem
\ref{thm:meta} with parameters $p = s, q = d$, and $W = 1/d$. If $A$ is a YES
instance, then by Claim \ref{claim:meb}, there is a solution of cost at most
$(a-b)^2(1-1/s) = 1-\Theta(1/s)$. On the other hand, if $A$ is a NO instance, then
for a given $x$, either $\norm{A_{j^*}-x}^2$ or $\norm{A_{p/2+j^*}-x}^2$ is at
least $a^2 = 1-O(1/d)$. As before, setting $s = \Theta(\eps^{-1})$ appropriately
causes these cases to differ by an additive $\eps$.
\\\\
Finally, it remains to show an $\Omega(M)$ bound in case $M = O(\eps^{-1}(n+d))$.
We will use our previous construction for showing an $\Omega(ns)$ bound,
but replace the value of $n$ there with $n'$,
where $n'$ is the largest
integer for which $n's \leq M/2$.
We claim that $n' \geq 1$. To see this,
by the premise of the theorem $M \geq 2(n+d)$.
Moreover, $s = \Theta(\eps^{-1})$ and $\eps^{-1} \leq \kappa (n+d)$. For a small
enough constant $\kappa > 0$, $s \leq (n+d) \leq M/2$, as needed.

As the theorem statement concerns matrices with $n$ rows, each
of unit norm, we must have
an input $A$ with $n$ rows.
In this case, since the first row of $A$ is $e_1$, which has sparsity $1$,
we can simply set all remaining rows to the value of $e_1$, without changing
the MEB solution.
Hence, the lower bound
is $\Omega(n's) = \Omega(M)$. Notice that the number of non-zero entries is at most
$n + n's \leq M/2 + M/2 = M$, as needed.

This completes the proof.
\end{proof}
\fi 

\subsection{An $\tilde{\Omega}(n \eps^{-2})$ Bound for Minimum Enclosing Ball}

\subsubsection{Intuition}

Before diving into the intricate lower bound of this section, we describe a simple construction which lies at its core. Consider two distributions over arrays of size $d$: the first distribution, $\mu$, is uniformly distributed over all strings with exactly $\frac{3d}{4}$ entries that are $1$, and $\frac{d}{4}$ entries that are $-1$. The second distribution $\sigma$, is uniformly distributed over all strings with exactly $\frac{3d}{4} - D $ entries that are $1$, and $\frac{d}{4} + D $ entries that are $-1$, for $D = \tilde{O}( {\sqrt{d}})$.

Let $x \sim \mu$ with probability $\frac{1}{2}$ and $x \sim \sigma$ with probability $\frac{1}{2}$. Consider the task of deciding from which distribution $x$ was sampled. In both cases, the distributions are over the sphere of radius $\sqrt{d}$, so the norm itself cannot be used to distinguish them.  At the heart of our construction lies the following fact:
\begin{fact}
Any algorithm that decides with probability $\geq \frac{3}{4}$ the distribution that $x$ was sampled from, must read at least $\tilde{\Theta}(d)$ entries from $x$.
\end{fact}
We prove a version of this fact in the next sections. But first, let us explain the use of this fact in the lower bound construction: We create an instance of MEB which contains either $n$ vectors similar to the first type, or alternatively $n-1$ vector of the first type and an extra vector of the second type (with a small bias). To distinguish between the two types of instances, an algorithm has no choice but to check all $n$ vectors, and for each invest $O(d)$ work as per the above fact.
In our parameter setting, we'll choose $d = \tilde{O}(\veps^{-2})$, attaining the lower bound of $\tilde{O}(nd) = \tilde{O}(n \veps^{-2})$ in terms of time complexity.

To compute the difference in MEB center as $n \mapsto \infty$, note that by symmetry in the first case the center will be of the form $(a,a,...,a)$, where the value $a \in \reals$ is chosen to minimize the maximal distance:
$$ \arg \min_a \{ \frac{3}{4} (1-a)^2 + \frac{1}{4} (-1-a)^2 \} = \arg \min_a \{ a^2 -a + 1  \} = \frac{1}{2} $$
The second MEB center will be
$$ \arg \min_a \{ (\frac{3}{4}-\frac{D}{d}) (1-a)^2 + (\frac{1}{4}+\frac{D}{d}) (-1-a)^2 \} = \arg \min_a \{ a^2 - (1 - \frac{4D}{d})a + 1  \} = \frac{1}{2} - \frac{2D}{d} $$
Hence, the difference in MEB centers is on the order of $ \sqrt{ d \times (\frac{D}{d})^2  } = O(D^2 / d) = O(1)$. However, the whole construction is scaled to fit in the unit ball, and hence the difference in MEB centers becomes $\frac{1}{\sqrt{d}} \sim \veps$. Hence for an $\veps$ approximation the algorithm must distinguish between the two distributions, which in turn requires $\Omega(\veps^{-2})$ work.

\subsubsection{Probabilistic Lemmas}

For a set $S$ of points in $\reals^d$, let $\MEB(S)$ denote the smallest ball that contains $S$.
Let $\MRadius(S)$ be the radius of $\MEB(S)$, and $\MCenter(S)$ the unique center of $\MEB(S)$.

For our next lower bound, our bad instance will come from points on the hypercube
$\mH_d = \{-\frac{1}{\sqrt{d}}, \frac{1}{\sqrt{d}}\}^d$.

Call a vertex of $\mH_d$ \emph{regular} if it has
$\frac{3d}{4}$ coordinates equal to $\frac{1}{\sqrt{d}}$ and $\frac{d}{4}$
coordinates equal to $-\frac{1}{\sqrt{d}}$.  Call a vertex \emph{special}
if it has $\frac{3d}{4} - 12d\DD$ coordinates equal
to $\frac{1}{\sqrt{d}}$ and $\frac{d}{4}+ 12 d\DD$ coordinates equal to $-\frac{1}{\sqrt{d}}$,
where $\DD\equiv \frac{\ln n}{\sqrt{d}}$.

We will consider instances where all but one of the input rows $A_i$ are
random regular points, and one row may or may not be a random special point.
We will need some lemmas about these points.

\begin{lemma}
Let $a$ denote
a random regular point, $b$ a special point, and $c$ denote
the point $\bone_d/2\sqrt{d} = (\frac{1}{2\sqrt{d}}, \frac{1}{2\sqrt{d}},\ldots, \frac{1}{2\sqrt{d}})$.
Then
\begin{align}
\norm{a}^2  & = \norm{b}^2 = 1 \label{eq:norm a b} \\
\norm{c}^2 & = a\trans c = \frac{1}{4} \label{eq:norm c, a dot c}\\
\norm{a - c}^2  & = \frac{3}{4} \label{eq:one}  \\
b\trans c & = \E[a\trans b] = \frac{1}{4}-12\DD \label{eq:exp}
\end{align}
\end{lemma}

\begin{proof}
The norm claims are entirely straightforward, and we have
\[
a\trans c = \frac{1}{2d} \cdot \frac{3d}{4} - \frac{1}{2d} \cdot \frac{d}{4} = \frac{1}{4}.
\]
Also \eqref{eq:one}
follows by
\[
\norm{a-c}^2 = \norm{a}^2 + \norm{c}^2 - 2 a\trans c = 1 + \frac{1}{4} - 2\frac{1}{4} = \frac{3}{4}.
\]
For \eqref{eq:exp}, we have
$$b\trans c = \frac{1}{2d} \left (\frac{3d}{4} - 12d\DD \right )
- \frac{1}{2d} \left (\frac{d}{4} + 12d\DD \right )
= \frac{3}{8} - 6\DD - \frac{1}{8} -6\DD
= \frac{1}{4} - 12\DD,
$$
and by linearity of expectation,
\begin{eqnarray*}
\E[a\trans b ] & = & d \cdot \frac{1}{d} \cdot
\left (\frac{3}{4} \cdot \left (\frac{3}{4} - 12\DD \right )
              + \frac{1}{4} \cdot \left (\frac{1}{4} + 12\DD \right )
       - \frac{3}{4} \cdot \left (\frac{1}{4} + 12\DD \right )
              -\frac{1}{4} \cdot \left (\frac{3}{4} - 12\DD \right )
      \right )\\
& = & \frac{1}{4} - 12\DD.
\end{eqnarray*}
\end{proof}

Next, we show that $a\trans b$ is concentrated around its expectation \eqref{eq:exp}.

\begin{lemma}\label{lem:tail}
Let $a$ be a random regular point, and $b$ a special point.
For $d \geq 8 \ln^2 n$, $\Pr[a\trans b > \frac{1}{4} - 6\DD] \leq \frac{1}{n^3}$,
and  $\Pr[a\trans b < \frac{1}{4} - 18\DD] \leq \frac{1}{n^3}$.
\end{lemma}
\begin{proof}
We will prove the first tail estimate, and then discuss the changes needed to prove the second estimate.

We apply the upper tail of the following enhanced form of Hoeffding's bound, which holds for random
variables with bounded correlation.
\begin{fact}(Theorem 3.4 of \cite{ps97} with their value of $\lambda$ equal to $1$)\label{fact:hoeffding}
Let $X_1, \ldots, X_d$ be given random variables with support $\{0,1\}$ and let $X = \sum_{j=1}^d X_j$. Let
$\gamma > 0$ be arbitrary.
If there exist independent random variables $\hat{X}_1, \ldots, \hat{X}_d$ with
$\hat{X} = \sum_{j=1}^d \hat{X}_j$ and $\E[X] \leq\E[\hat{X}]$ such that for all $J \subseteq [d]$,
$$\Pr \left [\wedge_{j \in J} X_j = 1 \right ] \leq \prod_{j \in J} \Pr \left [\hat{X}_j = 1 \right ],$$
then
$$\Pr[X > (1+\gamma)\E[\hat{X}]] \leq \left [\frac{e^{\gamma}}{(1+\gamma)^{1+\gamma}} \right ]^{\E[\hat{X}]}.$$
\end{fact}
Define $X_j = \frac{d}{2} \cdot \left (a_j b_j + \frac{1}{d} \right )$.
Since $a_j b_j \in \{-\frac{1}{\sqrt{d}}, \frac{1}{\sqrt{d}}\}$, the $X_j$ have support $\{0,1\}$. Let
$\hat{X}_1, \ldots, \hat{X}_d$ be i.i.d. variables with support $\{0,1\}$ with
$\E[\hat{X}_j] =\E[X_j]$ for all $j$.

We claim that for all $J \subseteq [d]$,
$\Pr \left [\wedge_{j \in J} X_j = 1 \right ] \leq \prod_{j \in J} \Pr \left [\hat{X}_j = 1 \right ].$
By symmetry, it suffices to prove it for $J \in \{[1], [2], \ldots, [d]\}$. We prove it by induction.
The base case $J = [1]$ follows since $\E[\hat{X}_j] =\E[X_j]$. To prove the inequality
for $J = [\ell]$, $\ell \geq 2$, assume the inequality holds for $[\ell-1]$. Then,
$$\Pr[\wedge_{j \in [\ell]} X_j = 1] = \Pr[\wedge_{j \in [\ell-1]} X_j = 1] \cdot \Pr[X_{\ell} = 1 \mid \wedge_{j \in [\ell-1]}X_j = 1],$$
and by the inductive hypothesis,
$$\Pr[\wedge_{j \in [\ell-1]} X_j = 1] \leq \prod_{j \in [\ell-1]} \Pr \left [\hat{X}_j = 1 \right ],$$
so to complete the induction it is enough to show
\begin{eqnarray}\label{eqn:negCor}
\Pr[X_{\ell} = 1 \mid \wedge_{j \in [\ell-1]}X_j = 1] \leq \Pr[X_{\ell} = 1].
\end{eqnarray}
Letting $\Delta(a, b)$ be the number of coordinates $j$ for which $a_j \neq b_j$, we have
$$\Pr[X_{\ell} = 1] = 1- \frac{\E[\Delta(a, b)]}{d}.$$
If $\wedge_{j \in [\ell-1]} X_j = 1$ occurs,
then the first $\ell-1$ coordinates of $a_j$ and $b_j$ have the same sign, and so
$$\Pr[X_{\ell} = 1 \mid \wedge_{j \in [\ell-1]}X_j = 1]
	= 1 - \frac{\E[\Delta(a,b)\mid \wedge_{j \in [\ell-1]} X_j=1]}{d-\ell+1}
	= 1 - \frac{\E[\Delta(a,b)]}{d-\ell+1},
$$
which proves (\ref{eqn:negCor}).

We will apply Fact \ref{fact:hoeffding} to bound $\Pr[a\trans b > r]$ for
$r =\frac{1}{4} -6\DD$.
Since $X = \frac{d}{2} a\trans b + \frac{d}{2} = \frac{d}{2}(1+a\trans b)$, we have
\[
\frac{X-\E[X]}{\E[X]}
	= \frac{a\trans b - \E[a\trans b]}{1 + \E[a\trans b]},
\]
where we have used that \eqref{eq:exp} implies $\E[X]$ is positive (for large enough $d$),
so we can perform the division. So
\[
\frac{X-\E[X]}{\E[X]} - \frac{r - \E[a\trans b]}{1 + \E[a\trans b]}
	= \frac{\ a\trans b - \E[a\trans b]}{1 + \E[a\trans b]}- \frac{r - \E[a\trans b]}{1 + \E[\ a\trans b]}
	= \frac{\ a\trans b - r}{1 + \E[a\trans b]},
\]
and so
\[
\Pr [a\trans b > r ]
	=  \Pr \left[\frac{X-\E[X]}{\E[X]}  > \frac{r-\E[a\trans b ]}{1 + \E[a\trans b ]} \right].
\]
By Fact \ref{fact:hoeffding}, for
$\gamma = \frac{r - \E[a\trans b ]}{1+ \E[a\trans b ]}$,
we have for $\gamma > 0$,
$$\Pr[a\trans b > r]
\leq \left [\frac{e^{\gamma}}{(1+\gamma)^{1+\gamma}} \right ]^{d(1+\E[a\trans b])/2}.$$
By  \eqref{eq:exp}, $r -\E[a\trans b] = 6\DD$,
and $1 \leq 1 +\E[a\trans b] \leq 2$,
so $\gamma \in \left [3\DD, 6\DD\right ]$.
It is well-known (see Theorem 4.3 of \cite{mr95}) that for $0 < \gamma < 2e-1$,
$e^\gamma \le (1+\gamma)^{1+\gamma}e^{-\gamma^2/4}$, and so
$$ \left [\frac{e^{\gamma}}{(1+\gamma)^{1+\gamma}} \right ]^{d(1+\E[a\trans b])/2}
\leq \exp\left (-\frac{\gamma^2}{4} (d(1+\E[a\trans b])/2 ) \right ) = \exp(-\gamma^2 d(1+\E[a\trans b])/8).$$
Since $\gamma \geq 3\DD$ and $\E[a\trans b] > 0$, this is at most
$ \exp(-\DD^2d) \le \exp(-(\ln n)^2) \le n^{-3}$, for large enough $n$, using
the definition of $\DD$.

For the second tail estimate, we can apply the same argument to $-a$ and $b$, proving that
$\Pr[-a\trans b > r] \le 1/n^3$, where $r \equiv -1/4 + 18\DD$.
We let $X_j$ be the $\{0,1\}$ variables $ \frac{d}{2}(-a_j b_j + \frac{1}{d})$,
with expected sum $\E[X] = 3d/8 + 6\DD$. As above,
$\Pr[-a\trans b > r] = \Pr[\frac{X-\E[X]}{\E[X]} > \gamma]$, where
$\gamma \equiv \frac{r - \E[-a\trans b]}{1+\E[-a\trans b]}$.
Now $\gamma\sqrt{d}$ is between $6\ln n$ and $8\ln n$,
so the same relations apply as above, and the second tail estimate follows.
\end{proof}

Note that since by \eqref{eq:one}  all regular points are distance
$\sqrt{3}/2$ from $c$, that distance is an upper bound for the
the MEB radius of a collection of regular points.

The next lemmas give more properties of MEBs involving regular and special points,
under the assumption that the above concentration bounds on $a\trans b$ hold for
a given special point $b$ and all $a$ in a collection of regular points.

That is, let $S$ be a collection of random regular points.
Let $\mathcal{E}$ be the event that for all $a\in S$,
$-18\DD \le a\trans b - \frac{1}{4} \le - 6\DD$.
By Lemma \ref{lem:tail} and a union bound,
$$\Pr[\mathcal{E}] \geq 1-\frac{2}{n^2},$$
when $S$ has at most $n$ points.

The condition of event $\mE$ applies not only
to every point in $S$, but to every point in
the convex hull $\conv S$.

\begin{lemma}\label{lem:convex}
For special point $b$ and collection $S$ of points $a$, if event $\mE$ holds,
then for every $a_S\in\conv S$,
$-18\DD \le a_S\trans b - \frac{1}{4} \le - 6\DD$.
\end{lemma}
\begin{proof}
Since $a_S\in\conv S$,
we have $a_S=\sum_{a\in S} p_a a$ for some values $p_a$ with
$\sum_{a \in S} p_a=1$ and $p_a\ge 0$ for all $a\in S$.
Therefore, assuming $\mE$ holds,
\[
a_S\trans b
	=\left[ \sum_{a \in S} p_a a\right]\trans b
	= \sum_{a \in S} p_a a\trans b
	\le \sum_{a \in S} p_a (1/4 - 6\DD)
	= 1/4 - 6\DD,
\]
and similarly $a_S\trans b \ge 1/4 - 18\DD$.
\end{proof}

\begin{lemma}\label{lem:two}
Suppose $b$ is a special point and $S$ is a collection of regular points
such that event $\mE$ holds.
Then for any $a_S\in\conv S$, $\norm{a_S-b} \ge  \frac{\sqrt{3}}{2} + 6D$.
Since $\MCenter(S)\in\conv S$, this bound applies to $\norm{\MCenter(S) - b}$ as well.
\end{lemma}

\begin{proof}
Let $H$ be the hyperplane normal to $c=\bone_d/2\sqrt{d}$ and containing $c$. Then
$S\subset H$, and so $\conv S\subset H$, and since the minimum norm point in $H$ is
$c$, all points $a_S\in\conv S$ have $\norm{a_S}^2 \ge \norm{c}^2 = 1/4$.
By the assumption that event $\mE$ holds, and the previous lemma, we have
$a_S\trans b  \le \frac{1}{4} - 6\DD$.
Using this fact, $\norm{b}=1$, and $\norm{a_S}^2\ge 1/4$, we have
\begin{align*}
\norm{a_S - b}^2
	& = \norm{a_S}^2 + \norm{b}^2  - 2 a_S\trans b \\
	& \geq \frac{1}{4} + 1 - 2 \left (\frac{1}{4} - 6\DD \right )\\
	& =  \frac{3}{4} +12\DD,
\end{align*}
and so
 $\norm{a_S-b} \geq \frac{\sqrt{3}}{2} +6\DD$
 provided
$\DD$ is smaller than a small constant.
\end{proof}

\begin{lemma}\label{lem:three}
Suppose $a$ is a regular point, $b$ is a special point,
and $a\trans b \ge \frac{1}{4} - 18\DD$.
Then there is a point $q \in \reals^d$ for which
$\norm{q-b}=  \frac{\sqrt{3}}{2}$ and
$\norm{q-a} \le  \frac{\sqrt{3}}{2} + \Theta(\DD^2)$, as $\DD\rightarrow 0$.
\end{lemma}
\begin{proof}
As usual let $c \equiv \bone_d/2\sqrt{d}$ and consider
the point $q$ at distance $\frac{\sqrt{3}}{2}$ from $b$ on the line segment $\overline{cb}$, so
$$q = c + \gamma \cdot \frac{b-c}{\norm{b-c}} = c + \gamma\alpha(b-c),$$
where $\alpha\equiv 1/\norm{b-c}$ and $\gamma$ is a value in $\Theta(\DD)$.
From the definition of $q$,
\begin{align*}
\norm{q-a}^2 & = \norm{q}^2 + \norm{a}^2 - 2a\trans q\\
& =   \norm{c}^2 + 2\gamma\alpha b\trans c - 2\gamma\alpha\norm{c}^2  + \gamma^2
	+\norm{a}^2 - 2a\trans c  - 2\gamma\alpha a\trans b + 2\gamma\alpha a\trans c.
\end{align*}
Recall from \eqref{eq:norm a b} that $\norm{a}=1$, from \eqref{eq:norm c, a dot c} that $a\trans c = \norm{c}^2 = \frac{1}{4}$, from \eqref{eq:exp} that $b\trans c = 1/4 - 12\DD$, and
the assumption $a\trans b \ge 1/4 - 18\DD$, we have
\begin{align*}
\norm{q - a}^2
	& =   1/4 + 2\gamma\alpha (1/4-12\DD) - 2\gamma\alpha(1/4)  + \gamma^2
	+1 - 2(1/4)  - 2\gamma\alpha (1/4 - 18\DD) + 2\gamma\alpha (1/4)
	\\ & = 3/4 + 12\gamma \alpha \DD + \gamma^2
	\\ & \le 3/4 + \Theta(\DD^2),
\end{align*}
where the last inequality uses
 $\gamma = \Theta(\DD)$ and $\alpha = \Theta(1)$.
 \end{proof}

\subsubsection{Main Theorem}

Given an $n \times d$ matrix $A$ together with the norms $\|A_i\|$ for all rows $A_i$, as well as the promise that
all $\|A_i\| = O(1)$, the $\eps$-{\bf MEB-Coreset} problem is to output a subset $S$
of $\tilde{O}(\eps^{-1})$ rows of $A$ for which $A_i \in (1+\eps) \cdot \MEB(S)$. Our main theorem in this section
is the following.
\begin{theorem}\label{thm:coreset}
If $n \eps^{-1} \geq d$ and $d = \tilde{\Omega}(\eps^{-2})$,
then any randomized algorithm which with probability $\geq 4/5$
solves $\eps$-{\bf MEB-Coreset} must read $\tilde{\Omega}(n \eps^{-2})$
entries of $A$ for some choice of its random coins.
\end{theorem}
We also define the following problem. Given an $n \times d$ matrix $A$
together with the norms $\|A_i\|$ for all rows
$A_i$, as well as the promise that all $\|A_i\| = O(1)$, the $\eps$-{\bf MEB-Center} problem is to output a
vector $x \in \reals^d$ for which $\|A_i - x\| \leq (1+\eps) \min_{y \in \reals^d} \max_{i \in [n]}\|y-A_i\|$.
We also show the following.
\begin{theorem}\label{thm:center}
If $n \eps^{-2} \geq d$ and $d = \tilde{\Omega}(\eps^{-2})$, then
any randomized algorithm which with probability $\geq 4/5$
solves $\eps$-{\bf MEB-Center} by outputting a convex combination of the rows $A_i$
must read $\tilde{\Omega}(n \eps^{-2})$ entries of $A$ for some choice of its random coins.
\end{theorem}
These theorems will follow from the same hardness construction, which we now describe.
Put $d = 8 \eps^{-2} \ln^2 n$, which we assume is a sufficiently large power of $2$.
We also assume $n$ is even. We construct two families
$\mathcal{F}$ and $\cG$ of $n \times d$ matrices $A$.

The family $\mathcal{F}$ consists of all $A$ for which each of the $n$ rows
in $A$ is a regular point.

The family $\cG$ consists of all $A$ for which exactly $n-1$ rows of $A$ are regular
points, and one row of $A$ is a special point.

(Recall that we say that a vertex of
on $\mH_d$ is \emph{regular} if it has exactly $\frac{3d}{4}$ coordinates equal to $\frac{1}{\sqrt{d}}$. We
say a point on $\mH_d$ is \emph{special} if it has exactly $d\left(\frac{3}{4} - 12 \DD\right)$
coordinates equal to $\frac{1}{\sqrt{d}}$, where $\DD$ is $\frac{\ln n}{\sqrt{d}}$.)

Let $\mu$ be the distribution on $n \times d$ matrices for which half of its mass is uniformly
distributed on matrices in $\mathcal{F}$, while the remaining half is uniformly distributed
on the matrices in $\cG$. Let $\bA \sim \mu$.
We show that any
randomized algorithm $\Alg$ which decides whether $\bA \in \mathcal{F}$ or $\bA \in \cG$
with probability at least $3/4$ must read $\tilde{\Omega}(nd)$ entries of $\bA$ for some choice
of its random coins. W.l.o.g., we
may assume that $\Alg$ is deterministic, since we may average out its random coins,
as we may fix its coin tosses that lead to the largest success probability (over the choice of $\bA$).
By symmetry and independence of the rows, we can assume that in each row, $\Alg$ queries entries in order, that is,
if $\Alg$ makes $s$ queries to a row $A_i$, we can assume it queries $A_{i,1}, A_{i,2}, \ldots, A_{i,s}$, and in that order.

Let $r = d/(C \ln^2 n)$ for a sufficiently large constant $C > 0$.
For a vector $u \in \reals^d$, let pref$(u)$ denote its first $r$ coordinates.
Let $\rho$ be the distribution of pref$(u)$ for a random regular
point $u$. Let $\rho'$ be the distribution of pref$(u)$ for a random special point $u$.

\begin{lemma}(Statistical Difference Lemma)\label{lem:mainDiff}
For $C > 0$ a sufficiently large constant,
$$\|\rho - \rho'\|_1 \leq \frac{1}{10}.$$
\end{lemma}
\begin{proof}
We will apply the following fact twice, once to $\rho$ and once to $\rho'$.
\begin{fact}(special case of Theorem 4 of \cite{df80})\label{fact:swap}
Suppose an urn $U$ contains $d$ balls, each marked by one of two colors.
Let $H_{U^r}$ be the distribution of $r$ draws made at
random without replacement from $U$, and $M_{U^r}$ be the distribution of $r$ draws made
at random with replacement. Then,
$$\|H_{U^k} - M_{U^k}\|_1 \leq \frac{4r}{d}.$$
\end{fact}
Let $\sigma$ be the distribution with support
$\{\frac{1}{\sqrt{d}}, -\frac{1}{\sqrt{d}}\}$ with
$\sigma(\frac{1}{\sqrt{d}}) = \frac{3}{4}$ and
$\sigma(-\frac{1}{\sqrt{d}}) = \frac{1}{4}$.
Let $\tau$ be the distribution with support
$\{\frac{1}{\sqrt{d}}, -\frac{1}{\sqrt{d}}\}$ with
$\tau(\frac{1}{\sqrt{d}}) = \frac{3}{4} - 12\DD$
and $\tau(-\frac{1}{\sqrt{d}}) = \frac{1}{4} + 12\DD$.

Let $\sigma^r$ be the joint distribution of $r$ independent samples from $\sigma$,
and similarly define $\tau^r$. Applying Fact \ref{fact:swap} with $r = 1/100\DD^2$,
$$\|\rho - \sigma^r\|_1 \leq \frac{1}{25 d\DD^2},$$
and
$$\|\rho'-\tau^r\|_1 \leq\frac{1}{25 d\DD^2}.$$
By the triangle inequality,
$$\norm{\rho-\rho}_1 \leq \norm{\rho-\sigma^r}_1 + \|\sigma^r - \tau^r\|_1 + \|\tau^r - \rho'\|_1 \leq
\|\sigma^r - \tau^r\|_1 + \frac{2}{25 d\DD^2},$$
and so it remains to bound $\|\sigma^r - \tau^r\|_1$. To do this, we use Stein's Lemma
(see, e.g., \ref{ct91}, Section 12.8), which shows that for two coins with bias in $[\Omega(1), 1-\Omega(1)]$,
one needs $\Theta(z^{-2})$ independent coins tosses to distinguish the distributions with constant probability,
where $z$ is the difference
in their expectations. Here, $z = 12\DD$, and so for constant $C > 0$ sufficiently
large, for $r = 1/C\DD^2$, it follows that $\|\sigma^r - \tau^r\|_1 \leq \frac{1}{20}$. We thus have
$$\|\rho-\rho'\|_1 \leq \frac{1}{20} + \frac{2}{25 d\DD^2} \leq \frac{1}{10},$$
where the last inequality uses $d\DD^2 = (\ln n)^2 \rightarrow\infty$.
\end{proof}

We use Lemma \ref{lem:mainDiff} to prove the following. We assume that $\Alg$ outputs $1$ if it decides that
$\bA \in \mathcal{F}$, otherwise it outputs $0$.

\begin{theorem}\label{thm:mainStat}
If $\Alg$ queries $o(nr)$ entries of $\bA$,
it cannot decide if $\bA \in \mathcal{F}$ with probability at least $3/4$.
\end{theorem}
\begin{proof}
We can think of $\bA$ as being generated according to the following random process.
\begin{enumerate}
\item Choose an index $i^* \in [n]$ uniformly at random.
\item Choose rows $\bA_j$ for $j \in [n]$ to be random independent regular points.
\item With probability $1/2$, do nothing. Otherwise,
with the remaining probability $1/2$, replace $\bA_{i^*}$ with a random special point.
\item Output $\bA$.
\end{enumerate}
Define the advantage $\adv(\Alg)$ to be:
$$\adv(\Alg) \equiv \left | \Pr_{A \in_R \cG}[\Alg(A) = 1] - \Pr_{A \in_R \mathcal{F}}[\Alg(A) = 1] \right |.$$
To prove the theorem, it suffices to show $\adv(\Alg) < 1/4$. Let $\bar{\bA}_{i^*}$ denote the
rows of $\bA$, excluding row $i^*$, generated in step 2.
By the description of the random process above, we have
$$\adv(\Alg) = {\bf E}_{i^*, \ \bar{\bA}_{i^*}}
\left | \Pr_{\textrm{special }\bA_{i^*}}[\Alg(A) = 1 \mid i^*, \ \bar{\bA}_{i^*}]
- \Pr_{\textrm{regular }\bA_{i^*}}[\Alg(A) = 1 \mid i^*, \ \bar{\bA}_{i^*}] \right |.$$
To analyze this quantity, we first condition on a certain event $\mathcal{E}(i, \bar{\bA}_{i^*})$ holding,
which will occur with probability $1-o(1)$, and allow us to discard the pairs $(i, \bar{\bA}_{i^*})$ that do not
satisfy the condition of the event. Intuitively, the event is just that for most regular $\bA_{i^*}$,
algorithm $\Alg$ does not read more than $r$
entries in $\bA_{i^*}$. This holds with probability $1-o(1)$, over the choice of $i^*$ and $\bar{\bA}_{i^*}$,
because all $n$ rows of $A$ are i.i.d., and so
on average $\Alg$ can only afford to read $o(r)$ entries in each row.

More formally, we say a pair $(i, \bar{A}_{i^*})$ is {\it good}  if
$$\Pr_{\textrm{regular }\bA_{i^*}}[\Alg \textrm{ queries at most }r \textrm{ queries of }\bA_{i^*} \mid
(i, \bar{A}_{i^*}) = (i^*, \bar\bA_{i^*})].$$
Let $\mathcal{E}(i^*, \bar{\bA}_{i^*})$ be the event that $(i, \bar{\bA}_{i^*})$ is good.
Then, $\Pr_{i^*, \bar{\bA}_{i^*}}[\mathcal{E}(i^*, \bar{\bA}_{i^*})] = 1-o(1)$,
and we can upper bound the advantage by
$$
{\bf E}_{i^*, \bar{\bA}_{i^*}}
\left | \Pr_{\textrm{special }\bA_{i^*}}[\Alg(A) = 1 \mid \mathcal{E}(i^*, \bar{\bA}_{i^*}), \ i^*, \ \bar{\bA}_{i^*}]
- \Pr_{\textrm{regular }\bA_{i^*}}[\Alg(A) = 1 \mid \mathcal{E}(i^*, \bar{\bA}_{i^*}), \ i^*, \ \bar{\bA}_{i^*}] \right |
+ o(1).$$
Consider the algorithm $\Alg'_{i^*}$, which on input $A$, makes the same sequence of queries to $A$ as $\Alg$
unless it must query more than $r$ positions of $A_{i^*}$. In this case, it outputs an arbitrary value
in $\{0,1\}$, otherwise it outputs $\Alg(A)$.
\begin{claim}\label{claim:central}
$$\left |\Pr_{\textrm{regular }\bA_{i^*}}[\Alg(A) = 1 \mid \mathcal{E}(i^*, \bar{\bA}_{i^*}), \ i^*, \ \bar{\bA}_{i^*}]
- \Pr_{\textrm{regular }\bA_{i^*}}[\Alg'_{i^*}(A) = 1 \mid i^*, \ \bar{\bA}_{i^*}] \right | = o(1),$$
\end{claim}
\begin{proof}
Since $\mathcal{E}(i^*, \bar{\bA}_{i^*})$ occurs,
$$\Pr_{\textrm{regular }\bA_{i^*}}[\Alg \textrm{ makes at most } r \textrm{ queries to } \bA_{i^*} \mid
i^*, \ \bar{\bA}_{i^*}]
= 1-o(1).$$
This implies that
$$\left |\Pr_{\textrm{regular }\bA_{i^*}}[\Alg(A) = 1 \mid \mathcal{E}(i^*, \bar{\bA}_{i^*}), \ i^*, \ \bar{\bA}_{i^*}]
- \Pr_{\textrm{regular }\bA_{i^*}}[\Alg'_{i^*}(A) = 1 \mid i^*, \ \bar{\bA}_{i^*}] \right | = o(1).$$
\end{proof}
By Lemma \ref{lem:mainDiff}, we have that
$$\left |\Pr_{\textrm{regular }\bA_{i^*}}[\Alg'_{i^*}(A) = 1 \mid i^*, \ \bar{\bA}_{i^*}]
- \Pr_{\textrm{special }\bA_{i^*}}[\Alg'_{i^*}(A) = 1 \mid i^*, \ \bar{\bA}_{i^*}] \right | \leq \frac{1}{10}.$$
Hence, by Claim \ref{claim:central} and the triangle inequality, we have that
$$\left |\Pr_{\textrm{regular }\bA_{i^*}}[\Alg(A) = 1 \mid \mathcal{E}(i^*, \bar{\bA}_{i^*}), \ i^*, \ \bar{\bA}_{i^*}]
- \Pr_{\textrm{special }\bA_{i^*}}[\Alg'_{i^*}(A) = 1 \mid i^*, \ \bar{\bA}_{i^*}] \right | \leq \frac{1}{10} + o(1).$$
To finish the proof, it suffices to show the following claim
\begin{claim}\label{claim:second}
$$\left | \Pr_{\textrm{special }\bA_{i^*}}[\Alg(A)
= 1 \mid \mathcal{E}(i^*, \bar{\bA}_{i^*}), \ i^*, \ \bar{\bA}_{i^*}]
- \Pr_{\textrm{special }\bA_{i^*}}[\Alg'_{i^*}(A) = 1 \mid i^*, \ \bar{\bA}_{i^*}] \right | \leq \frac{1}{10} +o(1).$$
\end{claim}
Indeed, if we show Claim \ref{claim:second}, then by the triangle inequality we will have that
$\adv(\Alg) \leq \frac{1}{5} + o(1) < \frac{1}{4}$.
\begin{proofof}{of Claim \ref{claim:second}}
Since $\mathcal{E}(i^*, \bar{\bA}_{i^*})$ occurs,
$$\Pr_{\textrm{regular }\bA_{i^*}}[\Alg \textrm{ makes at most } r \textrm{ queries to } \bA_{i^*} \mid
i^*, \ \bar{\bA}_{i^*}]
= 1-o(1).$$
Since $\rho$ is the distribution of prefixes of regular points, this condition can be rewritten as
$$\Pr_{u \sim \rho}[\Alg \textrm{ makes at most } r \textrm{ queries to the } i^*\textrm{-th row} \mid
i^*, \ \bar{\bA}_{i^*}, \ \textrm{pref}(A_{i^*}) = u] = 1-o(1).$$
By Lemma \ref{lem:mainDiff}, we thus have,
$$\Pr_{u \sim \rho'}[\Alg \textrm{ makes at most } r \textrm{ queries to the } i^*\textrm{-th row} \mid
i^*, \ \bar{\bA}_{i^*}, \ \textrm{pref}(A_{i^*}) = u] \geq \frac{9}{10} -o(1).$$
Since $\rho'$ is the distribution of prefixes of special points, this condition can be rewritten as
$$\Pr_{\textrm{special }\bA_{i^*}}[\Alg \textrm{ makes at most } r \textrm{ queries to } \bA_{i^*} \mid
i^*, \ \bar{\bA}_{i^*}] \geq \frac{9}{10} - o(1).$$
This implies that
$$\left | \Pr_{\textrm{special }\bA_{i^*}}[\Alg(A)
= 1 \mid \mathcal{E}(i^*, \bar{\bA}_{i^*}), \ i^*, \ \bar{\bA}_{i^*}]- \Pr_{\textrm{special }\bA_{i^*}}[\Alg'_{i^*}(A)
= 1 \mid i^*, \ \bar{\bA}_{i^*}] \right | \leq \frac{1}{10} + o(1).$$
\end{proofof}
This completes the proof of the theorem.
\end{proof}

\subsubsection{Proofs of Theorem \ref{thm:coreset} and \ref{thm:center}}

Next we show how Theorem \ref{thm:mainStat} implies Theorem \ref{thm:coreset} and Theorem \ref{thm:center},
using the results on MEBs of regular and special points.

\begin{proofof}{of Theorem \ref{thm:coreset}}
We set the dimension $d = 4 \cdot 36 \cdot \eps^{-2} \ln^2(n-1)$. Let $A'$ denote the set of
regular rows of $\bA$. We condition on event $\mE$, namely,
that every convex combination $p^TA$, where $p\in \Delta_{n-1}$,satisfies
$p^T A' b \leq \frac{1}{4} - 6\DD$.
This event occurs with probability at least $1-2n^{-2}$.
(We may neglect the difference between $n$ and $n-1$ in some expressions.)

It follows by Lemma \ref{lem:two} that if $\bA \in \cG$, then
for every $S \subseteq A'$,
$$\norm{\MCenter(S) - b} \geq \frac{\sqrt{3}}{2} + 2\eps.$$
By \eqref{eq:one}, $\MRadius(A') \leq \frac{\sqrt{3}}{2}$.
It follows that any algorithm that, with probability at least $4/5$,
outputs a subset $S$ of $\tilde{O}(\eps^{-1})$ rows of $\bA$
for which $\bA_i \in (1+\eps)\cdot \MEB(S)$ must include the point $b \in S$.

Given such an algorithm, by reading each of the $\tilde{O}(\eps^{-1})$ rows output, we can
determine if $\bA \in \mathcal{F}$ or $\bA \in \cG$ with an additional
$\tilde{O}(\eps^{-1}d)$ time. By Theorem \ref{thm:mainStat}, the total time must be
$\tilde{\Omega}(n\eps^{-2})$. By assumption, $n\eps^{-1} \geq d$, and so
any randomized algorithm that solves $\eps$-{\bf MEB-Coreset} with probability at least $4/5$,
can decide if $\bA \in \mathcal{F}$ with probability at least $4/5-2n^{-2} \geq 3/4$, and so
it must read $\tilde{\Omega}(n\eps^{-2})$ entries for some choice of its random coins.
\end{proofof}

\begin{proofof}{of Theorem \ref{thm:center}}
We again set the dimension $d = 4 \cdot 36 \cdot \eps^{-2} \ln^2(n-1)$. Let $A'$ denote the set of
regular rows of $\bA$. We again condition on the event $\mE$.

By Lemma \ref{lem:two}, if $\bA \in \cG$, then
for every convex combination $p^TA'$,
$$\norm{p^TA' - b} \geq \frac{\sqrt{3}}{2} + 2\eps,$$
and so the MEB radius returned by any algorithm that outputs a convex combination of
rows of $A'$ must be at least $\frac{\sqrt{3}}{2} + 2\eps$.

However, by \eqref{eq:one}, if $\bA \in \mathcal{F}$, then
$\MRadius(\bA) \leq \frac{\sqrt{3}}{2}$. On the other hand, by Lemma \ref{lem:three},
if $\bA \in \cG$, then MEB-radius$(\bA) \leq \frac{\sqrt{3}}{2} + \Theta(\eps^2)$.

It follows that
if $\bA \in \cG$, then the convex combination $p^T \bA$ output by the algorithm
must have a non-zero coefficient multiplying the special point $b$. This, in particular,
implies that $p^T \bA$ is not on the affine hyperplane $H$ with normal vector $\bone_d$ containing
the point $c=\bone_d/2\sqrt{d}$. However,
if $\bA \in \mathcal{F}$, then any convex combination of the points is on $H$. The
output $p^T \bA$ of the algorithm is on $H$ if and only if
$p^T \bA \bone_d = \frac{\sqrt{d}}{2}$, which can be tested in $O(d)$ time.

By Theorem \ref{thm:mainStat}, the total time must be
$\tilde{\Omega}(n\eps^{-2})$. By assumption, $n\eps^{-2} \geq d$, and so
any randomized algorithm that solves $\eps$-{\bf MEB-Center} with probability
$\geq 4/5$ by outputting a convex combination of rows can decide if $\bA \in \mathcal{F}$
with probability at least $4/5 -2n^{-2} \geq 3/4$, and so must read $\tilde{\Omega}(n\eps^{-2})$
entries for some choice of its random coins.
\end{proofof}

\subsection{Las Vegas Algorithms}\label{subsec:LV}
While our algorithms are Monte Carlo, meaning they err with small probability,
it may be desirable to obtain Las Vegas algorithms, i.e., randomized algorithms
that have low expected time but never err. We show this cannot be done in
sublinear time.
\begin{theorem}
For the classification and minimum enclosing ball problems, there is no Las
Vegas algorithm that reads an expected $o(M)$ entries of its input matrix and
solves the problem to within a one-sided additive error of at most $1/2$.
This holds even if $\norm{A_i} = 1$ for all rows $A_i$.
\end{theorem}
\ifFOCS\else  
\begin{proof}

Suppose first that $n \geq M$. Consider $n \times d$ matrices $A, B^1, \ldots
B^M$, where for each $C \in \{A, B^1, \ldots, B^M\}$, $C_{i,j} = 0$ if either $j
> 1$ or $i > M$. Also, $A_{i, 1} = 1$ for $i \in [M]$, while for each $j$,
$B^j_{1, i} = 1$ if $i \in [M] \setminus \{j\}$, while $B^j_{1,j} = -1$. With
probability $1/2$ the matrix $A$ is chosen, otherwise a matrix $B^j$ is chosen
for a random $j$. Notice that whichever case we are in, each of the first $M$
rows of the input matrix has norm equal to $1$, while all remaining rows have
norm $0$. It is easy to see that distinguishing these two cases with probability
$\geq 2/3$ requires reading $\Omega(M)$ entries. As $\Omega(M)$ is a lower bound
for Monte Carlo algorithms, it is also a lower bound for Las Vegas algorithms.
Moreover, distinguishing these two cases is necessary, since if the problem is
classification, if $C = A$ the margin is $1$, otherwise it is $0$, while if the
problem is minimum enclosing ball, if $C =
  A$ the cost is $0$, otherwise it is $1$.

We now assume $M > n$. Let $d'$ be the largest integer for which $nd' < M$. Here
$d' \geq 1$. Let $A$ be the $n \times d'$ matrix, where $A_{i,j} =
\frac{1}{\sqrt{d'}}$ for all $i$ and $j$. The margin of $A$ is $1$, and the
minimum enclosing ball has radius $0$.

Suppose there were an algorithm $Alg$ on input $A$ for which there is an
assignment to $Alg$'s random tape $r$ for which $Alg$ reads at most $nd'/4$ of
its entries. If there were no such $r$, the expected running time of $Alg$ is
already $\Omega(nd') = \Omega(M)$. Let $A_{\ell}$ be a row of $A$ for which
$Alg$ reads at most $d'/4$ entries of $A_{\ell}$ given random tape $r$, and let
$S \subset [d']$ be the set of indices in $A_{\ell}$ read, where $|S| \leq
d'/4$. Consider the $n \times d'$ matrix $B$ for which $B_{i,j} = A_{i,j}$ for
all $i \neq \ell$, while $B_{\ell, j} = A_{\ell, j}$ for all $j \in S$, and
$B_{\ell, j} = -A_{\ell, j}$ for all $j \in [d'] \setminus S$. Notice that all
rows of $A$ and $B$ have norm $1$.

To bound the margin of $B$, consider any vector $x$ of norm at most $1$. Then
$$
(A_{\ell} + B_{\ell})x
    \le \norm{x} \cdot \norm{A_{\ell}+B_{\ell}}
    \le \norm{A_{\ell}+B_{\ell}}.
$$
$A_{\ell}+B_{\ell}$ has at least $3d'/4$ entries that
are $0$, while the non-zero entries all have value $2/\sqrt{d'}$. Hence,
$\norm{A_{\ell}+B_{\ell}}^2 \leq \frac{d'}{4} \cdot \frac{4}{d'} = 1$. It follows
that either $A_{\ell} x$ or $B_{\ell} x$ is at most $1/2$, which bounds the
margin of $B$. As $Alg$ cannot distinguish $A$ and $B$ given random tape $r$, it
cannot have one-sided additive error at most $1/2$.

For minimum enclosing ball, notice that
 $\norm{A_{\ell}-B_{\ell}}^2 \cdot \frac{1}{4}
  \geq \frac{3d'}{4} \cdot \frac{4}{d'} \cdot \frac{1}{4}
  = \frac{3}{4}$,
which lower bounds the cost of the minimum enclosing ball of $B$.
As $Alg$ cannot distinguish $A$ and $B$ given random tape $r$, it cannot have
one-sided additive error at most $3/4$.
\end{proof}
\fi 

\section{Concluding Remarks}

We have described a general method for sublinear optimization of constrained convex programs, and showed applications to classical problems in machine learning such as linear classification and minimum enclosing ball obtaining improvements in leading-order terms over the state of the art.
The application of our sublinear primal-dual algorithms to soft margin SVM and related convex problems is currently explored in ongoing work with Nati Srebro.

In all our running times
the dimension $d$ can be replaced by the parameter
$S$, which is the maximum over the input rows $A_i$
of the number of nonzero entries in $A_i$. Note
that $d\ge S \ge M/n$. Here we require the assumption that entries of any given
row can be recovered in $O(S)$ time, which is compatible with keeping each row as a hash table
or (up to a logarithmic factor in run-time) in sorted order.

\vspace{1em}
\paragraph{Acknowledgements}
We thank Nati Srebro and an anonymous referee for helpful comments on the relation between this work and PAC learning theory.

\ifFOCS
\else
  \bibliographystyle{alpha}
\fi
\bibliography{SketchedOptimization}

\ifFOCS\else 
\appendix

\section{Main Tools}

\subsection{Tools from online learning}

\paragraph{Online linear optimization}

The following lemma is essentially due to Zinkevich \cite{Zink}:
\begin{lemma}[OGD]\label{lem:ogd}
Consider a set of vectors $q_1,\ldots,q_T \in \reals^d$ such that
$\norm{q_i}_2 \leq c$. Let $x_0 \gets 0$,
and
$\xtil_{t+1} \gets x_t + \frac{1}{\sqrt{T}} q_{t} \ , \ x_{t+1} \gets \frac{\xtil_{t+1}}{\max\{1,\norm{\xtil_{t+1}}\}}$.
Then
\[
\max_{x\in\ball} \sum_{t=1}^T q_t\trans x -  \sum_{t=1}^T q_t\trans x_t
    \leq  2c \sqrt{T}.
\]
This is true even if each $q_t$ is dependent on $x_1,\ldots,x_{t-1}$.
\end{lemma}
\begin{proof}

Assume $c=1$, generalization is by straightforward scaling.
Let $\eta = \frac{1}{\sqrt{T}}$. By definition and for any $\norm{x}\leq 1$,
\[
\norm{ x - {x}_{t+1} }^2
    \leq \norm{ x - \xtil_{t+1} }^2
    = \norm{ x - x_{t} - \eta q_t }^2
    = \norm{x - x_t}^2 - 2 \eta q_t\trans (x - x_t) + \eta^2 \norm{q_t}^2.
\]
Rearranging we obtain
\[
q_t\trans (x - x_t)
    \leq \frac{1}{2\eta}[\norm{x - x_t}^2 - \norm{x - x_{t+1}}^2] + \eta/2.
\]
Summing up over $t=1$ to $T$ yields
\[
\sum_t q_t\trans x - \sum_t q_t\trans x_t
    \leq \frac{1}{2\eta} \norm{x -x_1}^2 + \eta T/2
    \leq \frac{2}{\eta } + \frac{\eta}{2} T \leq 2 \sqrt{T}.
\]
\end{proof}

For our streaming and parallel implementation, a simpler version of gradient descent, also essentially due to Zinkevich \cite{Zink}, is given by:
\begin{lemma}[Lazy Projection OGD]\label{lem:lazyogd}
Consider a set of vectors $q_1,\ldots,q_T \in \reals^d$ such that
$\norm{q_i}_2 \leq 1$. Let
$$ x_{t+1} \gets \arg \min _{x \in \ball} \left \{ \sum_{\tau=1}^t q_{\tau}^\top \cdot x + \sqrt{2T} \|x\|_2^2  \right \}$$
Then
\[
\max_{x\in\ball} \sum_{t=1}^T q_t\trans x -  \sum_{t=1}^T q_t\trans x_t
    \leq  2 \sqrt{2 T}.
\]
This is true even if each $q_t$ is dependent on $x_1,\ldots,x_{t-1}$.
\end{lemma}
For a proof see Theorem 2.1 in \cite{Hsurvey10}, where we take $\mathcal{R}(x) = \|x\|_2^2$, and the norm of the linear cost functions is bounded by $\|q_t\|_2 \leq 1$, as is the diameter of $\K$ - the ball in our case. Notice that the solution of the above optimization problem is simply:
$$ x_{t+1} = \frac{y_{t+1}}{\max \{1 , \|y_{t+1}\|\}} \ ,  \ y_{t+1} = \frac{-\sum_{\tau=1}^t q_\tau}{\sqrt{2T}}$$

\paragraph{Strongly convex loss functions}

The following Lemma is essentially due to
\cite{HazanKKA06}.

For $H\in\reals$ with $H>0$, a function $f:\reals^d\rightarrow \reals$
is \emph{$H$-strongly convex} in $\ball$ if for all $x\in\ball$,
all the eigenvalues of $\grad^2 f(x)$ are at least $H$.

\begin{lemma}[OGDStrictlyConvex]\label{lem:OGDStrictlyConvex}
Consider a set of H-strongly convex functions $f_1,\ldots,f_T$ such that the
norm of their gradients is bounded over the unit ball $\ball$ by
 $G \geq \max_t \max_{x \in \ball} \norm{\nabla f_t(x)}$.
Let $x_0 \in \ball$, and
$\xtil_{t+1} \gets x_t - \frac{1}{t} \nabla f_{t}(x_t)  \ ,
 \ x_{t+1} \gets \frac{\xtil_{t+1}}{\max\{1,\norm{\xtil_{t+1}}\}}$.
Then
\[
 \sum_{t=1}^T f_t(x_t) - \min_{\norm{x}_2 \leq 1} \sum_{t=1}^T f_t (x)
    \leq  \frac{G^2}{H}  \log{T}.
\]
This is true even if each $f_t$ is dependent on $x_1,\ldots,x_{t-1}$.
\end{lemma}

Again, for the MEB application and its relatives it is easier to implement the
lazy versions in the streaming model. The following Lemma is the analogous tool
we need:
\begin{lemma}\label{lem:OGDStrictlyConvexLazy}
Consider a set of H-strongly convex functions $f_1,\ldots,f_T$ such that the
norm of their gradients is bounded over the unit ball $\ball$ by
 $G \geq \max_t \max_{x \in \ball} \norm{\nabla f_t(x)}$.
Let
$$ x_{t+1} \gets \arg \min _{x \in \ball} \left \{ \sum_{\tau=1}^t f_{\tau}(x)   \right \}$$
Then
\[
 \sum_{t=1}^T f_t(x_t) - \min_{x\in\ball} \sum_{t=1}^T f_t (x)
    \leq  \frac{2 G^2}{H}  \log{T}.
\]
This is true even if each $f_t$ is dependent on $x_1,\ldots,x_{t-1}$.
\end{lemma}
\begin{proof}

By Lemma 2.3 in \cite{Hsurvey10} we have:
\[
 \sum_{t=1}^T f_t(x_t) - \min_{\norm{x}_2 \leq 1} \sum_{t=1}^T f_t (x) \leq \sum_t [f_t(x_t) - f_t(x_{t+1})]
\]
Denote by $\Phi_t(x) = \sum_{\tau = 1}^{t} f_\tau$.
Then by
Taylor expansion at $x_{t+1}$,
there exists a $z_t \in [x_{t+1},x_t]$ for which
\begin{align*}
\Phi_t(x_t)
	& =  \Phi_t(x_{t+1}) +  (x_t - x_{t+1})^\top \nabla \Phi_t(x_{t+1}) + \frac{1}{2} \|x_t - x_{t+1}\|^2_{z_t} \\
	& \geq \Phi_t(x_{t+1}) + \frac{1}{2}\|x_t - x_{t+1}\|^2_{z_t},
\end{align*}
using the notation
$\|y\|_{z}^2 = y^\top \nabla^2 \Phi_t(z) y $.
The inequality above is true because $x_{t+1}$ is a minimum of
$\Phi_t$ over $\K$.
Thus,
\begin{align*}
\|x_t - x_{t+1}\|^2_{z_t} & \leq  2\,\Phi_t(x_t) - 2\,\Phi_t(x_{t+1}) \\
& =  2\,\ (\Phi_{t-1}(x_t) - \Phi_{t-1}(x_{t+1})) + 2 [ f_t(x_t) - f_t(x_{t+1})] \\
& \le 2 [ f_t(x_t) - f_t(x_{t+1})] & \mbox{ optimality of $x_t$} \\
& \leq  2\,\nabla f_t(x_t)^\top (x_t - x_{t+1}) & \mbox{ convexity of $f_t$}~.
\end{align*}
By convexity and Cauchy-Schwarz:
\begin{align*}
f_t(x_t) - f_t(x_{t+1}) & \leq \nabla f_t(x_t) ( x_t - x_{t+1}) \leq \|\nabla f_t(x_t) \|_{z_t}^* \|x_t - x_{t+1}\|_{z_t}   \\
& \leq  \|\nabla f_t(x_t) \|_{z_t}^* \sqrt{ 2\,\nabla f_t(x_t)^\top (x_t - x_{t+1}) }
\end{align*}
Shifting sides and squaring, we get
\begin{align*}
f_t(x_t) - f_t(x_{t+1}) & \leq  \nabla f_t(x_t) ( x_t - x_{t+1}) \leq 2 \|\nabla f_t(x_t) \|_{z_t}^{* \ 2}
\end{align*}
Since $f_t$ are assumed to be $H$-strongly convex, we have $\| \cdot\|_z \geq \|\cdot\|_{Ht}$, and hence for the dual norm,
\begin{align*}
f_t(x_t) - f_t(x_{t+1}) & \leq  2 \|\nabla f_t(x_t) \|_{z_t}^{* \ 2} \leq 2 \frac{\|\nabla f_t(x_t)\|_2^2}{H t} \leq \frac{2 G^2}{H t}
\end{align*}
Summing over all iterations we get
\[
 \sum_{t=1}^T f_t(x_t) - \min_{\norm{x}_2 \leq 1} \sum_{t=1}^T f_t (x) \leq \sum_t [f_t(x_t) - f_t(x_{t+1})] \leq \sum_t\frac{2 G^2}{H t} \leq \frac{2G^2}{H} \log T
\]

\end{proof}

\paragraph{Combining sampling and regret minimization}

\begin{lemma}\label{lem:RandomStrictlyConvex}
Consider a set of H-strongly convex functions $f_1,\ldots,f_T$ such that the
norm of their gradients is bounded over the unit ball by $G \geq \max_t \max_{x
\in \ball} \norm{\nabla f_t(x)}$. Let
$$y_{t+1}
 \gets \mycases  { \arg \min _{x \in \ball} \left \{ \sum_{\tau=1}^t f_{\tau}(x)   \right \} }  {\mbox{w.p.}  \ \ \alpha}
                    { y_t}  { o/w}$$
Then for a fixed $x^*$ we have
\[ \E[
 \sum_{t=1}^T f_t(y_t)  - \sum_{t=1}^T f_t (x^*)]
    \leq  \frac{1}{\alpha} \frac{2 G^2}{H}  \log{T}.
\]
This is true even if each $f_t$ is dependent on $y_1,\ldots,y_{t-1}$.
\end{lemma}
\begin{proof}

Consider the sequence of functions $\tilde{f}_t$ defined as
$$\tilde{f}_{t}
 \gets \mycases  { \frac{f_t} {\alpha}}  {\mbox{w.p.}  \ \ \alpha}
                    { 0}  { o/w}$$
Where $0$ denotes the all-zero function. Then the algorithm from Lemma
\ref{lem:OGDStrictlyConvexLazy} applied to the functions $\tilde{f}_t$ is exactly
the algorithm we apply above to the functions $f_t$. Notice that the functions
$\tilde{f}_t$ are $\frac{H}{\alpha}$-strongly convex, and in addition their
gradients are bounded by $\frac{G}{\alpha}$. Hence applying Lemma
\ref{lem:OGDStrictlyConvexLazy} we obtain
\[ \E[   \sum_{t=1}^T {f}_t(y_t) -  \sum_{t=1}^T {f}_t (x^*)]
  =
\E[   \sum_{t=1}^T \tilde{f}_t(x_t) -  \sum_{t=1}^T \tilde{f}_t (x^*)]
    \leq  \frac{1}{\alpha} \frac{2 G^2}{H}  \log{T}.
\]

\end{proof}

\section{Auxiliary lemmas}\label{sec:aux lemmas}

First, some simple lemmas about random variables.

\begin{lemma} \label{lem:clip}
Let $X$ be a random variable with $|\E[X]| \leq C$, and let
$\bar{X} = \clip(X, C) = \min\{C, \max\{-C, X\}\}$ for some $C\in\reals$. Then
$$ | \E[\bar{X}] - \E[X] | \le \frac{\var[X]}{C}.$$
\end{lemma}
\begin{proof}

By direct calculation:
\begin{align*}
\E[\bar{X}] - \E[X]
       & =  \int_{x < -C } \Pr[x] (-C-x) + \int_{x > C } \Pr[x] (C-x),
    \\ & \le \int_{x < -C } \Pr[x] |x| - \int_{x < -C } \Pr[x] C
    \\ & \le \int_{x < -C } \Pr[x] x^2/C - \int_{x < -C } \Pr[x] C
    \\ & = \int_{x < -C } \Pr[x] \frac{x^2 - C^2}{C}
    \\ & \le \int_{x < -C } \Pr[x] \frac{x^2 - \E[X]^2}{C} & \mbox{since $|\E[X]| \leq C$}
    \\ & =  \frac{\var[X^2]}{C}
\end{align*}
and similarly $\E[\bar{X}] - \E[X] \ge - \var[X]/C$,
and the result follows.

\end{proof}

\begin{lemma}\label{lem:convex moment}
For random variables $X$ and $Y$, and $\alpha\in [0,1]$,
\[\E[(\alpha X + (1-\alpha)Y)^2] \le \max\{\E[X^2], \E[Y^2]\}.\]
\end{lemma}

This implies by induction that the second moment of a convex combination of
random variables is no more than the maximum of their second moments.

\begin{proof}

We have, using Cauchy-Schwarz for the first inequality,
\begin{align*}
E[(\alpha X + (1-\alpha)Y)^2]
	   & = \alpha^2\E[X^2] + 2\alpha(1-\alpha)\E[XY] + (1-\alpha)^2\E[Y^2]
	\\ & \le \alpha^2\E[X^2] + 2\alpha(1-\alpha)\sqrt{\E[X^2]\E[Y^2]} + (1-\alpha)^2\E[Y^2]
	\\ & = (\alpha\sqrt{\E[X^2]} + (1-\alpha)\sqrt{\E[Y^2]})^2
	\\ & \le \max\{\sqrt{\E[X^2]}, \sqrt{\E[Y^2]}\}^2
	\\ & = \max\{\E[X^2], \E[Y^2]\}.
\end{align*}
\end{proof}

\subsection{Martingale and concentration lemmas}

The Bernstein inequality, that holds for random variables $Z_t, t\in[T]$ that are
independent, and such that for all $t$, $\E[Z_t]=0$, $\E[Z_t^2]\le s$, and
$|Z_t| \le V$, states
\begin{equation}\label{eq:Bernstein}
\log \Prob{\sum_{t\in [T]} Z_t \ge \alpha} \le -\alpha^2/2(Ts + \alpha V/3)
\end{equation}
Here we need a similar bound for random variables which are not independent, but
form a martingale with respect to a certain filtration.
Many concentration results have been proven for Martingales, including
somewhere, in all likelihood, the present lemma. However, for clarity and
completeness, we will outline how the proof of the Bernstein inequality can be
adapted to this setting.

\begin{lemma}\label{lem:Bern}
Let $\{Z_t\}$ be a martingale difference sequence with respect to filtration
$\{S_t\}$, such that $\E[Z_t|S_1,...,S_t] = 0$. Assume the filtration $\{S_t\}$
is such that the values in $S_t$ are determined using only those in $S_{t-1}$,
and not any previous history, and so the joint probability distribution
\[
\Prob{S_1=s_1, S_2=s_2, \ldots, S_T=s_t}
    = \prod_{t\in[T-1]} \Prob{S_{t+1}=s_{t+1}\mid S_t=s_t},
\]

In addition, assume for all $t$,  $\E[Z_t^2 | S_1,...,S_t]\le s$, and $|Z_t| \le V$. Then
\[
\log \Prob{\sum_{t\in T} Z_t \ge \alpha} \le -\alpha^2/2(Ts + \alpha V/3).
\]
\end{lemma}

\begin{proof}

A key step in proving the Bernstein inequality is to show an upper bound on the
exponential generating function $\E[\exp(\lambda Z)]$, where
 $Z\equiv \sum_t Z_t$, and $\lambda>0$ is a parameter to be chosen. This step is
where the hypothesis of independence is applied. In our setting, we can show a
similar upper bound on this expectation:
Let $\E_t[]$ denote expectation with respect to $S_t$,
and $\E_{[T]}$ denote expectation with respect to $S_t$ for $t\in [T]$.
This expression for the probability distribution
implies that for any real-valued function $f$ of state tuples $S_t$,
\begin{align*}
\E_{[T]} & [\prod_{t\in[T]} f(S_t)]
    \\ & = f(s_1) \int_{s_2,\ldots,s_T} [\prod_{t\in[T-1]} f(s_{t+1})]
                        [\prod_{t\in[T-1]} \Prob{S_{t+1}=s_{t+1}\mid S_t=s_t}]
    \\ & = f(s_1) \int_{s_2,\ldots,s_{T-1}}\left[
        [\prod_{t\in[T-2]} f(s_{t+1})]
        [\prod_{t\in[T-2]} \Prob{S_{t+1}=s_{t+1}\mid S_t=s_t}]\right.
    \\ &\quad\quad\qquad\qquad
            \left. \int_{s_T} f(s_T) \Prob{S_T=s_T\mid S_{T-1}=s_{T-1}}\right],
\end{align*}
where the inner integral can be denoted as the conditional
expectation $\E_T[f(S_T)\mid S_{T-1}]$. By induction this is
\[
    f(s_1)
     \left[ \int_{s_2} f(s_2) \Prob{S_2=s_2\mid S_1=s_1}\left[\int_{s_3} \ldots
    \int_{s_T} f(s_T) \Prob{S_T=s_T\mid S_{T-1}=s_{T-1}}\right]\ldots\right],
\]
and by writing the constant $f(S_1)$ as the expectation with respect
to the constant $S_0=s_0$,
and using $\E_X[\E_X[Y]] = \E_X[Y]$ for any random variables $X$ and $Y$,
we can write this as
\[
\E_{[T]} [\prod_{t\in[T]} f(S_t)]
    = \E_{[T]} [\prod_{t\in[T]} \E_t[f(S_t)\mid S_{t-1}]].
\]
For fixed $i$ and a given $\lambda\in\reals$, we take
$f(S_1) = 1$, and
$f(S_t) \equiv \exp(\lambda Z_{t-1})$,
to obtain
\[
\E_{[T]}\left[ \exp(\lambda \sum_{t\in[T]} Z_t)]\right]
    = \E_{[T]} \left[\prod_{t\in[T]} \E_t[\exp(\lambda Z_t)\mid S_{t-1}]\right].
\]
Now for \emph{any} random variable $X$ with $\E[X]=0$, $\E[X^2] \le s$, and
$|X|\le V$,
\[
\E[\exp(\lambda X)] \le \exp\left(\frac{s}{V^2}(e^{\lambda V} - 1 - \lambda V)\right),
\]
(as is shown and used for proving Bernstein's inequality in the independent case)
and therefore
\[
\E_{[T]}\left[ \exp(\lambda Z)\right]
    \le \E_{[T]}\left[\prod_{t\in[T]} \exp\left(\frac{s}{V^2}(e^{\lambda V} - 1 - \lambda V)\right)\right]
    = \exp\left(T\frac{s}{V^2}(e^{\lambda V} - 1 - \lambda V)\right).
\]
where $Z\equiv \sum_{t\in [T]} Z_t$.
This bound is the same as is obtained for independent $Z_t$, and so
the remainder of the proof is exactly as in the proof for the independent case:
Markov's inequality is applied to the random variable $\exp(\lambda Z)$,
obtaining
\[
\Prob{Z\ge \alpha}
    \le \exp(-\lambda\alpha)\E_{[T]}\left[ \exp(\lambda Z)\right]
    \le \exp(-\lambda\alpha + T\frac{s}{V^2}(e^{\lambda V} - 1 - \lambda V)),
\]
and an appropriate value $\lambda=\frac{1}{V}\log(1+\alpha V/Ts)$ is chosen
for minimizing the bound, yielding
\[
\Prob{Z\ge \alpha}
    \le \exp(-\frac{Ts}{V^2}((1+\gamma)\log(1+\gamma)-\gamma)),
\]
where $\gamma\equiv \alpha V/Ts$, and finally the inequality for $\gamma\ge 0$ that
$(1+\gamma)\log(1+\gamma) - \gamma \ge \frac{\gamma^2/2}{1+\gamma/3} $ is applied.
\end{proof}

\subsection{Proof of lemmas used in main theorem}

We restate and prove lemmas \ref{lem:v_t conc},\ref{lem:highprob2} and \ref{lem:medprob1}, in slightly more general form. In the following we only assume that $v_t(i) = \clip(\vtil_t(i),\frac{1}{\eta})$ is the clipping of a random variable $\vtil_t(i)$. The variance of $\vtil_t(i)$ is at most one $\var[\vtil_t(i)] \leq 1$, and we denote by $\mu_t(i) = \E[\vtil_t(i)]  $. We also assume that the expectations of $\vtil_t(i)$ are bounded by an absolute constant $ |\mu_t(i)| \leq C \leq \frac{1}{\eta}$. This constant is one for the perceptron application, but at most two for MEB.
Note that since the variance of $\vtil_t(i)$ is bounded by one, so is the variance of it's clipping $v_t(i)$ \footnote{This follows from the fact that the second moment only decreases by the clipping operation, and definition of variance as $\var(v_t(i)) = \min_z \E[ v_t(i)^2 - z^2 ]$. We can use $z = \E[\vtil_t(i)]$, and hence the decrease in second moment suffices.}.


\begin{lemma} \label{lem:genericHP_v_and_mu}
For $\eta \leq \sqrt{\frac{\log n}{10 T}}$, with probability at least $1-O(1/n)$,
\[
\max_i \sum_{t\in [T]} [v_t(i) - \mu_t(i)] \le 90 \eta T  .
\]
\end{lemma}

\begin{proof}

Lemma~\ref{lem:clip} implies that
$|\E[{v}_t(i)] - \mu_t(i)| \le \eta$, since $\var[\vtil_t(i)]\le 1$.

We show that for given $i\in [n]$, with probability $1 - O(1/n^{2})$,
$\sum_{t\in [T]} [v_t(i) - \E[{v}_t(i)]] \le 80 \eta T $, and then apply the
union bound over all $i\in [n]$. This together with the above bound on
$|\E[{v}_t(i)] - \mu_t(i)|$ implies the lemma via the triangle inequality.

Fixing $i$, let $Z_t^i \equiv v_t(i) - \E[{v}_t(i)]$, and consider the filtration given by
\[
    S_t \equiv (x_t, p_t, w_t, y_t, v_{t-1}, i_{t-1}, j_{t-1}, v_{t-1} - \E[{v}_{t-1}]),
\]
Using the notation $\E_t[\cdot] = \E[\cdot|S_t]$, Observe that
\begin{enumerate}
\item
$\forall t \ . \ \E_t[(Z_t^i) ^2 ] = \E_t[v_t(i)^2] - \E_t[v_t(i)]^2 = \var(v_t(i)) \leq 1$.
\item
$|Z_t^i|\le 2/\eta$.
This holds since by construction, $|v_t(i)|\le 1/\eta$, and hence
\begin{align*}
|Z_t^i| & = | v_t(i) - \E[v_t(i)]| \le | v_t(i)| + |\E[v_t(i)]| \leq \frac{2}{\eta}
\end{align*}
\end{enumerate}
Using these conditions, despite the fact that the $Z_t^i$ are not independent, we can use
Lemma~\ref{lem:Bern}, and conclude that
$Z\equiv \sum_{t\in T} Z_t^i$ satisfies the Bernstein-type inequality with $s=1$ and $V=2/\eta$
\[
\log \Prob{Z \ge \alpha} \le -\alpha^2/2(Ts + \alpha V/3) \le -\alpha^2/2(T + 2\alpha/3\eta),
\]
Letting $\alpha \gets 80 \eta T$, we have
\[
\log \Prob{Z \ge 80 \eta T} \le -\alpha^2/2(T + 2\alpha/3\eta) \le -20 \eta^2 T
\]
For $\eta = \sqrt{\frac{\log n}{10 T}}$, above probability is at most $e^{- 2 \log n} \leq \frac{1}{n^{2}} $.
\end{proof}

Lemma \ref{lem:highprob2} can be restated in the following more general form:
\begin{lemma}\label{lem:highprob2_generic}
For $\eta \leq \sqrt{\frac{\log n}{10 T}}$, with probability at least
$1-O(1/n)$, it holds that $\left| \sum_{t\in [T]} \mu_t(i_t) - \sum_t p_t
\trans v_t \right| \le 100 C \eta T .$
\end{lemma}
It is a corollary of the following two lemmas:
\begin{lemma}
For $\eta \leq \sqrt{\frac{\log n}{10 T}}$, with probability at least $1-O(1/n)$,
\[
\left|  \sum_{t\in [T]} p_t \trans v_t -  \sum_t p_t \trans \mu_t \right|  \le 90 \eta T  .
\]
\end{lemma}

\begin{proof}

This Lemma is proven in essentially the same manner as Lemma \ref{lem:v_t conc},
and proven below for completeness.

Lemma~\ref{lem:clip} implies that
$|\E[{v}_t(i)] - \mu_t(i)| \le \eta$, using $\var[\vtil_t(i)]\le 1$. Since $p_t$
is a distribution, it follows that $|\E[p_t \trans {v}_t] - p_t \trans \mu_t|
\leq \eta$

Let $Z_t \equiv p_t \trans v_t - \E[p_t \trans {v}_t] = \sum_i p_t(i) Z_t^i$,
where $Z_t^i = v_t(i) - \E[v_t(i)]$. Consider the filtration given by
\[
    S_t \equiv (x_t, p_t, w_t, y_t, v_{t-1}, i_{t-1}, j_{t-1}, v_{t-1} - \E[{v}_{t-1}]),
\]
Using the notation $\E_t[\cdot] = \E[\cdot|S_t]$, the quantities
$|Z_t|$ and $\E_t[Z_t^2] $ can be bounded as follows:
\begin{align*}
|Z_t|
    & = |\sum_i p_t(i) Z_t^i|   \le \sum_i p_t(i) |Z_t^i|
      \leq  2 \eta^{-1}
        & \mbox{using $|Z_t^i| \leq 2\eta^{-1} $ as in Lemma \ref{lem:v_t conc}}.
\end{align*}
Also, using properties of variance, we have
\[
\E[Z_t^2]
	= \Var[p_t \trans v_t] = \sum_i p_t(i)^2 \var(v_t(i)) 	\le \max_i \var[v_t(i)] \le 1.
\]


We can now apply the Bernstein-type inequality of
Lemma~\ref{lem:Bern}, and continue exactly as in Lemma \ref{lem:v_t conc}.
\end{proof}

\begin{lemma}
For $\eta \leq \sqrt{\frac{\log n}{10 T}}$, with probability at least $1-O(1/n)$,
\[
\left|  \sum_{t\in [T]}  \mu_t(i_t) -  \sum_t p_t \mu_t  \right|  \le 10 C \eta T  .
\]
\end{lemma}

\begin{proof}

Let $Z_t \equiv \mu_t(i_t) -  p_t \mu_t $, where now $\mu_t$ is a constant vector and $i_t$ is the random variable, and consider the filtration given by
\[
    S_t \equiv (x_t, p_t, w_t, y_t, v_{t-1}, i_{t-1}, j_{t-1}, Z_{t-1}),
\]
The expectation of $\mu_t(i_t)  $, conditioning on $S_t$ with respect to the random
choice $r(i_t)$, is $p_t \mu_t$. Hence $\E_t[ Z_t] = 0$,
where $\E_t[\cdot]$ denotes $\E[\cdot|S_t]$.
The parameters $|Z_t|$ and $\E[Z_t^2] $ can be bounded as follows:
\begin{align*}
|Z_t|  & \leq  |\mu_t(i)| + | p_t \mu_t| \leq 2C
\end{align*}
\begin{align*}
\E[Z_t^2]
    & = \E[ (\mu_t(i)  - p_t \trans \mu_t)^2 ]
    \leq 2 \E[ \mu_t(i)^2] + 2 (p_t \trans \mu_t)^2 \leq 4C^2
\end{align*}
Applying Lemma~\ref{lem:Bern} to $Z\equiv \sum_{t\in T} Z_t$,
with parameters $s \leq 4C^2 \ , \ V \leq 2C$, we obtain
\[
\log \Prob{Z \ge \alpha} \le -\alpha^2/(4C^2T + 2C\alpha),
\]
Letting $\alpha \gets 10 C \eta T$, we obtain
\[
\log \Prob{Z \ge 10 \eta T } \le -\frac{100 \eta^2 C^2 T^2}{4 C^2 T + 20 C^2 \eta T} \le 5 \eta^2 T \leq \log n
\]
Where the last inequality holds assuming $\eta \leq \sqrt{\frac{\log n}{ T}}$.
\end{proof}

Finally, we prove Lemma \ref{lem:medprob1} by a simple application of Markov's inequality:
\begin{lemma} \label{lem:medprob_generic}
w.p. at least $1 - \frac{1}{4}$  it holds that $\sum_t p_t \trans v_t^2 \leq 8C^2T .$
\end{lemma}
\begin{proof}
By assumption, $\E[\vtil^2_t(i)] \leq C^2$, and using Lemma \ref{lem:clip}, we have $\E[v_t(i)^2] \leq (C+\frac{1}{C})^2 \leq 2 C^2 $.

By linearity of expectation, we have
$\E[ \sum_t p_t \trans v_t^2 ] \leq 2 C^2 T$, and since the random variables $v_t^2$
are non-negative, applying Markov's inequality yields the lemma.

\end{proof}



\section{Bounded precision}\label{subsec:precision}
All algorithms in this paper can be implemented with bounded precision.

First we observe that approximation of both the training data and
the vectors that are ``played'' does not increase the regret
too much, for both settings we are working in.

\begin{lemma}\label{lem:approxOGD}
Given a sequence of functions $f_1,\ldots,f_T$
and another sequence $\ftil_1,\ldots,\ftil_T$
all mapping $\reals^d$ to $\reals$,
such that $|\ftil_t(x) - f_t(x)| \le\alpha_f$ for all $x\in B$ and $t\in[T]$,
suppose $x_1,\ldots,x_T\in\ball$ is a sequence
of regret $R$ against $\{\ftil_t\}$, that is,
\[
\max_{x\in\ball} \sum_{t\in[T]} \ftil_t(x) - \sum_{t\in[T]} \ftil_t (x_t)
    \le R.
\]
Now suppose $\xtil_1,\ldots,\xtil_T\in\reals^d$ is a sequence
with $|f_t(\xtil_t) - f_t(x_t)| \le \alpha_x$ for all $t\in[T]$.
Then
\[
\max_{x\in\ball}\sum_{t\in[T]}f_t(x) - \sum_{t\in[T]} f_t (\xtil_t)
    \le R + T(\alpha_x + 2\alpha_f).
\]
\end{lemma}

\begin{proof}
For $x\in\ball$, we have $\sum_{t\in[T]}f_t(x) \le  \sum_{t\in[T]} \ftil_t(x) + T\alpha_f$,
and
\[
 \sum_{t\in[T]} f_t (\xtil_t)
 	\ge  \sum_{t\in[T]} f_t (x_t) - T\alpha_x
	\ge  \sum_{t\in[T]} \ftil_t (x_t) - T\alpha_x - T\alpha_f,
\]
and the result follows by combining these inequalities.
\end{proof}

That is, $x_t$ is some sequence known to have small regret against the
``training functions'' $\ftil_t(x)$, which are approximations to the true
functions of interest, and the $\xtil_t$ are approximations to these $x_t$.
The lemma says that despite these approximations, the $\xtil_t$ sequence
has controllable regret against the true functions.

This lemma is stated in more generality than we need:
all functions considered here have the form $f_t(x) = b_t + q_t\trans x + \gamma\norm{x}^2$,
where $|b_t|\le 1$, $q_t\in\ball$, and $|\gamma| \le 1$. Thus
if $\ftil_t(x) = \tilde{b}_t + \qtil_t\trans x + \gamma\norm{x}^2$,
then the first condition $|\ftil_t(x) - f_t(x)| \le\alpha_f$
holds when $|b_t-\tilde{b}_t| + \norm{q_t -\qtil_t}\le \alpha_f$.
Also, the second condition $|f_t(\xtil_t) - f_t(x_t)| \le \alpha_x$
holds for such functions when
$\norm{\xtil_t - x_t}\le \alpha_x/3$.

\begin{lemma}\label{lem:approxMW}
Given a sequence of vectors $q_1,\ldots,q_T\in\reals^n$,
with $\norm{q_t}_\infty \le B$ for $t\in[T]$,
and a sequence $\qtil_1,\ldots,\qtil_T\in\reals^n$
such that $\norm{\qtil_t - q_t}_\infty\le\alpha_q$ for all $t\in[T]$,
suppose $p_1,\ldots,p_T\in\Delta$ is a sequence
of regret $R$ against $\{\qtil_t\}$, that is,
\[
\sum_{t\in[T]} p_t\trans \qtil_t -  \min_{p\in\Delta}\sum_{t\in[T]} p\trans \qtil_t
    \le R.
\]
Now suppose $\ptil_1,\ldots,\ptil_T\in\reals^n$ is a sequence
with $\norm{\ptil_t - p_t}_1\le\alpha_p$ for all $t\in[T]$.
Then
\[
\sum_{t\in[T]} \ptil_t\trans q_t -  \min_{p\in\Delta} \sum_{t\in[T]} p\trans q_t
    \le R + T(B\alpha_p + 2\alpha_q).
\]
\end{lemma}
\begin{proof}
For $p\in\Delta$ we have $\sum_{t\in[T]} p\trans q_t \ge \sum_{t\in[T]} p\trans \qtil_t + T\alpha_q$,
and
\[
\sum_{t\in[T]} \ptil_t\trans q_t
	\le \sum_{t\in[T]} p_t\trans q_t + TB\alpha_p
	\le \sum_{t\in[T]} p_t\trans \qtil_t + TB\alpha_p + T\alpha_q,
\]
The proof follows by combining the inequalities.
\end{proof}

Note that to have $\norm{\ptil_t - p_t}_1\le\alpha_p$, it is enough that the relative
error of each entry of $\ptil_t$ is $\alpha_p$.

The use of $\qtil_t$ in place of $q_t$ (for either of the two lemmas)
will be helpful for our semi-streaming and kernelized algorithms
(\S\ref{sec:streaming}, \S\ref{sec:kernel-long}), where computation of the norms $\norm{y_t}$
of the working vectors $y_t$ is a bottleneck; the above two lemmas imply that
it is enough to compute such norms to within relative $\epsilon$ or so.

\subsection{Bit Precision for \AlgPDP\ }\label{subsubsec:bits perceptron}

First, the bit precision needed for the OGD part of the algorithm.
Let $\gamma$ denote a sufficiently small constant fraction of $\epsilon$,
where the small constant is absolute.
From Lemma~\ref{lem:approxOGD} and following discussion,
we need only use the rows $A_i$ up to a precision that gives an approximation
$\tilde{A}_i$ that is within Euclidean distance $\gamma$,
and similarly for an approximation
$\xtil_t$ of $x_t$.  For the latter, in particular, we need only compute
$\norm{y_t}$ to within relative error $\gamma$. Thus a per-entry precision
of $\gamma/\sqrt{d}$ is sufficient.

We need $\norm{x_t}$ for $\ell_2$ sampling; arithmetic relative error $\gamma/\sqrt{d}$
in the sampling
procedure
gives an estimate of $\vtil_t(i)$ for which $\E[A\vtil_t]=A\hat{x}_t$,
where $\hat{x}_t$ is  a vector within $O(\gamma)$ Euclidean distance of $x_t$.
We can thus charge this error to the OGD analysis, where $\hat{x}_t$
is the $\xtil_t$ of Lemma~\ref{lem:approxOGD}.

For the MW part of the algorithm, we observe that due to the clipping
step, if the initial computation of $\vtil_t(i)$, Line~\ref{alg:vt initial},
is done with $\eta\epsilon/5$ relative
error, then the computed value is within $\epsilon/5$ additive error.
Similar precision for the clipping implies that the
computed value of $v_t(i)$, which takes the place of $\qtil_t$ in
Lemma~\ref{lem:approxMW}, is within $\epsilon/5$ of the exact
version, corresponding to $q_t$ in the lemma. Here $B$ of
the lemma, bounding $\norm{q_t}_\infty$, is $1/\eta$, due to
the clipping.

It remains to determine the arithmetic relative error needed in the update
step, Line~\ref{alg:vt update}, to keep the relative error of the computed
value of $p_t$, or $\ptil_t$ of Lemma~\ref{lem:approxMW}, small enough.
Indeed, if the relative error is a small enough constant fraction of $\eta\epsilon/T$,
then the relative error of all updates together can be $\eta\epsilon/3$. Thus
$\alpha_p\le\eta\epsilon/3$ and $\alpha_q\le\epsilon/3$
and the added regret due to arithmetic error is at most $T\epsilon$.

Summing up: the arithmetic precision needed is at most on the order
of
\[
-\log\min\{\epsilon/\sqrt{d}, \eta\epsilon, \eta \epsilon/T \}
	= O(\log(nd/\epsilon)),
\]
to obtain a solution with additive $T\epsilon/10$ regret over the solution
computed using exact computation. This implies an additional error
of $\epsilon/10$ to the computed solution, and thus changes only
constant factors in the algorithm.

\subsection{Bit Precision for Convex Quadratic Programming}

From the remarks following Lemma~\ref{lem:approxOGD}, the conditions of that lemma
hold in the setting of convex quadratic programming in the simplex, assuming that every
$A_i\in\ball$. Thus the discussion of \S\ref{subsubsec:bits perceptron} carries over, up to constants,
with the simplification that computation of $\norm{y_t}$ is not needed.

\fi 

\end{document}